\documentclass{article}

\usepackage{microtype}
\usepackage{graphicx}
\usepackage{subcaption}
\usepackage{booktabs} % for professional tables

\usepackage{hyperref}

\PassOptionsToPackage{table}{xcolor} % enable \rowcolors for table striping

\usepackage[accepted]{icml2026}

\usepackage{amsmath}
\usepackage{amssymb}
\usepackage{todonotes} % Load after xcolor to prevent option clash
\usepackage{appendix}
\usepackage{multirow}
\usepackage[font=small,labelfont=bf]{caption}
\usepackage{graphicx}
\usepackage{subcaption}
\usepackage{array}
\usepackage{subcaption}
\usepackage{wrapfig} 
\usepackage{titletoc}

\usepackage{listings}
\usepackage{wrapfig}

\usepackage{colortbl}  % xcolor already loaded above with dvipsnames option
\definecolor{lavender}{rgb}{0.9, 0.9, 0.98}
\definecolor{red}{RGB}{255, 0, 0}
\definecolor{orange}{RGB}{252, 130, 62}
\definecolor{blue}{RGB}{0, 0,255}
\definecolor{darkgreen}{RGB}{0, 180,0}
\definecolor{lightgray}{gray}{0.9}  
\definecolor{lightblue}{rgb}{0.68,0.85,0.9}  
\definecolor{lightorange}{rgb}{1.0, 0.8, 0.6}
\definecolor{myblue}{HTML}{598BE7}

\newcommand{\note}[1]{\color{orange}}
\usepackage{booktabs}
\usepackage{acronym}
\acrodef{vla}[VLA]{vision-language-action}
\acrodef{il}[IL]{imitation learning}
\acrodef{rl}[RL]{reinforcement learning}
\acrodef{rlft}[RLFT]{RL fine-tuning}
\acrodef{fm}[FM]{flow matching}

\usepackage{caption}
\usepackage{array}

\usepackage{amsmath,amsthm,amsfonts,mathtools}
\newcommand{\DKL}{D_{\mathrm{KL}}}

\usepackage[most]{tcolorbox}

\newcommand{\E}{\mathbb{E}}
\newcommand{\KL}{D_{\mathrm{KL}}}

\usepackage{etoolbox}
\renewcommand{\arraystretch}{1.0}  % tighter rows
\AtBeginEnvironment{tabular}{\footnotesize}
\AtBeginEnvironment{tabular*}{\footnotesize}

\newcommand{\methodname}{BMD}

\makeatletter
\renewcommand\paragraph{\@startsection{paragraph}{4}{\z@}%
  {0.25ex \@plus 0.5ex \@minus 0.2ex}% space before
  {-0.8em}% space after (negative = run-in)
  {\normalfont\normalsize\bfseries}}
\makeatother

\newcommand{\tablerowcolors}{\rowcolors{1}{lightgray!50}{white}}

\usepackage[capitalize,noabbrev]{cleveref}

\theoremstyle{plain}
\newtheorem{theorem}{Theorem}[section]
\newtheorem{proposition}[theorem]{Proposition}

\theoremstyle{definition}

\theoremstyle{remark}
\newtheorem{remark}[theorem]{Remark}

\icmltitlerunning{Behavioral Mode Discovery for Fine-tuning Multimodal Generative Policies}

\begin{document}

\twocolumn[
  % \icmltitle{Unsupervised Mode Discovery for Fine-tuning Multimodal Generative Policies}
  
  \icmltitle{Behavioral Mode Discovery for Fine-tuning Multimodal Generative Policies}

  % It is OKAY to include author information, even for blind submissions: the
  % style file will automatically remove it for you unless you've provided
  % the [accepted] option to the icml2026 package.

  % List of affiliations: The first argument should be a (short) identifier you
  % will use later to specify author affiliations Academic affiliations
  % should list Department, University, City, Region, Country Industry
  % affiliations should list Company, City, Region, Country

  % You can specify symbols, otherwise they are numbered in order. Ideally, you
  % should not use this facility. Affiliations will be numbered in order of
  % appearance and this is the preferred way.
  \icmlsetsymbol{equal}{*}

  \begin{icmlauthorlist}
    \icmlauthor{Alberta Longhini}{sch,comp}
    \icmlauthor{David Emukpere}{comp}
    \icmlauthor{Jean-Michel Renders}{comp}
    \icmlauthor{Seungsu Kim}{comp}
  \end{icmlauthorlist}

  \icmlaffiliation{comp}{Naver Labs Europe, 6 Chem. de Maupertuis, Meylan, France}
  \icmlaffiliation{sch}{Department of Computer Science, Stanford University, CA, USA}

  \icmlcorrespondingauthor{Alberta Longhini}{alberta@stanford.edu}
  \icmlcorrespondingauthor{David Emukpere}{david.emukpere@naverlabs.com}

  % You may provide any keywords that you find helpful for describing your
  % paper; these are used to populate the "keywords" metadata in the PDF but
  % will not be shown in the document
  \icmlkeywords{Machine Learning, ICML}

% ]
    % \vskip 0.1in
    % {%
    % \renewcommand\twocolumn[1][]{#1}%
    % \begin{center}
    %     \centering
    %     \captionsetup{type=figure}
    % \includegraphics[width=\textwidth]{figures/01_intro/multimodality_comparison_v4.pdf} 
    % % \vspace{-1cm}
    %     \captionof{figure}{\textbf{Behavioral Mode Discovery (\methodname{}) prevents mode collapse during RL fine-tuning}.  Standard RL fine-tuning (RLFT) of a pre-trained multimodal policy often concentrates probability mass on few reward-maximizing behaviors, forgetting modes discovered during pre-training. Our approach, \methodname{}, regularizes RLFT to preserve multimodality while adapting to the downstream task. In the figure, each panel overlays trajectories (black) starting from the origin in a reward landscape with four symmetric goals (bright peaks). \emph{Left:} the pre-trained diffusion policy covers all four modes. \emph{Middle (DPPO):}  under two rotated reward shifts (cyan arrows), fine-tuning collapses to a subset of goals. \emph{Right (\methodname{}):} under the same shifts, the policy adapts without forgetting the modes.}

    %     \label{fig:fig1}
        
    % \end{center}%
    % }

  \vskip 0.3in
]

    % \vskip 0.3in

% this must go after the closing bracket ] following \twocolumn[ ...

% This command actually creates the footnote in the first column listing the
% affiliations and the copyright notice. The command takes one argument, which
% is text to display at the start of the footnote. The \icmlEqualContribution
% command is standard text for equal contribution. Remove it (just {}) if you
% do not need this facility.

% Use ONE of the following lines. DO NOT remove the command.
% If you have no special notice, KEEP empty braces:
    \printAffiliationsAndNotice{}  % no special notice (required even if empty)
% Or, if applicable, use the standard equal contribution text:
% \printAffiliationsAndNotice{\icmlEqualContribution}

% \begin{abstract}
% We address the problem of fine-tuning pre-trained generative policies with reinforcement learning (RL) while preserving the multimodality of their action distributions. Existing approaches often collapse diverse behaviors into a single reward-maximizing mode. We propose an unsupervised mode discovery framework that regularizes RL fine-tuning via mutual information, enabling improved task performance while maintaining behavioral diversity.
% \end{abstract}

\begin{abstract}
    We address the problem of fine-tuning pre-trained generative policies with reinforcement learning (RL) while preserving the multimodality of their action distributions. Existing methods for RL fine-tuning of generative policies (e.g., diffusion policies) improve task performance but often collapse diverse behaviors into a single reward-maximizing mode. To mitigate this issue, we propose an unsupervised mode discovery framework that uncovers latent behavioral modes within generative policies. The discovered modes enable the use of mutual information as an intrinsic reward, regularizing RL fine-tuning to enhance task success while maintaining behavioral diversity. Experiments on robotic manipulation tasks demonstrate that our method consistently outperforms conventional fine-tuning approaches, achieving higher success rates and preserving richer multimodal action distributions.
    \footnote{Website: \url{https://behav-mod-dis.github.io}}
    % Fine-tuning pre-trained generative policies with reinforcement learning (RL) can improve task performance, but it often collapses the multimodal behaviors acquired during supervised pre-training into a single reward-maximizing strategy. We propose an unsupervised mode-discovery framework that uncovers latent behavioral modes in a pre-trained generative policy and uses a mutual-information objective to make these modes identifiable. This signal yields an intrinsic reward that regularizes RL fine-tuning, promoting task adaptation while preserving behavioral diversity. Experiments on multimodal robotic manipulation tasks show that our approach improves success rates over conventional fine-tuning and better maintains multimodal action distributions.
\end{abstract}

% In the unusual situation where you want a paper to appear in the
% references without citing it in the main text, use \nocite

% \begin{figure}[t]
%     \centering
%     \includegraphics[width=0.97\columnwidth]{figures/01_intro/all_skills.pdf}
%     \caption{\textbf{Lorem ipsum dolor sit amet.} Lorem ipsum dolor sit amet, consectetur adipiscing elit. Sed do eiusmod tempor incididunt ut labore et dolore magna aliqua. TO TRY: Anymal and Avoid on top, or on second page. Try }
%     \label{fig:fig1}
%   \end{figure}

\section{Introduction}
\label{sec:intro}

% Importance and explanation of the Problem we address: Fine-tuning pre-trained diffusion policies while maintaining multimodality.
 Robotic tasks are inherently multimodal, admitting diverse yet valid strategies: a cup can be grasped from either side, a block can be rotated clockwise or counterclockwise, and kinematic redundancy enables diverse motions to reach the same goal. Preserving this diversity is crucial for policies that are robust, versatile, and adaptable to perturbations and unforeseen situations. %Policies that collapse to a single dominant strategy not only lose flexibility but also fail to generalize when the environment demands an alternative solution. 
Recent advances in policy learning show that expressive architectures such as diffusion~\citep{chi2023diffusion,kang2023efficient,psenka2023learning} and flow-based models~\citep{lipman2022flow,park2025flow} can capture multimodal behavior from demonstrations. However, their behavior remains limited by the support of the demonstration dataset. \Ac{rl} provides a natural mechanism to improve these pre-trained policies beyond demonstrations, yet \ac{rlft} often concentrates probability mass on a small set of reward-maximizing behaviors, degrading behavioral diversity~\citep{zhou2024rethinking,brown2019extrapolating}. 
The central problem we address is therefore: \textit{how can we fine-tune pre-trained generative policies with \ac{rl} while preserving the multimodality acquired during supervised pre-training?}

% What the community currently does: most overlook the multimodal aspect, destroying multimodality, some try to be clearer on what multimodality is, but they assum knowledge of the number of modalities of the policy
% Despite the community’s growing interest in policies showcasing multimodal behaviors, little work systematically examines how \ac{rl} adaptation affects multimodality. Existing research splits broadly into two directions. A first line of work focuses on \ac{rlft} of expressive policies such as diffusion or flow models to improve robustness and returns~\citep{park2025flow,ren2024diffusion,chen2024diffusion}. These approaches, however, do not account for multimodality in the action distribution, and often collapse the diverse behaviors captured during demonstration into a single dominant strategy. A second line of work begins to address multimodality more explicitly, for instance by proposing metrics to characterize it~\citep{jia2024towards} or by leveraging language conditioning and instruction diversity~\citep{black2410pi0,kim2024openvla}. Yet, these efforts either rely on assumptions that the number of modes is known in advance or that multimodality can be fully captured through language and labels. In practice, the modalities contained in the demonstration are usually unknown, and language provides only a coarse handle on behavior, which prevents precise encoding of low-level motor attributes such as magnitudes, scales, and endpoints~\citep{lee2025molmoact}. %, leaving much of the multimodal structure unaddressed during RL adaptation. 
Despite growing interest in multimodal policies, the effect of \ac{rlft} on behavioral diversity remains underexplored. Prior work falls broadly into two directions. A first line studies \ac{rlft} of generative policies such as diffusion or flow models to improve robustness and task success~\citep{park2025flow,ren2024diffusion,chen2024diffusion}. However, these methods typically optimize task reward without explicitly accounting for multimodality, and can induce mode collapse, eliminating behaviors learned from demonstration. A second line addresses multimodality more directly, for instance, by proposing diversity metrics~\citep{jia2024towards} or leveraging language conditioning and instruction diversity~\citep{black2410pi0,kim2024openvla}. Yet, these approaches often assume that the number of modes is known \emph{a priori} or that language and labels provide a sufficient description of multimodal structure. In practice, the modes present in demonstrations are rarely known, and language offers only coarse semantic control, making it difficult to preserve fine-grained motor attributes such as magnitudes, temporal scales, and endpoints~\citep{lee2025molmoact}. % , leaving much of the multimodal structure unaddressed during \ac{rl} adaptation. 

We propose \methodname{} (\emph{\textbf{B}ehavioral \textbf{M}ode \textbf{D}iscovery}), a framework for \ac{rlft} that preserves multimodal behavior by uncovering latent behavioral modes in pre-trained generative policies. 
We begin by introducing a tractable proxy for multimodality in noise-conditioned generative policies 
based on the mutual information between the latent noise and the induced trajectories. Inspired by unsupervised skill discovery~\citep{gregor2016variational, eysenbach2018diversity}, we then design a mode-discovery procedure that exposes trajectory-level modes implicitly encoded in the latent noise without annotations or prior knowledge, yielding a controllable latent representation and a practical mutual-information estimate. Finally, we use this estimate as an intrinsic reward during \ac{rlft}, regularizing fine-tuning to prevent mode collapse while maintaining task performance. Our experiments show that across diverse multimodal robotic tasks, our approach matches or improves task success relative to standard fine-tuning while preserving multimodality.

% In summary, our contributions are: \textbf{1)} A mutual-information--based proxy for quantifying multimodality in noise-conditioned generative policies, without mode labels or language supervision. \textbf{2)} An unsupervised mode-discovery procedure that identifies and makes controllable latent behavioral modes in pre-trained generative policies. \textbf{3)} A mode-preserving \ac{rlft} objective that uses a mutual-information intrinsic reward to regularize fine-tuning and prevent mode collapse while maintaining task performance. \textbf{4)} An empirical evaluation on multimodal robotic manipulation tasks demonstrating improved task success and stronger multimodality preservation compared to standard fine-tuning.
In summary, our contributions are:
\textbf{1)} A mutual-information--based proxy for measuring multimodality in noise-conditioned generative policies. % without mode labels or language supervision.
\textbf{2)} An unsupervised procedure for discovering and controlling latent behavioral modes in pre-trained generative policies.
\textbf{3)} A mode-preserving \ac{rlft} objective that uses a mutual-information intrinsic reward to regularize fine-tuning and prevent mode collapse while maintaining task performance.
\textbf{4)} An empirical evaluation on multimodal robotic manipulation and locomotion tasks showing improved success and stronger multimodality preservation over standard fine-tuning.

\section{Related Work}
\label{sec:rw}
\label{sec:fine-tuning_tech}

We briefly review two areas closely connected to our central idea and contributions. For a more comprehensive discussion on related work, see Appendix~\ref{appendix:rw}.

\paragraph{Fine-tuning of Pre-trained Generative Policies.}
\label{sec:fine-tuning_tech}

Diffusion- and flow-based models provide expressive policy parameterizations for multimodal action distributions, but \ac{rlft} is challenging due to sequential sampling and the cost of backpropagating through the generative process. Recent work addresses these issues through three main strategies: direct fine-tuning, residual policies, and steering policies. 
\emph{Direct fine-tuning} approaches adapt the network weights either by distilling the model into a one-step sampler for easier backpropagation~\citep{park2025flow, chen2024diffusion}, by casting the denoising process as a sequential decision problem~\citep{ren2024diffusion}, or by using differentiable approximations that allow offline Q-learning without backpropagating through all denoising steps~\citep{kang2023efficient}. %Despite their promise, such approaches often collapse to a single reward-maximizing mode. 
% adapt generative models to RL by directly biasing action samples toward high-value regions of the learned $Q$-function by either distilling the action sampling to enable one-step sampling which is more amenable for backpropagarion~\citep{park2025flow,chen2024diffusion}, reframing the diffusion as a decision-making problem~\cite{ren2024diffusion}, or thorugh differentiable approximations to the sampling process, enabling policy optimization via offline Q-learning while avoiding full gradient flow through the denoising trajectory~\cite{kang2023efficient}. These methods, however, often suffer from mode collapse.
\emph{Residual policy} learning methods instead freeze a pre-trained generative policy and learn a small corrective controller via \ac{rl} to address execution errors \citep{ankile2024imitation,yuan2024policy}. These techniques can yield substantial performance gains over pure \ac{il}, and potentially preserve the diversity learned from demonstrations.
\emph{Steering policy} methods instead bias the sampling process toward high-value actions without modifying the generative model itself~\cite{wagenmaker2025steering,yang2023policy,wang2022diffusion}.
% Some methods directly adjust training data or sampled actions using Q-values, either by nudging demonstration actions toward higher values~\citep{yang2023policy} or by combining diffusion with Q-learning to bias samples while staying close to the demonstration manifold~\citep{wang2022diffusion}.
% For example, ~\cite{wagenmaker2025steering} proposed to learn to control the latent noise of generative models, guiding the sampling process toward regions of the noise space whose denoised actions yield higher reward.
A common limitation of the aforementioned approaches is that they lack explicit mechanisms to preserve multimodality, and often converge to a single reward-maximizing behavior. Our work extends the steering-policy framework by discovering and controlling latent modalities of noise-conditioned generative policies, allowing us to regularize \ac{rlft} to preserve multimodality while adapting to the downstream task.

\paragraph{Skill Discovery.}
Multimodal behavior learning has also been studied through unsupervised skill discovery, which aims to acquire diverse and distinguishable behaviors without external rewards. A common approach is to maximize mutual information between a latent skill variable and visited states or trajectories~\citep{gregor2016variational, eysenbach2018diversity}. Most existing methods train policies from scratch in reward-free settings, but diversity alone often leads to skills that may be ill-suited for downstream tasks. To address this, prior work has incorporated language guidance~\citep{LGSD}, extrinsic rewards~\citep{SLIM}, or state-space regularization~\citep{CSD}. Our approach differs by leveraging a pre-trained generative model to uncover useful behaviors already encoded in demonstrations. To our knowledge, we are the first to study skill discovery in this context, treating skills as modes in the latent noise space of a pre-trained generative policy.

\section{Problem Formulation}
\label{sec:problem_formulation}

We study \ac{rlft} of a pre-trained noise-conditioned generative policy to maximize expected return while preserving the multimodality acquired from demonstrations. Here, multimodality refers to the presence of multiple distinct high-probability behaviors, which may arise from heterogeneous task goals and/or multiple feasible trajectories that solve the same goal. We model the environment as a Markov Decision Process (MDP) $(\mathcal{S}, \mathcal{A}, r, p, \gamma)$ with state space $\mathcal{S}$, action space $\mathcal{A}$, reward function $r$, transition dynamics $p$, and discount factor $\gamma \in [0,1)$. The objective of \ac{rl} is to learn a policy $\pi_\theta(a \mid s)$ maximizing the expected discounted return
\[
    J(\pi) = \mathbb{E}_{\pi}\!\left[ \sum_{t=0}^\infty \gamma^t r(s_t,a_t)\right],\]

where $s_t$ and $a_t$ are distributed according to the transition dynamics $p$ and the policy $\pi$.

% In our setting, multimodality arises within a single task: the policy’s trajectory distribution may place mass on multiple distinct solutions, where each mode corresponds to a self-consistent behavioral strategy that successfully accomplishes the objective. 

%%%%%%%%%%%%%%%%%%%%%%%%%%%%%%%%%% OLD %%%%%%%%%%%%%%%%%%%%%%%%%%%%%%%%%%%%%%%
% \paragraph{Pre-trained Multimodal Action Distirbutions.} %\al{Here we discuss mode in the trajectories but not in the action distribution.}
% We assume access to an offline \new{}dataset of state-action pairs 
% ${\mathcal{D} = \{(s_t, a_t)\}_{i=1}^N}$  collected by diverse behavioral policies (e.g.\, human demonstrations), which is used to pre-train a generative policy $\pi_\theta(a \mid s)$ via imitation learning.  We define a mode of the policy $\pi_\theta$ as a latent variable $z \in \mathcal{Z}$, implicitly encoded in the pre-trained multimodal policy, which induces the trajectory distribution ${ p^\pi(\tau \mid z) = p(s_0)\prod_{t=0}^{T-1} \pi(a_t \mid s_t, z)\, p(s_{t+1}\mid s_t,a_t)}$, so that different values of $z$ correspond to distinct self-consistent strategies that solve the task, i.e., different modes. We assume the original modes $z\in\mathcal{Z}$ contained in the datasets are unknown. When relevant, we make explicit the dependence of the generative policy on its input noise variable  $w \in \mathcal{W}$ by denoting it as $\pi_\theta(a \mid s, w)$.

%%%%%%%%%%%%%%%%%%%%%%%%%%%%%%%%%%%%%%%%%%%%%%%%%%%%%%%%%%%%%%%%%%%%%%

\paragraph{Pre-Trained Multimodal Generative Policies.}%Action Distirbutions.} %\al{Here we discuss mode in the trajectories but not in the action distribution.}
We assume access to an offline demonstration dataset ${\mathcal{D}
= \{\tau^{(i)}\}_{i=1}^N, \tau^{(i)} = (s^{(i)}_0,a^{(i)}_0,\dots,s^{(i)}_{T_i},a^{(i)}_{T_i})}$, collected with diverse behavioral policies (e.g.\, human demonstrations), used to pre-train a generative policy $\pi_\theta$ via imitation learning. We further assume that the resulting policy is \emph{multimodal}, in the sense that it assigns non-negligible probability mass to multiple distinct behaviors reflected in the dataset~\cite{chi2023diffusion}. When relevant, we make explicit the dependence of the generative policy on its input noise variable $w \in \mathcal{W}$ by denoting it as $\pi_\theta(a \mid s, w)$.

\paragraph{Behavioral Modes.}
We represent behavioral modes with a discrete latent variable $z \in \mathcal{Z}$, where each value indexes a policy instance inducing the trajectory distribution
$
p^\pi(\tau \mid z) = p(s_0)\prod_{t=0}^{T-1} \pi(a_t \mid s_t, z)\, p(s_{t+1}\mid s_t,a_t).
$
Different $z$ correspond to distinct self-consistent strategies (behavioral modes) in the data. Although we focus on discrete $z$, continuous latents (e.g., $z \in \mathbb{R}^d$) are also possible. We assume the true modes are unknown but implicitly encoded in the pre-trained multimodal policy.

%Following~\citep{wagenmaker2025steering}, we assume the pre-trained policy is \emph{steerable} through a latent input noise variable $w \in \mathcal{W}$. %$w \in \mathcal{W} := \mathbb{R}^d$. 
% Specifically, given $w$, a diffusion or flow-based policy can be rewritten by making explicit the dependence of the action on the input noise as $\pi_{\theta}(a \mid s, w)$. 

% \paragraph{Steerability Assumption.} We assume that the pre-trained generative policy $\pi_\theta(a \mid s, w)$ is \emph{steerable}, in the sense that its behavior can be systematically influenced through the choice of the latent noise input $w \in \mathcal{W}$. A \emph{steering policy} $\pi_\psi^{\mathcal{W}}(w \mid s)$, parameterized by $\psi$, selects which point $w$ in the latent-noise space to denoise, %acts in the latent space $\mathcal{W}$ and  
% biasing the generative model toward different behavioral modes~\citep{wagenmaker2025steering}. %By choosing which point in the latent-noise space to denoise, the steering policy can bias the generative model toward different behavioral modes.

\paragraph{Steerability Assumption.} We assume that the pre-trained generative policy $\pi_\theta(a \mid s, w)$ can be \emph{steered} by controlling its latent noise input $w \in \mathcal{W}$. We introduce a \emph{steering policy} $\pi_\psi^{\mathcal{W}}(w \mid s)$ that maps the current state to a distribution over noise inputs, thereby biasing the actions produced by $\pi_\theta$ toward different behaviors~\citep{wagenmaker2025steering}. 

%selects noise variables conditioned on the current state, thereby indirectly shaping the action distribution of $\pi_\theta$. selecting which point in the latent-noise space to denoise, we can steer the action produced by $\pi_\theta$ to a desired mode.  %Steerability also requires that the generative process preserves dependencies between the input noise and the generated actions~\citep{domingo2024adjoint}. %; in particular, the noise schedule must not be memoryless.
%We assume that the pre-trained policy is \emph{steerable} through its latent input noise variable $w \in \mathcal{W}$. A \emph{steering policy} $\pi_\psi^{\mathcal{W}}(w \mid s)$  parameterized by $\psi$ is a policy that selects latent noise variables, thereby indirectly controlling the behavior of $\pi_\theta$. While we do not assume that the pretrained policy has to be deterministic to be steerable, the noise schedule must not be memoryless~\citep{domingo2024adjoint}, meaning that the dependency between noise variables and the generated samples is preserved throughout the generation process.

\paragraph{Fine-tuning Objective.}
Our goal is to fine-tune the policy $\pi_{\theta}$ in order to (i) maximize the expected return and (ii) preserve the multimodality present in the pre-trained policy $\pi_\theta$. 
We formalize this as the regularized optimization problem
\[
    \max_\theta \; J(\pi_{\theta}) + \lambda \, \mathcal{M}(\pi_{\theta}),
\]
where $\mathcal{M}$ denotes a multimodality measure of the generative policy, %induced action distribution, 
and $\lambda \ge 0$ balances task performance with diversity preservation. Importantly, we do not assume prior knowledge of the number of modes in $\pi_\theta$. %, nor do we rely on differentiability of the policy to locate local maxima.
%Designing a practical measure for multimodality under these constraints is a central contribution of this work.

% \paragraph{Fine-tuning Objective.}
% Let $\pi_{\theta_0}$ denote a pre-trained generative policy. Our goal is to fine-tune its parameters $\theta$ to (i) maximize expected return and (ii) preserve the multimodality present in $\pi_{\theta_0}$. We formalize this as
% \[
%     \max_{\theta}\; J(\pi_{\theta}) + \lambda \, \mathcal{M}(\pi_{\theta}),
% \]
% where $\mathcal{M}$ measures multimodality of the generative policy and $\lambda \ge 0$ trades off task performance and diversity preservation. In the following, we introduce auxiliary models (e.g., a steering policy) to obtain a tractable estimate of $\mathcal{M}(\pi_{\theta})$ and an intrinsic reward for optimization.
\section{Mode Discovery for \ac{rl} Fine-tuning}
\label{sec:method}

To prevent mode collapse during \ac{rlft} of pre-trained generative policies, our method explicitly discovers and controls latent behavioral modes of the policy. This allows us to regularize \ac{rlft} with a trajectory-diversity objective while maintaining task performance. Our framework comprises three components:
(i) a tractable mutual-information proxy $\mathcal{M}(\cdot)$ for quantifying multimodality;
(ii) an unsupervised mode-discovery mechanism that reparameterizes a steering policy $\pi_{\psi}^{\mathcal{W}}(w \mid s)$ with a latent variable $z \in \mathcal{Z}$ to uncover and control modes while estimating multimodality;
(iii) a mutual-information intrinsic reward combined with task rewards to regularize \ac{rlft}, explicitly retaining diverse behaviors. We describe each component below.

\begin{figure*}[t]
    \centering
    \includegraphics[width=0.98\textwidth]{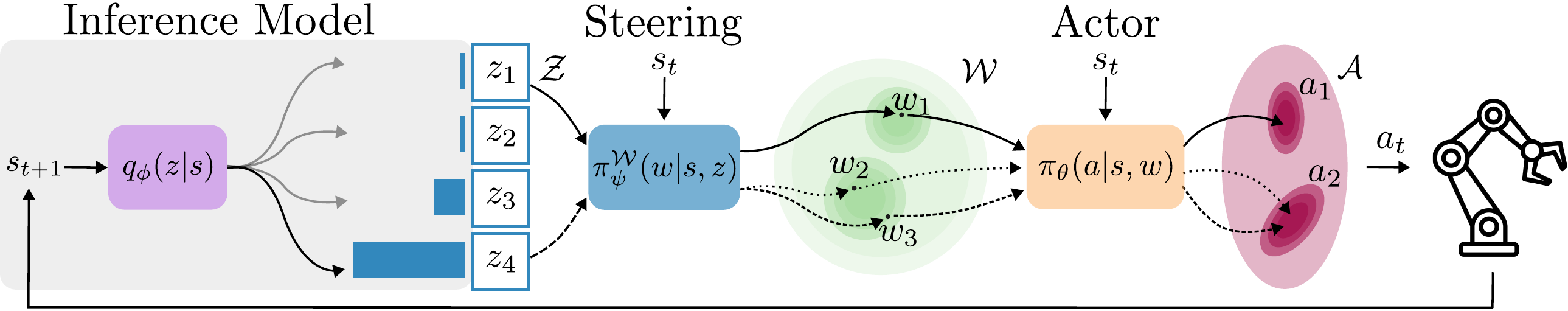}
    \caption{ 
    \textbf{Behavioral Mode Discovery via Latent Reparameterization of a Steering Policy.}  
Joint training of an inference model $q_\phi(z \mid s)$ and a steering policy $\pi^{\mathcal{W}}_\psi(w \mid s,z)$ to uncover latent modes $z \in \mathcal{Z}$ of a frozen policy $\pi_\theta(a \mid s,w)$.  
The steering policy structures the latent noise space by sampling $w\sim \mathcal{W}$ according to $z$, biasing towards diverse $a \in \mathcal{A}$. The inference model identifies $z$ given the visited state $ s _ {t+1} $, estimating $I(Z; S)$ (Eq.~\ref{eq:variational_MI}). This estimate is used as an intrinsic reward during mode discovery and \ac{rlft}.
    }
    \label{fig:method}
\end{figure*}

\subsection{Measuring Multimodality in Generative Policies}

% \reb{
%We assume access to a pre-trained generative policy that is inherently multimodal, reflecting the diverse behavioral strategies present in the demonstration dataset. 
% Classical definitions of multimodality characterize modes as local maxima of the explicit action distribution ${\pi_\theta(a \mid s)}$~\citep{stoepker2016testing}, but this view becomes impractical for diffusion or flow-based policies, whose action densities are not available in closed form.
% Under the assumption taht  the multimodality of the pre-trained policy $\pi_\theta(a \mid s)$ is realized through this latent noise, in 
% the sense that there exists a set of states of non-zero measure for which different values of $w$ 
% induce different action distributions, i.e.,  $\pi_\theta(\cdot \mid s, w_1) \neq \pi_\theta(\cdot \mid s, w_2)$ for some $w_1 \neq w_2$.  it is possible to show that the latent $W$ and the action $A$ are statistically dependent given $S$, and 
% therefore the conditional mutual information is strictly positive
% (proof in Appendix~\ref{sec:proof_MI_positive})
% \[
%     I(W; A \mid S) = \mathbb{E}_{s \sim p(s)}\!\left[D_{\mathrm{KL}}\!\big(\pi_\theta(a \mid s, w)\,\|\, p(a \mid s)\big)\right] > 0,
% \]
% where $p(a \mid s) = \mathbb{E}_{w \sim \mathcal{W}}[\pi_\theta(a \mid s, w)]$ is the marginal action distribution. 

Classical notions of multimodality define modes as local maxima of an explicit action density 
$\pi_\theta(a \mid s)$~\citep{stoepker2016testing}. This definition becomes impractical for 
modern generative policies such as diffusion or flow-based models, whose action densities are unavailable in closed form.
We instead characterize multimodality through the latent noise $\mathcal{W}$ of a pre-trained 
generative policy ${\pi_\theta(a \mid s, w)}$ with $w \sim \mathcal{W}$. 
We assume that multimodality is realized through this latent representation, in the sense that 
there exists a set of states of non-zero measure for which different latent values induce distinct 
action distributions, i.e.,
${\pi_\theta(\cdot \mid s, w_1) \neq \pi_\theta(\cdot \mid s, w_2)}$ for some $w_1 \neq w_2$. Under this assumption, it is possible to show that the latent variable $W$ and the action $A$ are statistically dependent conditioned on the state $S$. 
As a consequence, the conditional mutual information between $W$ and $A$ given $S$ is strictly positive (proof in Appendix~\ref{sec:proof_MI_positive})
\[
\tiny
I(W; A \mid S)
= \mathbb{E}_{s,w}\!\left[
    D_{\mathrm{KL}}\!\big(\pi_\theta(\cdot \mid s, w)\,\|\, p(\cdot \mid s)\big)
\right] > 0,
\]
% \begin{equation}
% {\small
% \begin{aligned}
% I(W;A\mid S)
% &=
% \mathbb{E}_{s,w}
% \Big[
% D_{\mathrm{KL}}\big(
% \pi_\theta(\cdot\mid s,w)
% \,\|\, 
% p(\cdot\mid s)
% \big)
% \Big] \\
% &> 0,
% \qquad
% s\sim p(s),\; w\sim p(w\mid s).
% \end{aligned}
% }
% \end{equation}
with $s\sim p(s),\; w\sim p(w\mid s)$ and where ${p(a \mid s) = \mathbb{E}_{w \sim \mathcal{W}}[\pi_\theta(a \mid s, w)]}$ is the marginal action distribution.

Although $I(W; A \mid S) > 0$ provides a valid proxy for multimodality in the pre-trained action
distribution, it does not guarantee multimodality in the induced trajectories during fine-tuning.
Action-level diversity does not necessarily translate into trajectory-level diversity, as distinct
actions may lead to similar state transitions under the environment dynamics. For instance, in a
kinematically redundant manipulator, multiple joint-space actions can yield nearly identical
end-effector motions, collapsing multimodality in action space into a single mode in trajectory
space. %Consequently, this quantity alone is insufficient to guarantee the preservation of behavioral diversity during fine-tuning.} 

To address this limitation, we quantify multimodality at the trajectory level via the mutual information between the latent noise variable $W$ and visited states, $I(W; S)$, which captures diversity in the induced behaviors. Under the environment dynamics, actions mediate the influence of $W$ on future states, forming the Markov chain $W \rightarrow A \rightarrow S'$, and the data processing inequality yields $I(W; S') \le I(W; A)$. Consequently, state-based mutual information provides a conservative estimate of multimodality, retaining only distinctions that persist through the dynamics. Henceforth, with slight abuse of notation, $I(W; S)$ denotes the mutual information between $W$ and the state-visitation distribution induced by trajectories, omitting time indices for clarity.

This perspective is inspired by the unsupervised skill discovery literature, which similarly
emphasizes trajectory diversity through mutual information between latent variables and visited
states~\citep{gregor2016variational,eysenbach2018diversity,sharma2019dynamics}. Unlike prior work,
however, which seeks to \emph{learn} a latent skill space from scratch, our setting starts from a
pre-trained generative policy whose latent noise space $\mathcal{W}$ already encodes multiple
behavioral modes. 
% These modes, however, are implicit in the structure of $\mathcal{W}$.
% These modes, however, are not explicitly represented in $\mathcal{W}$ as the noise variable $w \in \mathcal{W}$ is
% time-local and varies at each time step, whereas the behavioral modes we are interested in emerge at the trajectory level. As a result, the mode structure is implicit in
% $\mathcal{W}$ and not directly accessible through $W$ alone.
% In the next section, we introduce a method to uncover and control
% these latent behavioral modes, effectively lifting step-wise noise into coherent trajectory-level
% structure.
However, these modes are not explicitly represented in $\mathcal{W}$: the noise variable $w \in \mathcal{W}$ is time-local and varies at each step, whereas the behavioral modes of interest emerge at the trajectory level. Consequently, mode structure is only implicit in $\mathcal{W}$ and not directly accessible from $\mathcal{W}$ alone. In the next section, we introduce a method to uncover and control these latent modes, lifting step-wise noise into coherent trajectory-level structure.

\subsection{Behavioral Mode Discovery of Generative Policies}

To explicitly uncover and organize the behavioral modes implicit in the latent noise space
$\mathcal{W}$, we introduce \methodname{} (\emph{\textbf{B}ehavioral \textbf{M}ode \textbf{D}iscovery}). Specifically, \methodname{} reparameterizes a steering policy with a latent variable $z \in \mathcal{Z}$ that organizes $\mathcal{W}$ into trajectory-level modes. 

\paragraph{Latent Reparameterization.}

Let $\pi_\psi^{\mathcal{W}}(w\mid s)$ denote a steering policy that selects the latent noise $w\in\mathcal{W}$ seeding the generation process. We introduce a latent variable $z\in\mathcal{Z}$ and define a latent-conditioned steering policy $\pi_\psi^{\mathcal{W}}(w\mid s,z)$, which induces the family of action distributions
\begin{equation*}
    \pi_{\theta,\psi}(a\mid s,z) \;=\; \int \pi_\theta(a\mid s,w)\,\pi_\psi^{\mathcal{W}}(w\mid s,z)\,dw.
    \label{eq:steered_policy}
\end{equation*}
% Under this reparameterization, multimodality, \reb{or trajectory-level diversity,} can be measured by the mutual information \reb{$I(Z;S)$,  which places us back in the standard skill-discovery setting and allows us to leverage this class of methods to optimize the steering policy.
% By training the steering policy to maximize $I(Z;S)$, we encourage different values of $z$ to induce distinct state–trajectory distributions by steering the sampling of $\mathcal{W}$, thereby uncovering the modes implicitly encoded in the latent noise space.} Distinct values of $z$ can therefore select different behaviors (modes) encoded by the \emph{fixed} pretrained policy $\pi_\theta(a\mid s,w)$ through the steering policy $\pi_\psi^{\mathcal{W}}(w\mid s,z)$. 

Under this reparameterization, multimodality, or trajectory-level diversity, can be quantified by
the mutual information $I(Z; S)$ between the latent code and the visited states. Combined with the
assumption that $I(W; S) > 0$, this formulation allows us to discover the behavioral modes implicitly
encoded in the latent noise space by training a steering policy to maximize $I(Z; S)$. Maximizing
this objective encourages different values of $z$ to steer the sampling of $\mathcal{W}$ toward
distinct state–trajectory distributions, thereby inducing coherent and identifiable behavioral
modes. %However, $I(Z; S)$ is not directly tractable to optimize. 

% While $I(Z; S)$ defines the desired objective, it is not directly tractable to optimize. In the
% following, we therefore derive a variational lower bound that yields a practical learning signal
% for training the steering policy. Distinct values of $z$ can then be used to selectively invoke
% different behaviors (modes) of the \emph{fixed} pre-trained policy
% $\pi_\theta(a \mid s, w)$ through the latent-conditioned steering policy
% $\pi_\psi^{\mathcal{W}}(w \mid s, z)$.

\paragraph{Variational Lower Bound.}
To optimize $\pi_\psi^{\mathcal{W}}$ via $I(Z;  S)$
we follow standard practice in skill discovery~\citep{eysenbach2018diversity}, which derives a variational lower bound on $I(Z;  S)$ by introducing an inference model $q_\phi(z \mid s)$ that approximates the posterior over latent codes. This yields

\begin{align}
    I(Z; S) 
    &= \mathbb{E}_{p(s,z)} \left[ \log \frac{p(z \mid s)}{p(z)} \right] \nonumber\\
    &\geq \mathbb{E}_{p(s,z)} \left[ \log q_\phi(z \mid s) - \log p(z) \right],
    \label{eq:variational_MI}
\end{align}

where $p(s,z)$ denotes the joint distribution induced by sampling 
$z \sim p(z)$ and rolling out a policy $\pi(a \mid s, z)$ in the environment.
We refer to ~\cite{eysenbach2018diversity} for the derivation of Equation~\ref{eq:variational_MI}.

\paragraph{Steering Policy Training.} The log-posterior likelihood in Equation~\ref{eq:variational_MI} is used as an intrinsic reward for training the steering policy $\pi_\psi^{\mathcal{W}}$, aligning the RL objective with the identifiability of $z$. This establishes a feedback loop in which $q_\phi$ improves at classifying latent codes while the $\pi_\psi^{\mathcal{W}}$ is incentivized to select $w\in \mathcal{W}$ so as to induce trajectories that are consistent and discriminable. 

As a result, the proposed method explicitly uncovers and organizes the behavioral modes implicit in
the pre-trained policy, yielding a trajectory-level representation that both quantifies
multimodality, via Equation~\ref{eq:variational_MI}, and allows different modes to be controlled via $z$. An overview of the mode
discovery process is shown in Figure~\ref{fig:method}.

% The proposed method simultaneously enables controllability over the latent space and the measurement of multimodality at the trajectory level. An overview of the method for mode discovery is illustrated in Figure~\ref{fig:method}.

\subsection{Regularized Policy Fine-tuning}
% Recall from Section~\ref{sec:problem_formulation} that we formulated the fine-tuning objective as maximizing task return regularized by a multimodality measure $\mathcal{M}$.  
% The variational lower bound introduced in Equation~\ref{eq:variational_MI} provides a tractable instantiation of $\mathcal{M}$, which allows us to regularize the fine-tuning objective as define in Section~\ref{sec:problem_formulation} by leveraging it as an intrinsic signal to preserve multimodality during \ac{rlft}.  We define the augmented reward

We leverage the variational lower bound from Equation~\ref{eq:variational_MI} as a tractable instantiation of $\mathcal{M}$ and use it as an intrinsic signal to regularize the fine-tuning objective defined in Section~\ref{sec:problem_formulation}. We define the augmented reward
\begin{equation}
    r_{\text{total}}(s,z) = r_{\mathrm{env}}(s,a) 
    + \lambda \Big( \log q_\phi(z \mid s) - \log p(z) \Big),
\end{equation}
where $r_{\mathrm{env}}$ is the environment reward and $\lambda \geq 0$ balances task performance with multimodality preservation.
Directly combining task and intrinsic rewards may lead to mode collapse if the task signal dominates before the multimodal structure is discovered. We therefore adopt a two-stage scheme: first, the steering policy is trained with the intrinsic objective alone to uncover the modes of the pretrained policy; then, the environment reward is introduced to guide fine-tuning toward high-return behaviors without destroying diversity.  During mode discovery, we also apply a short-to-long horizon curriculum to stabilize learning.
Algorithm~\ref{alg:mode_finetuning} summarizes the procedure (details in Appendix~\ref{appendix:algorithm}). While we use PPO~\citep{schulman2017proximal} in our experiments, the framework is agnostic to the underlying RL algorithm.

% \paragraph{Broader Use of the Framework.}
% While the formulation above fine-tunes the generative model indirectly via the steering policy, the framework is not limited to this case. Because the steering head actively explores diverse input-noise regions while pursuing reward, it can be combined with direct fine-tuning of the diffusion weights, acting as a structured exploration agent. At test time, the steering policy can either be retained—allowing explicit control over the behavioral mode—or removed, reverting to random sampling from the noise prior. Furthermore, while outside the present study, the learned latent space $\mathcal{Z}$ provides a natural basis for grounding semantic labels (e.g.\, language instructions) when limited annotations are available. 

\paragraph{Broader Use of the Framework.}
\textbf{1)} We discussed fine-tuning the generative model indirectly through the steering policy. However, the proposed framework is agnostic to the choice of fine-tuning mechanism: the steering head can be used as a form of structured exploration over the noise space and coupled with direct optimization of the diffusion parameters. \textbf{2)} At test time, the steering policy can either be retained to enable explicit control over behavioral modes or removed to revert to sampling from the noise prior. \textbf{3)} While outside the scope of this work, the learned latent space $\mathcal{Z}$ also provides a natural basis for grounding semantic labels (e.g., language instructions) when limited annotations are available.

\begin{algorithm}[tb]
\small
\caption{Mode Discovery and Fine-Tuning of Generative Policies}
\label{alg:mode_finetuning}
\begin{algorithmic}
\STATE {\bfseries Input:} pre-trained diffusion policy $\pi_{\theta}(a\mid s,w)$;
steering policy $\pi^{\mathcal{W}}_{\psi}(w\mid s,z)$;
inference model $q_{\phi}(z\mid s,a)$;
critic $V_{\omega}(s,z)$;
latent prior $p(z)$;
epochs $E$, episodes $N$, warm-up epochs $E_{\text{wp}}$;
horizon scheduler $H_{\text{schedule}}(e)$

\STATE {\bfseries Init:} $\psi,\phi,\omega$; set $\lambda \ge 0$

\FOR{$e = 1$ {\bfseries to} $E$}
  \FOR{$n = 1$ {\bfseries to} $N$}
    \STATE $H \gets H_{\text{schedule}}(e)$
    \STATE Sample $z \sim p(z)$ and rollout the policy:
    \STATE $w_t \sim \pi^{\mathcal{W}}_{\psi}(w\mid s_t,z)$
    \STATE $a_t \sim \pi_{\theta}(a\mid s_t,w_t)$
    \STATE $s_{t+1}\sim p(\cdot\mid s_t,a_t)$
    \STATE Intrinsic reward:
    \STATE $r^{\text{int}}_t \gets \lambda(\log q_{\phi}(z\mid s_{t+1}) - \log p(z))$
    \IF{$e < E_{\text{wp}}$}
      \STATE $r^{\text{tot}}_t \gets r^{\text{int}}_t$
    \ELSE
      \STATE $r^{\text{tot}}_t \gets r_{\text{env}}(s_t,a_t) + r^{\text{int}}_t$
    \ENDIF
  \ENDFOR
  \STATE Update steering policy and critic using $r^{\text{tot}}_t$ (PPO):
  \STATE $\min_{\psi,\omega} L_{\pi}^{\text{PPO}}(\psi) + c_V L_V(\omega) + c_{\mathcal{H}} L_{\mathcal{H}}(\psi)$
  \STATE Update inference model:
  \STATE $\min_{\phi} L_q(\phi) = -\mathbb{E}[\log q_{\phi}(z\mid s)]$
\ENDFOR
\end{algorithmic}
\end{algorithm}

\section{Experiments}

Our experimental evaluation is centered around three main questions: (i) Is the mutual information in Equation~\ref{eq:variational_MI} a valid estimate of multimodality? (ii) Do existing fine-tuning techniques preserve multimodality? (iii) Does our method retain multimodality without sacrificing task performance? We evaluate \methodname{} across increasingly challenging domains: a 2D Gaussian-mixture reward landscape; multimodal manipulation in ManiSkill~\citep{tao2024maniskill3} and D3IL~\citep{jia2024towards}; and high-dimensional and sequential control (ANYmal locomotion, Franka Kitchen~\citep{fu2020d4rl}). We also ablate key design choices.

\paragraph{Baselines.}
% Following the characterization introduced in Section~\ref{sec:fine-tuning_tech},
We benchmark our approach against representative strategies for on-policy fine-tuning of generative policies, focusing on diffusion models but noting that analogous evaluations apply to flow-matching policies. We include \texttt{DPPO}~\citep{ren2024diffusion}, as a direct finetuning approach, Policy Decorator~\citep{yuan2024policy} as a residual fine-tuning approach (\texttt{RES}), and we consider \texttt{DSRL}~\cite{wagenmaker2025steering} as a steering policy-based approach. For \texttt{DPPO}, we select DDIM parameterization to ensure steerability while balancing $\eta>0$ and the number of diffusion steps for stable weight updates. We further include a DDPM-based version that samples with the full denoising chain and fine-tunes the last $10$ diffusion steps for completeness (\texttt{DPPO[10]}).
Importantly, our approach is orthogonal to these categories and can be combined with any of them.
Therefore, we report results both for the standalone baselines and their variants augmented with our multimodality regularizer, denoted as \texttt{X[\methodname{}]}, where \texttt{X} indicates the corresponding baseline. Full implementation details for all baselines and their regularized variants are provided in Appendix~\ref{appendix:implementation_details}.

\paragraph{Evaluation Metrics.}
% We assume access in simulation to the ground truth modes of the trajectories executed by the policy, and we evaluate fine-tuned policies along two axes: \emph{task success} and \emph{behavioral diversity}.

We assume access to the ground truth modes of the trajectories executed by the policy in simulation, and we evaluate fine-tuned policies along two axes: \emph{task success} and \emph{behavioral diversity}. We report the overall success rate $\mathrm{SR}$, and two mode-aggregated measures of the success rate to integrate behavioral diversity: the success rate weighted for each mode
$
{\mathrm{SR}_{\text{M}}=\tfrac{1}{K}\sum_{i=1}^K \mathrm{SR}_i,}
$
which guards against degenerate solutions (e.g., \(100\%\) success on a single mode but failure on others), and
mode coverage ${\mathrm{mc}@\tau=\tfrac{1}{K}\sum_{i=1}^K \mathbf{1}\{\mathrm{SR}_i \ge \tau \}}$,
the fraction of modes solved above threshold \(\tau=0.8\). 
We further compute the entropy of the empirical distribution over modes among all rollouts: $H(\pi)=-\sum_{i=1}^K p_i\log p_i$, where $p_i$ is the fraction of episodes in mode $i$. %A higher entropy reflects more balanced usage of the available modes, whereas a reduction after fine-tuning is indicative of mode collapse. 
All metrics are computed from $N=1024$ evaluation episodes with fixed seeds for fair comparison, and we report both the mean and standard deviation over three independent runs with different random seeds. 

\begin{figure*}[t]
  \centering
  \includegraphics[width=0.88\textwidth]{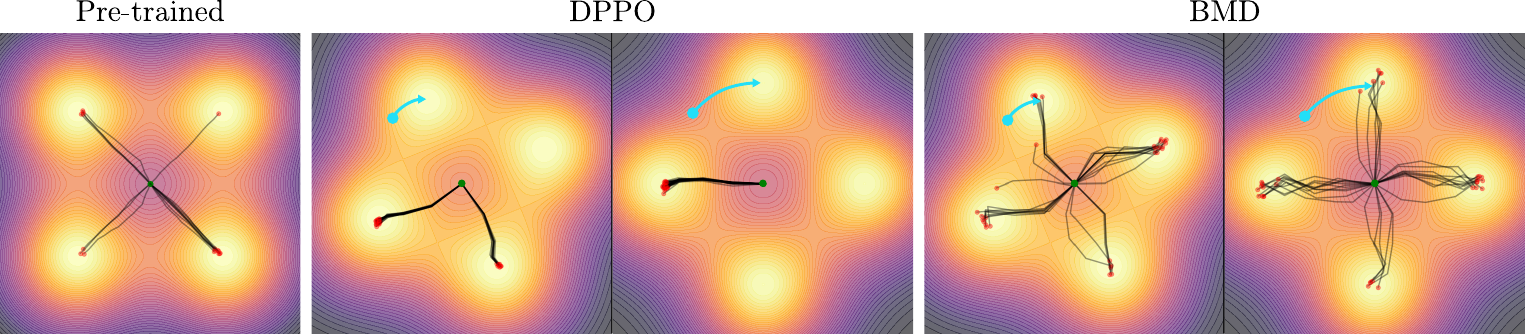}
  \caption{ 
  \textbf{Qualitative trajectories for two rotated reward landscapes.} Each panel overlays policy rollouts originating from the center of the environment, in a four-goal reward landscape (bright peaks). \emph{Left:} the pre-trained policy covers all modes. \emph{Middle:} \texttt{DPPO} fine-tuning collapses to a subset of modes after reward rotation. \emph{Right:} \texttt{\methodname{}} fine-tunes the policy while maintaining full mode coverage.}
  \label{fig:multimodality_comparison}
\end{figure*}

\subsection{2D Gaussian Mixture}
\label{sec:2D_Gaussians}
To study the proposed questions in a controlled setting, we designed a 2D navigation environment where the reward landscape is a mixture of $4$ Gaussians centered at fixed goal locations.
% The agent starts at the origin and moves by 2D displacements, receiving rewards defined by the Gaussian peaks.  
%We study both a \emph{balanced} variant, where all goals have equal weight, and an \emph{unbalanced} variant, where mode weights are randomized and normalized via a softmax, producing uneven but non-degenerate reward magnitudes.  
Further details and illustrations about the environment are provided in Appendix~\ref{appendix:gaussian_env}.

\paragraph{Mutual Information as a Proxy for Multimodality.}  
We first evaluate if mutual information provides a reliable proxy for multimodality. To this end, we construct $3$ expert datasets in the Gaussian-mixture environment containing one, two, or four goal modes, and train separate policies on each dataset. Visualizations of the demonstrations,  rollouts of policies trained on each dataset alongside Monte Carlo estimates of their action distributions are shown in the Appendix in Figure~\ref{fig:q1-env-trajectories}. %). Figure~\ref{fig:q1-env-trajectories} shows rollouts of policies trained on each dataset (top row) alongside Monte Carlo estimates of their action distributions at $t=0$ (bottom row). 
% As expected, policies trained on multimodal datasets exhibit multi-peaked action distributions.
We hypothesize that a valid multimodality metric $\mathcal{M}$ should increase with the number of modes. To test this, we estimate $\mathcal{M}$ with Equation~\ref{eq:variational_MI} by jointly training a steering policy and an inference model $q_\phi$ over a discrete latent space $\mathcal{Z}=\{0,1,2,3\}$. 

% \renewcommand{\arraystretch}{1.} % Adjust row spacing
% \setlength{\arrayrulewidth}{0.05pt} % Adjust line thickness

% Wrapped table: placed here so it is anchored next to text.
\begin{wraptable}[8]{r}{0.48\columnwidth}
  \vspace{0.2\baselineskip}
  \centering
  \caption{Mutual information and inference-model loss.}
  \label{tab:mi_disc}
  \begin{tabular}{lcc}
  \toprule
  \textbf{Policy} & \textbf{$\mathcal{M}$} & $\mathbf{L_q(\phi)}$ \\
  \midrule
  \tablerowcolors
  1 mode  & $0.00 \scriptscriptstyle\pm 0.00$  & $1.38 \scriptscriptstyle\pm 0.00$ \\
  2 modes & $0.58 \scriptscriptstyle\pm 0.02$  & $0.82 \scriptscriptstyle\pm 0.02$ \\
  4 modes & $1.06 \scriptscriptstyle\pm 0.00$  & $0.33 \scriptscriptstyle\pm 0.02$ \\
  \bottomrule
  \end{tabular}
  \vspace{-0.7\baselineskip}
\end{wraptable}

Table~\ref{tab:mi_disc} reports the estimated mutual information and inference-model loss from $q_\phi$ on $1024$ trajectories with randomly sampled $z\in\mathcal{Z}$. As expected, mutual information increases with the number of modes, while the inference loss decreases, indicating that $q_\phi$ reliably discriminates latent codes when multimodality exists. These results support mutual information as a proxy for measuring multimodality, when such multimodality exists. Figure~\ref{fig:q1-modes-by-z} in the Appendix further shows that conditioning the steering policy on individual $z$ produces distinct, coherent trajectories, confirming that the latent space organizes noise into meaningful and controllable behavioral modes.

\renewcommand{\arraystretch}{1.10} % Adjust row spacing
\setlength{\arrayrulewidth}{0.04pt} % Adjust line thickness

\begin{table}[!t]
  \centering
  \caption{Evaluation on the Gaussian-mixture environment under fine-tuning reward landscape \textbf{G1}.}
  \label{tab:toy_results_g1}
  \resizebox{\columnwidth}{!}{%
  \begin{tabular}{lcc>{\columncolor{lightblue!20}}c c}
  \toprule
   \textbf{Method} & $\mathrm{SR}$ $(\uparrow)$ & $\mathrm{SR}_{\mathrm{M}}$ $(\uparrow)$ & $\mathrm{mc}@80$ $(\uparrow)$ & $\mathcal{{H}}$ $(\uparrow)$ \\
  \midrule
  \rowcolor{lightgray!50}\texttt{RES} & $0.98 \scriptscriptstyle\pm 0.02$ & $0.98 \scriptscriptstyle\pm 0.02$ & \cellcolor{lightblue!20}$4.00 / 4$  & $1.00 \scriptscriptstyle\pm 0.00$  \\
  \texttt{DSRL} & $1.00 \scriptscriptstyle\pm 0.00$ & $0.25 \scriptscriptstyle\pm 0.00$ & $1.00 / 4$& $0.00 \scriptscriptstyle\pm 0.00$  \\
  \rowcolor{lightgray!50}\texttt{DPPO} & $1.00 \scriptscriptstyle\pm 0.00$ & $0.58 \scriptscriptstyle\pm 0.12$ & \cellcolor{lightblue!20}$2.33 / 4$& $0.40 \scriptscriptstyle\pm 0.03$  \\
  \texttt{DPPO[10]} & $0.66 \scriptscriptstyle\pm 0.32$ & $0.16 \scriptscriptstyle\pm 0.08$ & $0.33 / 4$& $0.00 \scriptscriptstyle\pm 0.00$  \\
  \midrule
  \rowcolor{lightgray!50}\texttt{RES[\methodname{}]} & \textbf{$1.00 \scriptscriptstyle\pm 0.00$} & \textbf{$1.00 \scriptscriptstyle\pm 0.00$} & \cellcolor{lightblue!20}\textbf{$4.00 / 4$} &\textbf{ $0.99 \scriptscriptstyle\pm 0.00$} \\
  \texttt{DSRL[\methodname{}]} & $0.33 \scriptscriptstyle\pm 0.47$ & $0.33 \scriptscriptstyle\pm 0.47$ & $1.33 / 4$ & $0.46 \scriptscriptstyle\pm 0.41$  \\
  \rowcolor{lightgray!50}\texttt{DPPO[\methodname{}]} & $1.00 \scriptscriptstyle\pm 0.00$ & $1.00 \scriptscriptstyle\pm 0.00$ & \cellcolor{lightblue!20}$4.00 / 4$ & $0.99 \scriptscriptstyle\pm 0.00$  \\
  \bottomrule
  \end{tabular}%
  }
  \end{table}

\paragraph{Fine-Tuning Evaluation.}
We evaluate the performance of existing fine-tuning methods against \methodname{} in preserving multimodality when the reward used for adaptation differs from the one implicitly encoded in the demonstrations.  
To simulate this mismatch, we define two shifted reward landscapes obtained by rotating the Gaussian peaks used for demonstrations by ${\tfrac{\pi}{8}}$ and ${\tfrac{\pi}{4}}$, denoted as \textbf{G1} and \textbf{G2}.  

Tables~\ref{tab:toy_results_g1} and~\ref{tab:toy_results_g2} report results for \ac{rlft} with \textbf{G1} and \textbf{G2}, respectively. Among baselines, the residual policy (\texttt{RES}) performs best, solving \textbf{G1} and retaining two modes in \textbf{G2}, as tuning the magnitude of the residual corrections to small values helps preserve multimodality. Gradient-based methods (\texttt{DPPO}, \texttt{DPPO[10]}) improve task success, with \texttt{DPPO} aided by extra denoising steps, but both collapse to fewer modes when rewards diverge from demonstrations. Steering alone (\texttt{DSRL}) is least effective, with limited success in \textbf{G1} and full performance drop in \textbf{G2}. %Multimodality retention further degrades in the unbalanced setting, where reward bias toward specific goals causes even strong baselines to collapse to dominant peaks. 
These results suggest that standard fine-tuning techniques fail to fully preserve the original multimodality when the reward landscape deviates from the distribution underlying the demonstrations.

In contrast, the \texttt{[\methodname{}]} variants preserve diversity more consistently: \texttt{RES[\methodname{}]} recovers full mode coverage across both goals, and \texttt{DPPO[\methodname{}]} shows similar gains. For the \texttt{DSRL} method,  \texttt{[\methodname{}]} mitigates but does not prevent collapse, indicating that additional fine-tuning of the original policy is required in this case. Overall, \methodname{} regularized fine-tuning prevents the RL objective from biasing the policy toward fewer behaviors.  Qualitative visualizations of the trajectories learned by the \texttt{DPPO} and \texttt{RES[\methodname{}]} policies are shown in Figure~\ref{fig:multimodality_comparison}. Appendix~\ref{appendix:dim_z} reports ablations on the role of the dimensionality of $\mathcal{Z}$.

\begin{table}[!t]
  \centering
  \caption{Evaluation on the Gaussian-mixture environment under fine-tuning reward landscape \textbf{G2}.}
\label{tab:toy_results_g2}
  \resizebox{\columnwidth}{!}{%
  \begin{tabular}{lcc>{\columncolor{lightblue!20}}c c}
  \toprule
  \textbf{Method} & $\mathrm{SR}$ $(\uparrow)$ & $\mathrm{SR}_{\mathrm{M}}$ $(\uparrow)$ & $\mathrm{mc}@80$ $(\uparrow)$ & $\mathcal{{H}}$ $(\uparrow)$ \\
  \midrule
  \rowcolor{lightgray!50}\texttt{RES} & $0.92 \scriptscriptstyle\pm 0.12$ & $0.50 \scriptscriptstyle\pm 0.00$ & \cellcolor{lightblue!20}$2.00 / 4$& $0.59 \scriptscriptstyle\pm 0.13$ \\
  \texttt{DSRL} & $0.33 \scriptscriptstyle\pm 0.47$ & $0.08 \scriptscriptstyle\pm 0.12$ & $0.33 / 4$& $0.00 \scriptscriptstyle\pm 0.00$ \\
  \rowcolor{lightgray!50}\texttt{DPPO} & $1.00 \scriptscriptstyle\pm 0.00$ & $0.42 \scriptscriptstyle\pm 0.12$ & \cellcolor{lightblue!20}$1.67 / 4$& $0.02 \scriptscriptstyle\pm 0.02$ \\
  \texttt{DPPO[10]} & $0.32 \scriptscriptstyle\pm 0.22$ & $0.11 \scriptscriptstyle\pm 0.05$ & $0.00 / 4$& $0.60 \scriptscriptstyle\pm 0.22$ \\
  \midrule
  \rowcolor{lightgray!50}\texttt{RES[\methodname{}]} & \textbf{$1.00 \scriptscriptstyle\pm 0.00$} & \textbf{$1.00 \scriptscriptstyle\pm 0.00$} & \cellcolor{lightblue!20}\textbf{$4.00 / 4$ }& \textbf{$0.94 \scriptscriptstyle\pm 0.00$} \\
  \texttt{DSRL[\methodname{}]} & $0.33 \scriptscriptstyle\pm 0.04$ & $0.08 \scriptscriptstyle\pm 0.12$ & $0.33 / 4$ & $0.84 \scriptscriptstyle\pm 0.14$ \\
  \rowcolor{lightgray!50}\texttt{DPPO[\methodname{}]} & $1.00 \scriptscriptstyle\pm 0.00$ & $0.75 \scriptscriptstyle\pm 0.00$ & \cellcolor{lightblue!20}$3.00 / 4$ & $0.74 \scriptscriptstyle\pm 0.00$ \\
  \bottomrule
  \end{tabular}%
  }
  \end{table}

\begin{table*}[t]
\centering
\begin{minipage}{0.465\linewidth}
\centering
\caption{Baselines fine-tuning. }
\label{tab:baselines}
\resizebox{\linewidth}{!}{%
\centering
\begin{tabular}{lcc>{\columncolor{lightblue!20}}c c}
\toprule
 \textbf{Method} & $\mathrm{SR}$$(\uparrow)$ & $\mathrm{SR}_{\mathrm{M}}$ $(\uparrow)$& $\mathrm{mc}@0.80$ $(\uparrow)$& $\mathcal{H}$ $(\uparrow)$\\
\midrule
\rowcolor{lightgray!50} \multicolumn{5}{c}{\emph{Reach}} \\
\midrule
\texttt{PRE} & $0.32 \scriptscriptstyle\pm 0.01$ & $0.31 \scriptscriptstyle\pm 0.00$ & $0.00 / 2$& $0.99 \scriptscriptstyle\pm 0.00$  \\
\midrule
\rowcolor{lightgray!50}\texttt{RES} & $1.00 \scriptscriptstyle\pm 0.00$ & $1.00 \scriptscriptstyle\pm 0.00$ & \cellcolor{lightblue!20}$2.00 / 2$& $0.98 \scriptscriptstyle\pm 0.01$  \\
\texttt{DSRL} & $0.98 \scriptscriptstyle\pm 0.00$ & $0.98 \scriptscriptstyle\pm 0.00$ & $2.00 / 2$& $0.97 \scriptscriptstyle\pm 0.00$ \\
\rowcolor{lightgray!50}\texttt{DPPO} & $0.93 \scriptscriptstyle\pm 0.01$ & $0.94 \scriptscriptstyle\pm 0.02$ & \cellcolor{lightblue!20}$2.00 / 2$& $0.66 \scriptscriptstyle\pm 0.33$  \\
\texttt{DPPO[10]} & $0.99 \scriptscriptstyle\pm 0.00$ & $0.99 \scriptscriptstyle\pm 0.00$ & $2.00 / 2$& $0.97 \scriptscriptstyle\pm 0.03$  \\
\midrule
\rowcolor{lightgray!50} \multicolumn{5}{c}{\emph{Lift}} \\
\midrule
\texttt{PRE} & $0.14 \scriptscriptstyle\pm 0.01$ & $0.15 \scriptscriptstyle\pm 0.01$ & $0.00 / 2$& $0.97 \scriptscriptstyle\pm 0.01$\\
\midrule
\rowcolor{lightgray!50}\texttt{RES} & $1.00 \scriptscriptstyle\pm 0.00$ & $0.50 \scriptscriptstyle\pm 0.00$ & \cellcolor{lightblue!20}$1.00 / 2$& $0.00 \scriptscriptstyle\pm 0.00$  \\
\texttt{DSRL} & $0.78 \scriptscriptstyle\pm 0.03$ & $0.78 \scriptscriptstyle\pm 0.03$ & $0.67 / 2$& $0.98 \scriptscriptstyle\pm 0.01$  \\
\rowcolor{lightgray!50}\texttt{DPPO} & $0.99 \scriptscriptstyle\pm 0.01$ & $0.57 \scriptscriptstyle\pm 0.10$ & \cellcolor{lightblue!20}$1.00 / 2$& $0.05 \scriptscriptstyle\pm 0.03$ \\
\texttt{DPPO[10]} & $1.00 \scriptscriptstyle\pm 0.00$ & $0.56 \scriptscriptstyle\pm 0.08$ & $1.00 / 2$& $0.02 \scriptscriptstyle\pm 0.01$  \\
\midrule
\rowcolor{lightgray!50} \multicolumn{5}{c}{\emph{Avoid}} \\
\midrule
\texttt{PRE} & $0.94 \scriptscriptstyle\pm 0.04$ & $0.86 \scriptscriptstyle\pm 0.04$ & $20.00 / 24$& $0.63 \scriptscriptstyle\pm 0.00$ \\
\midrule
\rowcolor{lightgray!50}\texttt{RES} & $0.98 \scriptscriptstyle\pm 0.03$ & $0.04 \scriptscriptstyle\pm 0.00$ & \cellcolor{lightblue!20}$1.00 / 24$& $0.00 \scriptscriptstyle\pm 0.00$  \\
\texttt{DSRL} & $1.00 \scriptscriptstyle\pm 0.01$ & $0.09 \scriptscriptstyle\pm 0.02$ & $2.00 / 24$& $0.01 \scriptscriptstyle\pm 0.00$ \\
\rowcolor{lightgray!50}\texttt{DPPO} & $1.00 \scriptscriptstyle\pm 0.00$ & $0.26 \scriptscriptstyle\pm 0.11$ & \cellcolor{lightblue!20}$6.33 / 24$& $0.13 \scriptscriptstyle\pm 0.15$ \\
\texttt{DPPO[10]} & $1.00 \scriptscriptstyle\pm 0.00$ & $0.04 \scriptscriptstyle\pm 0.00$ & $1.00 / 24$& $0.00 \scriptscriptstyle\pm 0.00$  \\
\bottomrule
\end{tabular}%
}
\end{minipage}\hfill
\begin{minipage}{0.5\linewidth}
\centering
\caption{Fine-tuning with regularization (\methodname{}).}
\label{tab:our_method}
\centering
\resizebox{0.95\linewidth}{!}{%
\begin{tabular}{lcc>{\columncolor{lightblue!20}}c c}
\toprule
 \textbf{Method} & $\mathrm{SR}$ $(\uparrow)$& $\mathrm{SR}_{\mathrm{M}}$ $(\uparrow)$ & $\mathrm{mc}@0.80$ $(\uparrow)$& $\mathcal{H}$ $(\uparrow)$ \\
\midrule
\rowcolor{lightgray!50} \multicolumn{5}{c}{\emph{Reach}} \\
\midrule
\texttt{PRE} & $0.32 \scriptscriptstyle\pm 0.01$ & $0.31 \scriptscriptstyle\pm 0.00$ & $0.00 / 2$& $0.99 \scriptscriptstyle\pm 0.00$  \\
\midrule
\rowcolor{lightgray!50}\texttt{RES[\methodname{}]} & $0.99 \scriptscriptstyle\pm 0.00$ & $0.99 \scriptscriptstyle\pm 0.00$ & \cellcolor{lightblue!20}$2.00 / 2$& $1.00 \scriptscriptstyle\pm 0.00$ \\
\texttt{DSRL[\methodname{}]} & $1.00 \scriptscriptstyle\pm 0.00$ & $1.00 \scriptscriptstyle\pm 0.00$ & $2.00 / 2$& $0.97 \scriptscriptstyle\pm 0.01$ \\
\rowcolor{lightgray!50}\texttt{DPPO[\methodname{}]} & $0.98 \scriptscriptstyle\pm 0.01$ & $0.98 \scriptscriptstyle\pm 0.01$ & \cellcolor{lightblue!20}$2.00 / 2$& $0.67 \scriptscriptstyle\pm 0.43$ \\
\texttt{DPPO[10]} & - & -  &  -  & -   \\
\midrule
\rowcolor{lightgray!50} \multicolumn{5}{c}{\emph{Lift}} \\
\midrule
\texttt{PRE} & $0.14 \scriptscriptstyle\pm 0.01$ & $0.15 \scriptscriptstyle\pm 0.01$ & $0.00 / 2$& $0.97 \scriptscriptstyle\pm 0.01$ \\
\midrule
\rowcolor{lightgray!50}\texttt{RES[\methodname{}]} & $0.99 \scriptscriptstyle\pm 0.00$ & $0.99 \scriptscriptstyle\pm 0.00$ & \cellcolor{lightblue!20}$2.00 / 2$& $1.00 \scriptscriptstyle\pm 0.00$ \\
\texttt{DSRL[\methodname{}]} & $0.88 \scriptscriptstyle\pm 0.07$ & $0.88 \scriptscriptstyle\pm 0.07$ & $1.67 / 2$& $0.99 \scriptscriptstyle\pm 0.01$ \\
\rowcolor{lightgray!50}\texttt{DPPO[\methodname{}]} & $0.99 \scriptscriptstyle\pm 0.00$ & $0.55 \scriptscriptstyle\pm 0.07$ & \cellcolor{lightblue!20}$1.00 / 2$& $0.06 \scriptscriptstyle\pm 0.04$ \\
\texttt{DPPO[10]} & - & -  &  -  & -   \\
\midrule
\rowcolor{lightgray!50} \multicolumn{5}{c}{\emph{Avoid}} \\
\midrule
\texttt{PRE} & $0.94 \scriptscriptstyle\pm 0.04$ & $0.86 \scriptscriptstyle\pm 0.04$ & $20.00 / 24$& $0.63 \scriptscriptstyle\pm 0.00$ \\
\midrule
\rowcolor{lightgray!50}\texttt{RES[\methodname{}]} & $0.99 \scriptscriptstyle\pm 0.01$ & $0.30 \scriptscriptstyle\pm 0.02$ & \cellcolor{lightblue!20}$7.33 / 24$& $0.53 \scriptscriptstyle\pm 0.01$  \\
\texttt{DSRL[\methodname{}]} & $1.00 \scriptscriptstyle\pm 0.00$ & $0.42 \scriptscriptstyle\pm 0.00$ & $10.0 / 24$& $0.58 \scriptscriptstyle\pm 0.00$  \\
\rowcolor{lightgray!50}\texttt{DPPO[\methodname{}]} & $0.94 \scriptscriptstyle\pm 0.07$ & $0.43 \scriptscriptstyle\pm 0.05$ & \cellcolor{lightblue!20}$9.67 / 24$& $0.57 \scriptscriptstyle\pm 0.01$  \\
\texttt{DPPO[10]} & - & -  &  -  & -   \\
\bottomrule
\end{tabular}%
}
\end{minipage}
\end{table*}

\begin{table*}[t]
\centering
\begin{minipage}{0.465\linewidth}
\centering
\caption{ANYmal locomotion environment. }
\label{tab:anymal_results}
\resizebox{\linewidth}{!}{%
\centering
\begin{tabular}{lcc>{\columncolor{lightblue!20}}c c}
\toprule
 \textbf{Method} & $\mathrm{SR}$ $(\uparrow)$ & $\mathrm{SR}_{\mathrm{M}}$ $(\uparrow)$ & $\mathrm{mc}@0.80$ $(\uparrow)$ & $\mathcal{H}$ $(\uparrow)$ \\
\midrule
\rowcolor{lightgray!50}\texttt{PRE} & $0.41 \scriptscriptstyle\pm 0.01$ & $0.39 \scriptscriptstyle\pm 0.01$ & \cellcolor{lightblue!20}$0.00 / 4$ & $0.96 \scriptscriptstyle\pm 0.00$ \\
\midrule
\rowcolor{lightgray!50}\texttt{RES} & $0.98 \scriptscriptstyle\pm 0.01$ & $0.90 \scriptscriptstyle\pm 0.11$ & \cellcolor{lightblue!20}$3.67 / 4$ & $0.90 \scriptscriptstyle\pm 0.09$ \\
\texttt{DSRL} & $1.00 \scriptscriptstyle\pm 0.00$ & $0.25 \scriptscriptstyle\pm 0.00$ & $1.00 / 4$ & $0.00 \scriptscriptstyle\pm 0.00$ \\
\rowcolor{lightgray!50}\texttt{DPPO} & $0.99 \scriptscriptstyle\pm 0.01$ & $0.32 \scriptscriptstyle\pm 0.10$ & \cellcolor{lightblue!20}$1.33 / 4$ & $0.09 \scriptscriptstyle\pm 0.13$ \\
\texttt{DPPO[10]} & $1.00 \scriptscriptstyle\pm 0.00$ & $0.42 \scriptscriptstyle\pm 0.12$ & $1.67 / 4$ & $0.15 \scriptscriptstyle\pm 0.20$\\
\midrule
\rowcolor{lightgray!50} \texttt{DSRL[ENTROPY]} & $1.00 \pm 0.00$ & $0.25 \pm 0.00$ & $1.00 / 4$ & $0.00 \pm 0.00$ \\
\texttt{DSRL[RND]} & $1.00 \pm 0.00$ & $0.25 \pm 0.00$ & $1.00 / 4$ & $0.00 \pm 0.00$ \\
\rowcolor{lightgray!50}\texttt{DSRL[\methodname{}]}  & $0.97 \scriptscriptstyle\pm 0.01$ & $0.89 \scriptscriptstyle\pm 0.12$ & \cellcolor{lightblue!20}$3.67 / 4$ & $0.91 \scriptscriptstyle\pm 0.12$ \\
\bottomrule
\end{tabular}
}
\end{minipage}\hfill
\begin{minipage}{0.465\linewidth}
\centering
\caption{Franka Kitchen environment.}
\label{tab:franka_results}
\centering
\resizebox{\linewidth}{!}{%

\centering
\begin{tabular}{lcc>{\columncolor{lightblue!20}}c c}
\toprule
  \textbf{Method} & $\mathrm{SR}$ $(\uparrow)$ & $\mathrm{SR}_{\mathrm{M}}$ $(\uparrow)$ & $\mathrm{mc}@0.80$ $(\uparrow)$ & $\mathcal{H}$ $(\uparrow)$ \\
\midrule
\rowcolor{lightgray!50}\texttt{PRE} & $0.00 \scriptscriptstyle\pm 0.00$ & $0.00 \scriptscriptstyle\pm 0.00$ & \cellcolor{lightblue!20}$0.00 / 24$ & $0.33 \scriptscriptstyle\pm 0.07$ \\
\midrule
% \rowcolor{lightgray!50}\texttt{RES} & 1.00 & 0.04 & $1.00 / 24$ & 0.00 \\
% \texttt{DSRL} & 1.00 & 0.04 & $1.00 / 24$ & 0.00 \\
% \rowcolor{lightgray!50}\texttt{DPPO[10]} & 1.00 & 0.04 & $1.00 / 24$ & 0.00 \\
% \texttt{DPPO} & 0.71 & 0.04 & $1.00 / 24$ & 0.19 \\
% \midrule
\rowcolor{lightgray!50}\texttt{RES} & $1.00 \scriptscriptstyle\pm 0.00$ & $0.04 \scriptscriptstyle\pm 0.00$ & \cellcolor{lightblue!20}$1.00 / 24$ & $0.00 \scriptscriptstyle\pm 0.00$ \\
\texttt{DSRL} & $0.99 \scriptscriptstyle\pm 0.01$ & $0.04 \scriptscriptstyle\pm 0.00$ & $1.00 / 24$ & $0.01 \scriptscriptstyle\pm 0.02$ \\
\rowcolor{lightgray!50}\texttt{DPPO} & $0.56 \scriptscriptstyle\pm 0.41$ & $0.03 \scriptscriptstyle\pm 0.02$ & \cellcolor{lightblue!20}$0.67 / 24$ & $0.08 \scriptscriptstyle\pm 0.08$ \\
\texttt{DPPO[10]} & $1.00 \scriptscriptstyle\pm 0.00$ & $0.04 \scriptscriptstyle\pm 0.00$ & $1.00 / 24$ & $0.00 \scriptscriptstyle\pm 0.00$ \\
\rowcolor{lightgray!50}\texttt{DSRL[\methodname{}]} & $0.82 \scriptscriptstyle\pm 0.10$ & $0.10 \scriptscriptstyle\pm 0.02$ & \cellcolor{lightblue!20}$2.33 / 24$ & $0.30 \scriptscriptstyle\pm 0.05$ \\
\bottomrule
\end{tabular}%
}
\end{minipage}
\end{table*}

% \al{To integrate: 
% \begin{itemize}
%     \item comparison with KL regularization, 
%     \item comparison without pre-training the steering policy.
%     \item (If time) Flow-based policies.
% \end{itemize}}

\begin{figure*}[t!]
  \centering
  % Constrain figure height to avoid oversized floats.
  \newcommand{\skillTopH}{0.18\textheight}
  \newcommand{\skillBotH}{0.22\textheight}
  % --- top row ---
  \begin{subfigure}[t]{0.485\linewidth}
    \centering
    \includegraphics[width=\linewidth,height=\skillTopH,keepaspectratio]{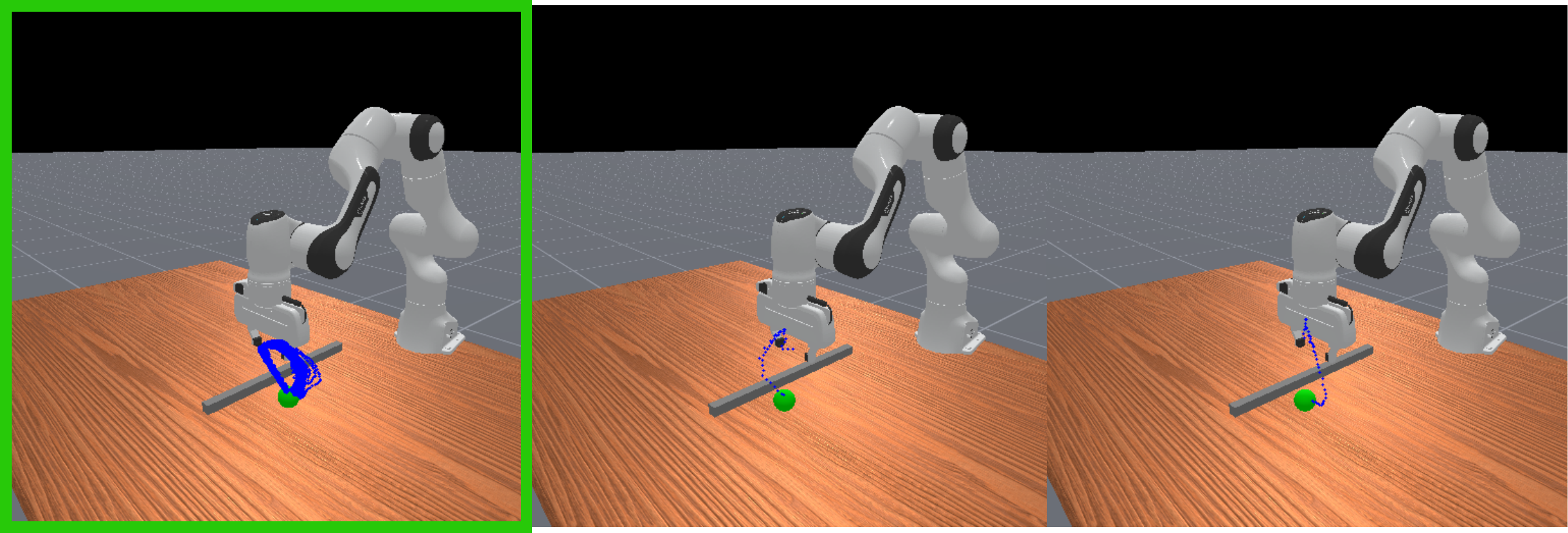}
    \caption{\emph{Reach}: (Left, green box) \texttt{DPPO}. (Right) \texttt{DPPO[\methodname{}]} modes.}
    \label{fig:top-left}
  \end{subfigure}\hfill
  \begin{subfigure}[t]{0.485\linewidth}
    \centering
    \includegraphics[width=\linewidth,height=\skillTopH,keepaspectratio]{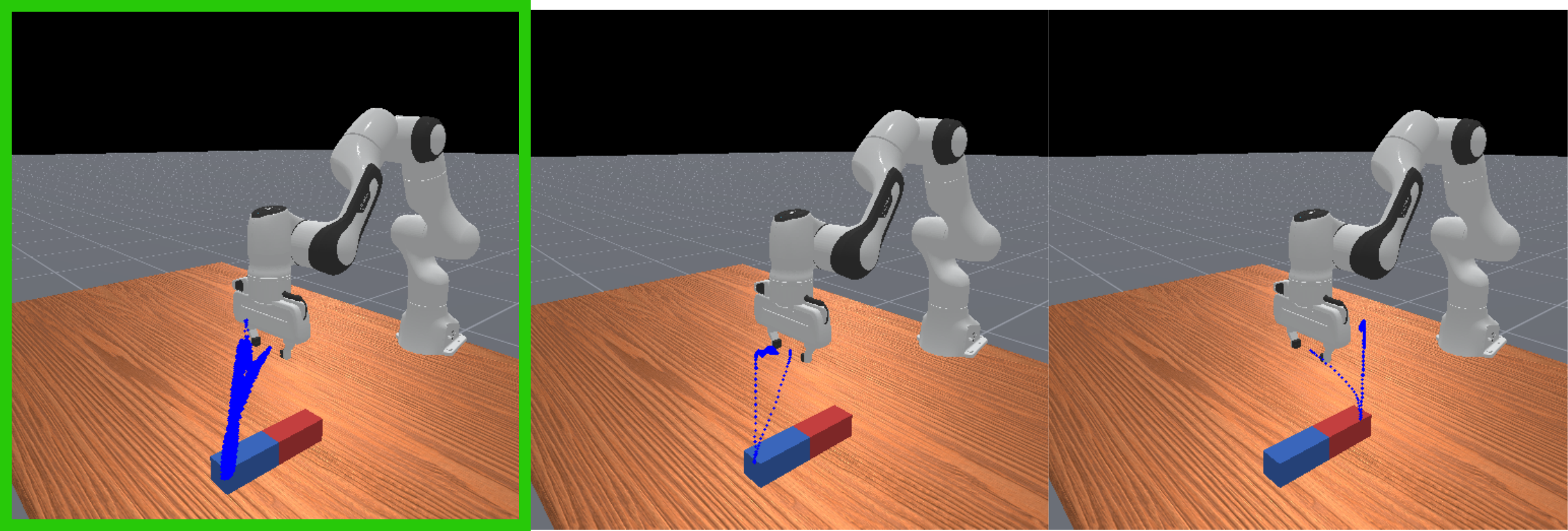}
    \caption{\emph{Lift}: (Left, green box) \texttt{DPPO}. (Right) \texttt{DPPO[\methodname{}]} modes.}
    \label{fig:top-right}
  \end{subfigure}

  % --- bottom row (wide) ---
  \vspace{0.3em}
  \begin{subfigure}[t]{0.80\linewidth}
    \centering
    \includegraphics[width=\linewidth,height=\skillBotH,keepaspectratio]{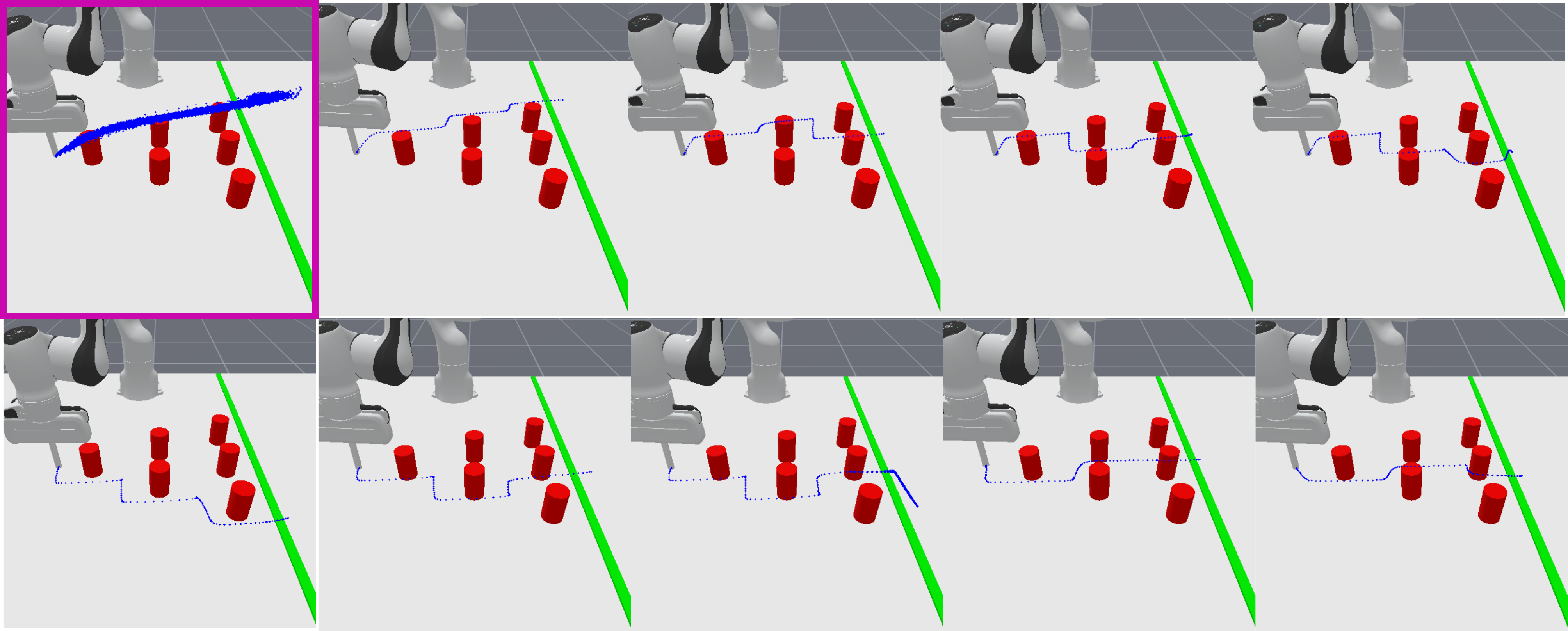}
    \caption{\emph{Avoid}: (Left, purple box) \texttt{DPPO}. (Right) \texttt{DPPO[\methodname{}]} modes.}
    \label{fig:bottom-wide}
  \end{subfigure}

  \caption{Visualization of policy rollouts (blue) from standard fine-tuning and \methodname{} fine-tuning across different tasks. Highlighted boxes (green, purple) show trajectories sampled form  \texttt{DPPO}, which exhibits multimodal behavior only in the \emph{Reach} task. The remaining visualizations represent the modes learned by \texttt{DPPO[\methodname{}]}, where trajectories are sampled by varying $z \in \mathcal{Z}$}
  \label{fig:learned_skills_main}
\end{figure*}

\subsection{Robotic Manipulation}

% We evaluate fine-tuning strategies for pre-trained diffusion policies on simulated robotic manipulation tasks, with the aim of assessing their ability to preserve multimodal behaviors. In particular, we test whether our proposed method regularizes learning to retain diversity while improving task performance.

% \paragraph{Tasks} 
% Next, we evaluate our method on three simulated robotic tasks: \emph{Reach}, \emph{Lift}, and \emph{Avoid}, based on ManiSkill~\citep{tao2024maniskill3} and visualized in Figure~\ref{fig:tasks}. Multimodality arises either from goal diversity or trajectory diversity in achieving the same goal. In \emph{Reach}, the agent can approach the target sphere from either side, though the choice is committed early. In \emph{Lift}, the peg can be grasped from different sides or regions, leading to richer modality structure. \emph{Avoid} is the most complex, with many avoidance strategies emerging later in the episode, each with different trajectory lengths. Dense or intermediate rewards are provided to support fine-tuning, and a heuristic is used to assign trajectories to modes for evaluation. For each task, we collect $1000$ demos with a motion planner and pre-train a diffusion model for $1000$ epochs.  Further implementation details are given in Appendix~\ref{appendix:manip_tasks}.

Next, we evaluate our method on three simulated robotic tasks: \emph{Reach}, \emph{Lift}, and \emph{Avoid}, implemented on ManiSkill~\citep{tao2024maniskill3} and visualized in Figure~\ref{fig:tasks}. Multimodality arises either from goal diversity or trajectory diversity in achieving the same goal. For each task, we collect $1000$ demos with a motion planner and pre-train a diffusion model for $1000$ epochs. Subsequently, dense or intermediate rewards are provided to support fine-tuning, and a heuristic is used to assign trajectories to modes for evaluation. Further implementation details are given in Appendix~\ref{appendix:manip_tasks}.

\paragraph{Standard Fine-tuning.} Table~\ref{tab:baselines} summarizes results for baselines without explicit multimodality preservation. In \emph{Reach}, all methods fine-tune the pre-trained policy without collapse, indicating that the inherent exploration of diffusion policies suffices to adapt both modes. In \emph{Lift}, fine-tuning improves success rates but fails to consistently solve both modalities; only the steering-based baseline (\texttt{DSRL}) maintains higher entropy but fails to consistently solve both modalities. In \emph{Avoid}, fine-tuning achieves high task success but eliminates multimodality, likely due to trajectory length asymmetries that bias toward shorter-horizon solutions. Taken together, the results align with the 2D Gaussian mixture experiments and indicate that standard RL fine-tuning progressively destroys multimodality as the reward landscape deviates from the pre-trained trajectory distribution and becomes unbalanced across modes. %This underscores the need for explicit regularization and motivates our method as a general mechanism to preserve multimodal behaviors during fine-tuning.

\paragraph{\methodname{} Fine-tuning.} 
% Here we evaluate the impact of incorporating our regularization during fine-tuning. As summarized in Table~\ref{tab:our_method}, the intrinsic reward enables adaptation of the pre-trained policy while preserving, in most cases, the multimodal structure of its behavior, with only a minimal trade-off between diversity and task performance. In the \emph{Reach} task, introducing our regularization has no adverse effect on the final success rate, confirming that diversity can be preserved without compromising performance. For the \emph{Lift} task, regularization enables the policy to retain both of the original solution modes from the pre-trained policy, improving upon standard fine-tuning. In the more challenging \emph{Avoid} task, it maintains high success rates while retaining a subset of the original modes, again surpassing standard fine-tuning. Although some collapse persists, the results indicate that the proposed regularization substantially mitigates mode loss, even under pronounced reward imbalance.
Table~\ref{tab:our_method} shows that incorporating our regularization enables fine-tuning of the pre-trained policy while largely preserving multimodality, with only minimal trade-offs between diversity and task performance. In \emph{Reach}, regularization leaves success rates unaffected, confirming that our regularization term does not sacrifice performance. In \emph{Lift}, \methodname{} enables the policy to retain both solution modes from the pre-trained policy, improving over standard fine-tuning. In the more challenging \emph{Avoid}, our method retains high success while preserving a subset of the modes, again outperforming baselines. Although some collapse remains, the results indicate that regularization substantially mitigates mode loss, even under pronounced reward imbalance. Qualitative visualizations of the skills learned by our method are shown in Figure~\ref{fig:learned_skills_main}. %For completeness, Appendix~\ref{appendix:manipulation_ablation} also presents ablations covering the curriculum, steering-policy pre-training for mode discovery, the choice of $\lambda$, and the effect of removing the steering policy after fine-tuning.
Additionally, ablations covering design choices such as curriculum learning and pre-training the steering policy for mode discovery, or the role of the regularization weight $\lambda$, are reported in Appendix~\ref{appendix:manipulation_ablation}.

\subsection{Robotic Locomotion and Sequential Manipulation}

We evaluate \methodname{} in more challenging settings: ANYmal locomotion, a high-dimensional extension of the 2D Gaussian setup, and the sequential, combinatorial Franka Kitchen task, both detailed in Appendix~\ref{appendix:manip_tasks}. We focus on  the \texttt{DSRL[\methodname{}]}  variant and compare against the previously introduced baselines, using latent dimensionalities of $4$ and $64$, respectively. %a high-dimensional locomotion task and a sequential manipulation task with combinatorial structure. We evaluate on ANYmal Locomotion, which extends the 2D Gaussian setup, and on the Franka Kitchen environment, both described in detail in the Appendix~\ref{appendix:manip_tasks}. In both cases, we fine-tune the steering policy and compare against the baselines introduced earlier. The latent dimensionality is set to $4$ and $16$, respectively.

Table~\ref{tab:anymal_results} summarizes the locomotion results. Despite the increased dimensionality, the results closely mirror those of the 2D Gaussian Mixture experiments: \texttt{RES} is a strong baseline in this setting, and \methodname{} preserves all four modes of the pre-trained policy. Additional rollout visualizations, mode-stability analyses, and robustness evaluations under observation noise and dynamics perturbations are provided in Appendix~\ref{appendix:mode stability} and Appendix~\ref{sec:noise_dyn_perturbations}.

Table~\ref{tab:franka_results} reports the results for Franka Kitchen, where success is defined as completing three subtasks in any order. We first observe that the original pre-trained policy exhibits limited mode coverage, as reflected by its relatively low entropy. All baselines successfully improve task success, but collapse the multimodality present in the pre-trained policy. In contrast, \methodname{} maintains the original behavioral diversity while still fine-tuning the policy to high success, but we observe the loss of one of the modalities for some seeds. The slight decrease in $\mathrm{SR}$ reflects the different success rate per sequences, but \methodname{} still achieves a higher $\mathrm{SR}_{\mathrm{M}}$ than the baselines. The unique successful sequences per method are reported in Table~\ref{tab:kitchen_sequences}. Overall, these results demonstrate that \methodname{} effectively preserves multimodal structure in both high-dimensional control tasks and sequential manipulation domains, where the baselines collapse to a single solution.

\section{Conclusions}
We studied \ac{rlft} of pre-trained generative policies while preserving multimodal behaviors. By focusing on diffusion policies pre-trained via imitation learning, we showed that standard fine-tuning often collapses into fewer behaviors. To address this, we proposed \methodname{}, which discovers latent behavioral modes via a steering policy and uses a mutual-information–based intrinsic reward to regularize fine-tuning toward retaining diversity. Across a range of robotic domains, this approach mitigated mode collapse while matching or improving task success relative to standard fine-tuning.

% Our results also highlight important limitations and directions for future work. The intrinsic-reward regularization requires careful tuning, and maintaining an inference model during fine-tuning can introduce instability as the policy distribution shifts. Moreover, the current formulation focuses on within-task modes and may require hierarchical or richer latent parameterizations to scale to large, heterogeneous datasets. Exploring alternative trajectory-diversity metrics beyond mutual information, more expressive latent spaces, and mechanisms for semantic grounding of discovered modes—e.g., via language or preference learning—are promising avenues for extending the framework.

Our results also highlighted several limitations and directions for future work. The proposed regularization relies on a specific trajectory-diversity proxy based on mutual information, motivating the exploration of alternative diversity metrics that may better capture task-relevant behavioral variation. In addition, we assumed a discrete latent space for mode discovery; extending the framework to continuous or hierarchical latents is an important direction. Finally, while we focused on leveraging discovered modes for \ac{rlft}, uncovering latent structure in pre-trained generative policies may enable broader applications, such as language grounding or enhanced controllability and human alignment, opening new avenues beyond fine-tuning.

\section*{Impact Statement}

This paper advances the understanding of controllability and fine-tuning behavior of multimodal generative policies, contributing to the advancement of the field of Machine Learning. There are many potential societal consequences of our work, none of which we feel must be specifically highlighted here.

\section*{Acknowledgements}
We would like to thank Tomi Silander and Chris Dance for their insightful feedback on our work.

\bibliography{icml2026_conference}

@article{park2025flow,
  title = {Flow Q-Learning},
  author = {Park, Seohong and Li, Qiyang and Levine, Sergey},
  journal = {International Conference on Machine Learning},
  year = {2025},
  doi = {10.48550/arXiv.2502.02538},
}

@article{black2410pi0,
  title = {$\pi_0$: A Vision-Language-Action Flow Model for General Robot Control},
  author = {Kevin Black and Noah Brown and Danny Driess and Adnan Esmail and Michael Equi and Chelsea Finn and Niccolo Fusai and Lachy Groom and Karol Hausman and Brian Ichter and others},
  year = {2024},
  eprint = {2410.24164},
  archiveprefix = {arXiv},
  primaryclass = {cs.LG},
  url = {https://arxiv.org/abs/2410.24164},
  journal = {arXiv.org},
  doi = {10.48550/arXiv.2410.24164},
}

@article{wang2022diffusion,
  title = {Diffusion policies as an expressive policy class for offline reinforcement learning},
  author = {Wang, Zhendong and Hunt, Jonathan J and Zhou, Mingyuan},
  journal = {International Conference on Learning Representations},
  year = {2022},
  doi = {10.48550/arXiv.2208.06193},
}

@article{kang2023efficient,
  title={Efficient diffusion policies for offline reinforcement learning},
  author={Kang, Bingyi and Ma, Xiao and Du, Chao and Pang, Tianyu and Yan, Shuicheng},
  journal={Advances in Neural Information Processing Systems},
  volume={36},
  pages={67195--67212},
  year={2023}
}

@article{wagenmaker2025steering,
  title={Steering Your Diffusion Policy with Latent Space Reinforcement Learning},
  author={Wagenmaker, Andrew and Nakamoto, Mitsuhiko and Zhang, Yunchu and Park, Seohong and Yagoub, Waleed and Nagabandi, Anusha and Gupta, Abhishek and Levine, Sergey},
  journal = {Conference on Robot Learning},
  year={2025}
}

@article{ankile2024imitation,
  title = {From Imitation to Refinement--Residual RL for Precise Assembly},
  author = {Ankile, Lars and Simeonov, Anthony and Shenfeld, Idan and Torne, Marcel and Agrawal, Pulkit},
  journal = {IEEE International Conference on Robotics and Automation},
  year = {2024},
  doi = {10.1109/ICRA55743.2025.11127442},
}

@article{lipman2022flow,
  title = {Flow matching for generative modeling},
  author = {Lipman, Yaron and Chen, Ricky TQ and Ben-Hamu, Heli and Nickel, Maximilian and Le, Matt},
  journal = {International Conference on Learning Representations},
  year = {2022},
}

@article{jia2024towards,
  title = {Towards diverse behaviors: A benchmark for imitation learning with human demonstrations},
  author = {Jia, Xiaogang and Blessing, Denis and Jiang, Xinkai and Reuss, Moritz and Donat, Atalay and Lioutikov, Rudolf and Neumann, Gerhard},
  journal = {International Conference on Learning Representations},
  year = {2024},
  doi = {10.48550/arXiv.2402.14606},
}

@article{yuan2024policy,
  title = {Policy decorator: Model-agnostic online refinement for large policy model},
  author = {Yuan, Xiu and Mu, Tongzhou and Tao, Stone and Fang, Yunhao and Zhang, Mengke and Su, Hao},
  journal = {International Conference on Learning Representations},
  year = {2024},
  doi = {10.48550/arXiv.2412.13630},
}

@inproceedings{
tao2024maniskill3,
title={ManiSkill3: {GPU} Parallelized Robot Simulation and Rendering for Generalizable Embodied {AI}},
author={Stone Tao and Fanbo Xiang and Arth Shukla and Yuzhe Qin and Xander Hinrichsen and Xiaodi Yuan and Chen Bao and Xinsong Lin and Yulin Liu and Tse-Kai Chan and Yuan Gao and Xuanlin Li and Tongzhou Mu and Nan Xiao and Arnav Gurha and Viswesh N and Yong Woo Choi and Yen-Ru Chen and Zhiao Huang and Roberto Calandra and Rui Chen and Shan Luo and Hao Su},
booktitle={7th Robot Learning Workshop: Towards Robots with Human-Level Abilities},
year={2025},
url={https://openreview.net/forum?id=GgTxudXaU8}
}

@article{song2020denoising,
  title = {Denoising diffusion implicit models},
  author = {Song, Jiaming and Meng, Chenlin and Ermon, Stefano},
  journal = {International Conference on Learning Representations},
  year = {2020},
}

@article{misra2019mish,
  title = {Mish: A self regularized non-monotonic activation function},
  author = {Misra, Diganta},
  journal = {British Machine Vision Conference},
  year = {2020},
  doi = {10.5244/c.34.191},
  publisher = {British Machine Vision Association},
}

@article{ho2020denoising,
  title={Denoising diffusion probabilistic models},
  author={Ho, Jonathan and Jain, Ajay and Abbeel, Pieter},
  journal={Advances in neural information processing systems},
  volume={33},
  pages={6840--6851},
  year={2020}
}

@article{zhou2024rethinking,
  title={Rethinking inverse reinforcement learning: from data alignment to task alignment},
  author={Zhou, Weichao and Li, Wenchao},
  journal={Advances in Neural Information Processing Systems},
  volume={37},
  pages={27647--27688},
  year={2024}
}

@inproceedings{brown2019extrapolating,
  title={Extrapolating beyond suboptimal demonstrations via inverse reinforcement learning from observations},
  author={Brown, Daniel and Goo, Wonjoon and Nagarajan, Prabhat and Niekum, Scott},
  booktitle={International conference on machine learning},
  pages={783--792},
  year={2019},
  organization={PMLR}
}

@article{li2024learning,
  title={Learning multimodal behaviors from scratch with diffusion policy gradient},
  author={Li, Steven and Krohn, Rickmer and Chen, Tao and Ajay, Anurag and Agrawal, Pulkit and Chalvatzaki, Georgia},
  journal={Advances in Neural Information Processing Systems},
  volume={37},
  pages={38456--38479},
  year={2024}
}

@article{yang2023policy,
  title={Policy representation via diffusion probability model for reinforcement learning},
  author={Yang, Long and Huang, Zhixiong and Lei, Fenghao and Zhong, Yucun and Yang, Yiming and Fang, Cong and Wen, Shiting and Zhou, Binbin and Lin, Zhouchen},
  journal={arXiv preprint arXiv:2305.13122},
  year={2023}
}

@article{chi2023diffusion,
  title={Diffusion policy: Visuomotor policy learning via action diffusion},
  author={Chi, Cheng and Xu, Zhenjia and Feng, Siyuan and Cousineau, Eric and Du, Yilun and Burchfiel, Benjamin and Tedrake, Russ and Song, Shuran},
  journal={The International Journal of Robotics Research},
  pages={02783649241273668},
  year={2023},
  publisher={SAGE Publications Sage UK: London, England}
}

@article{gregor2016variational,
  title={Variational intrinsic control},
  author={Gregor, Karol and Rezende, Danilo Jimenez and Wierstra, Daan},
  journal={International Conference on Learning Representations},
  year={2016}
}

@article{achiam2018variational,
  title={Variational option discovery algorithms},
  author={Achiam, Joshua and Edwards, Harrison and Amodei, Dario and Abbeel, Pieter},
  journal={arXiv preprint arXiv:1807.10299},
  year={2018}
}

@article{hansen2019fast,
  title = {Fast task inference with variational intrinsic successor features},
  author = {Hansen, S. and Dabney, Will and Barreto, André and Wiele, T. and Warde-Farley, David and Mnih, Volodymyr},
  journal = {International Conference on Learning Representations},
  year = {2019},
}

@article{sharma2019dynamics,
  title = {Dynamics-aware unsupervised discovery of skills},
  author = {Sharma, Archit and Gu, Shixiang and Levine, Sergey and Kumar, Vikash and Hausman, Karol},
  journal = {International Conference on Learning Representations},
  year = {2019},
}

@phdthesis{stoepker2016testing,
  title={Testing for multimodality},
  author={Stoepker, Ivo and van den Heuvel, E},
  year={2016},
  school={BS thesis, Eindhoven University of Technology, Eindhoven}
}

@inproceedings{liu2021aps,
  title={Aps: Active pretraining with successor features},
  author={Liu, Hao and Abbeel, Pieter},
  booktitle={International Conference on Machine Learning},
  pages={6736--6747},
  year={2021},
  organization={PMLR}
}

@misc{fu2020d4rl,
    title={D4RL: Datasets for Deep Data-Driven Reinforcement Learning},
    author={Justin Fu and Aviral Kumar and Ofir Nachum and George Tucker and Sergey Levine},
    year={2020},
    eprint={2004.07219},
    archivePrefix={arXiv},
    primaryClass={cs.LG}
}

@article{kim2021unsupervised,
  title = {Unsupervised skill discovery with bottleneck option learning},
  author = {Kim, Jaekyeom and Park, Seohong and Kim, Gunhee},
  journal = {International Conference on Machine Learning},
  year = {2021},
}

@article{zhang2021hierarchical,
  title = {Hierarchical reinforcement learning by discovering intrinsic options},
  author = {Zhang, Jesse and Yu, Haonan and Xu, Wei},
  journal = {International Conference on Learning Representations},
  year = {2021},
}

@article{machado2017eigenoption,
  title = {Eigenoption discovery through the deep successor representation},
  author = {Machado, Marlos C and Rosenbaum, Clemens and Guo, Xiaoxiao and Liu, Miao and Tesauro, Gerald and Campbell, Murray},
  journal = {International Conference on Learning Representations},
  year = {2017},
}

@article{liu2021behavior,
  title={Behavior from the void: Unsupervised active pre-training},
  author={Liu, Hao and Abbeel, Pieter},
  journal={Advances in Neural Information Processing Systems},
  volume={34},
  pages={18459--18473},
  year={2021}
}

@article{hausman2017multi,
  title={Multi-modal imitation learning from unstructured demonstrations using generative adversarial nets},
  author={Hausman, Karol and Chebotar, Yevgen and Schaal, Stefan and Sukhatme, Gaurav and Lim, Joseph J},
  journal={Advances in neural information processing systems},
  volume={30},
  year={2017}
}

@article{eysenbach2018diversity,
  title = {Diversity is all you need: Learning skills without a reward function},
  author = {Eysenbach, Benjamin and Gupta, Abhishek and Ibarz, Julian and Levine, Sergey},
  journal = {International Conference on Learning Representations},
  year = {2018},
}

@inproceedings{
chen2022latent,
title={{LAPO}: Latent-Variable Advantage-Weighted Policy Optimization for Offline Reinforcement Learning},
author={Xi Chen and Ali Ghadirzadeh and Tianhe Yu and Jianhao Wang and Yuan Gao and Wenzhe Li and Liang Bin and Chelsea Finn and Chongjie Zhang},
booktitle={Advances in Neural Information Processing Systems},
editor={Alice H. Oh and Alekh Agarwal and Danielle Belgrave and Kyunghyun Cho},
year={2022},
url={https://openreview.net/forum?id=pHd0v8W30O}
}

@article{lee2025molmoact,
  title={MolmoAct: Action Reasoning Models that can Reason in Space},
  author={Lee, Jason and Duan, Jiafei and Fang, Haoquan and Deng, Yuquan and Liu, Shuo and Li, Boyang and Fang, Bohan and Zhang, Jieyu and Wang, Yi Ru and Lee, Sangho and others},
  journal={arXiv preprint arXiv:2508.07917},
  year={2025}
}

@article{kim2024openvla,
  title = {Openvla: An open-source vision-language-action model},
  author = {Kim, Moo Jin and Pertsch, Karl and Karamcheti, Siddharth and Xiao, Ted and Balakrishna, A. and Nair, Suraj and Rafailov, Rafael and Foster, E. and Lam, Grace and Sanketi, Pannag R. and others},
  journal = {Conference on Robot Learning},
  year = {2024},
  doi = {10.48550/arXiv.2406.09246},
}

@article{psenka2023learning,
  title = {Learning a diffusion model policy from rewards via q-score matching},
  author = {Psenka, Michael and Escontrela, Alejandro and Abbeel, Pieter and Ma, Yi},
  journal = {International Conference on Machine Learning},
  year = {2023},
  doi = {10.48550/arXiv.2312.11752},
}

@article{ren2024diffusion,
  title = {Diffusion policy policy optimization},
  author = {Ren, Allen Z and Lidard, Justin and Ankile, Lars L and Simeonov, Anthony and Agrawal, Pulkit and Majumdar, Anirudha and Burchfiel, Benjamin and Dai, Hongkai and Simchowitz, Max},
  journal = {International Conference on Learning Representations},
  year = {2024},
  doi = {10.48550/arXiv.2409.00588},
}

@article{chen2024diffusion,
  title = {Diffusion policies creating a trust region for offline reinforcement learning},
  author = {Chen, Tianyu and Wang, Zhendong and Zhou, Mingyuan},
  journal = {Neural Information Processing Systems},
  year = {2024},
  pages = {50098-50125},
  doi = {10.48550/arXiv.2405.19690},
  publisher = {Neural Information Processing Systems Foundation, Inc. (NeurIPS)},
}

@article{schulman2017proximal,
  title={Proximal policy optimization algorithms},
  author={Schulman, John and Wolski, Filip and Dhariwal, Prafulla and Radford, Alec and Klimov, Oleg},
  journal={arXiv preprint arXiv:1707.06347},
  year={2017}
}

@inproceedings{huang2023reparameterized,
  title={Reparameterized policy learning for multimodal trajectory optimization},
  author={Huang, Zhiao and Liang, Litian and Ling, Zhan and Li, Xuanlin and Gan, Chuang and Su, Hao},
  booktitle={International Conference on Machine Learning},
  pages={13957--13975},
  year={2023},
  organization={PMLR}
}

@article{li2025train,
  title = {How to Train Your Robots? The Impact of Demonstration Modality on Imitation Learning},
  author = {Li, Haozhuo and Cui, Yuchen and Sadigh, Dorsa},
  journal = {IEEE International Conference on Robotics and Automation},
  year = {2025},
  pages = {1113-1120},
  doi = {10.1109/ICRA55743.2025.11128520},
  publisher = {IEEE},
}

@article{LGSD,
      title={Language Guided Skill Discovery}, 
      author={Seungeun Rho and Laura Smith and Tianyu Li and Sergey Levine and Xue Bin Peng and Sehoon Ha},
      journal={International Conference on Learning Representations},
      year={2025}
}

@article{SLIM,
      title={SLIM: Skill Learning with Multiple Critics}, 
      author={David Emukpere and Bingbing Wu and Julien Perez and Jean-Michel Renders},
      journal={2024 International Conference on Robotics and Automation (ICRA)},
      year={2024}
}

@article{CSD,
  title={Controllability-Aware Unsupervised Skill Discovery},
  author={Seohong Park and Kimin Lee and Youngwoon Lee and P. Abbeel},
  journal={International Conference on Machine Learning},
  year={2023},
}

@article{MUSIC,
  title={Mutual Information State Intrinsic Control},
  author={Zhao, Rui and Gao, Yang and Abbeel, Pieter and Tresp, Volker and Xu, Wei},
  journal={International Conference on Learning Representations},
  year={2021}
}

@article{UnsupervisedRLTransferManip,
  title={Unsupervised Reinforcement Learning for Transferable Manipulation Skill Discovery},
  author={Daesol Cho and Jigang Kim and H. Jin Kim},
  journal={IEEE Robotics and Automation Letters},
  year={2022},
  volume={7},
  pages={7455-7462},
}
\bibliographystyle{icml2026}

%%%%%%%%%%%%%%%%%%%%%%%%%%%%%%%%%%%%%%%%%%%%%%%%%%%%%%%%%%%%%%%%%%%%%%%%%%%%%%%
%%%%%%%%%%%%%%%%%%%%%%%%%%%%%%%%%%%%%%%%%%%%%%%%%%%%%%%%%%%%%%%%%%%%%%%%%%%%%%%
% APPENDIX
%%%%%%%%%%%%%%%%%%%%%%%%%%%%%%%%%%%%%%%%%%%%%%%%%%%%%%%%%%%%%%%%%%%%%%%%%%%%%%%
%%%%%%%%%%%%%%%%%%%%%%%%%%%%%%%%%%%%%%%%%%%%%%%%%%%%%%%%%%%%%%%%%%%%%%%%%%%%%%%
\newpage
\appendix
\onecolumn

% \section*{Appendix}
% \al{
% TODOs:
% \begin{itemize}
%     \item Background on diffusion and flow-based policies?
% \end{itemize}

% }

% \addcontentsline{toc}{section}{Appendix}

% Create appendix-specific ToC
% \startcontents[appendix]
% \printcontents[appendix]{l}{1}{\setcounter{tocdepth}{2}}

\section{Extended Conclusions, Limitations and Future Work}
\paragraph{Summary of the Contributions.} 

We studied the problem of \ac{rlft} of pre-trained generative policies while preserving multimodal action distributions. Focusing on diffusion policies trained from demonstrations, we showed that standard fine-tuning often collapsed multimodality to a dominant behavior when the fine-tuning reward landscape diverged from the demonstrations. To address this, we proposed using mutual information as a proxy for multimodality and introduced \methodname{}, an unsupervised mode-discovery method based on a latent reparameterization of a steering policy. We then used the steering policy together with the mutual-information estimate to provide an intrinsic reward that regularized \ac{rlft} toward retaining diverse behaviors.  We benchmarked the method across different robotics environments, and showcased that the proposed regularization mitigated mode collapse, supporting \methodname{} as a practical approach to fine-tuning generative policies without sacrificing behavioral diversity.

\paragraph{Limitations and Future Work.} Our study revealed several trade-offs and open directions. The intrinsic-reward regularization required careful tuning, as excessive weight slowed learning and reduced task success. Maintaining an inference model during fine-tuning also introduced instabilities, as it needed to track the policy’s shifting state distribution. 
Moreover, in its current form, \methodname{} is designed as a task-level mode extractor, as it focuses on discovering and preserving \emph{within-task} behavioral modes, rather than modeling the structure of multi-task datasets. 
Therefore, scaling \methodname{} to large, heterogeneous datasets may require richer latent parametrizations, such as a hierarchical latent space, to capture multi-task-level multimodality. 

Designing the appropriate parametrization and structure of the latent space is also challenging.  In all experiments, we employed a single categorical latent $\mathcal{Z}$ indexing discrete behavioral modes within each task. We deliberately used a mildly overparameterized space, and our results indicated that \methodname{} could reliably collapse redundant codes and recover the relevant modes, suggesting that overparameterization is not critical in practice. 
Nonetheless, exploring more expressive continuous or hybrid latent spaces, together with regularization strategies that improve the controllability of $\mathcal{Z}$, such as entropy or KL terms to balance code usage and mitigate mode collapse, is an important direction for future work.

We occasionally observed that distinct latent codes were mapped to the same environment-defined mode, indicating that our mutual information objective can be sensitive to small variations in the visited state distribution. This is a known weakness of mutual information-based unsupervised skill discovery illustrated in~\cite{CSD}, and systematically assessing which trajectory-diversity metrics most effectively retain multimodal behavior is a promising direction for future work.

Finally, although the formulation was independent of language supervision, the learned latent space is amenable to post-hoc semantic grounding. Aligning modes with language via preference learning or VLA mappings and developing a joint inference model that preserves diversity while enabling reliable semantic labels are compelling directions for future work.

\section{Extended Related Work}
\label{appendix:rw}

Robotic manipulation often admits multiple distinct solutions arising from kinematic redundancies, multimodal goals, or heterogeneous demonstrations~\citep{li2025train}. We review related work on handling such multimodality in the action distribution from three perspectives: (i) general approaches to learning multimodal behaviors, (ii) fine-tuning of generative policies, where we identify three categories of RL-based techniques schematically illustrated in Figure~\ref{fig:related_work_categories}, and (iii) the skill discovery literature, which closely connects to our central idea of unsupervised mode discovery. %A more comprehensive discussion is provided in Appendix~\ref{appendix:rw} \al{is it even needed a more extensive related work? Seems big and hard to shrink even more}.

\subsection{Multimodal Behavior Learning and Action Diversity}

Robotic manipulation often admits multiple distinct solutions, arising from kinematic redundancies, multimodal goals, or heterogeneous demonstrations~\citep{li2025train}. Standard RL policies parameterized by unimodal Gaussians collapse to a single behavior, limiting expressiveness and trapping learning in suboptimal modes~\citep{huang2023reparameterized}. Early work tackled this by introducing a latent‑conditioned policy within the policy gradient framework, casting trajectory generation as a latent‑variable model to encourage exploration of distinct modes and avoid local minima~\cite{huang2023reparameterized}. Imitation learning and offline RL have built on latent representations to infer discrete behaviors directly from data. Hausman et al.\ segment unlabeled demonstrations into “intention” clusters and learn a mode‑conditioned policy for each cluster~\citep{hausman2017multi}, while LAPO refines a multimodal policy via an advantage‑weighted divergence penalty that preserves original modes during offline finetuning~\citep{chen2022latent}. Integrating expressive policy representations such as diffusion and flow‑based generative policies further improves upon this by capturing complex, high‑dimensional distributions. Deep Diffusion Policy Gradient (DDiffPG)~\citep{li2024learning} demonstrated an RL agent that discovers and maintains multiple strategies by parameterizing the policy with a diffusion model. They address the tendency of the greedy RL objective to collapse to one mode by clustering experience and doing mode-specific value learning, thereby ensuring improvement of all discovered modes. 

Similar to these approaches, our work builds on a latent-variable model, but we employ it within a steering policy rather than the main policy. Unlike prior approaches that learn multimodal behaviors from scratch, we leverage a pre-trained diffusion model that already encodes diverse demonstrations and focus on \ac{rlft} while preserving multimodality in the action distribution.

\begin{figure}[t]
  \centering
  \includegraphics[width=0.95\linewidth]{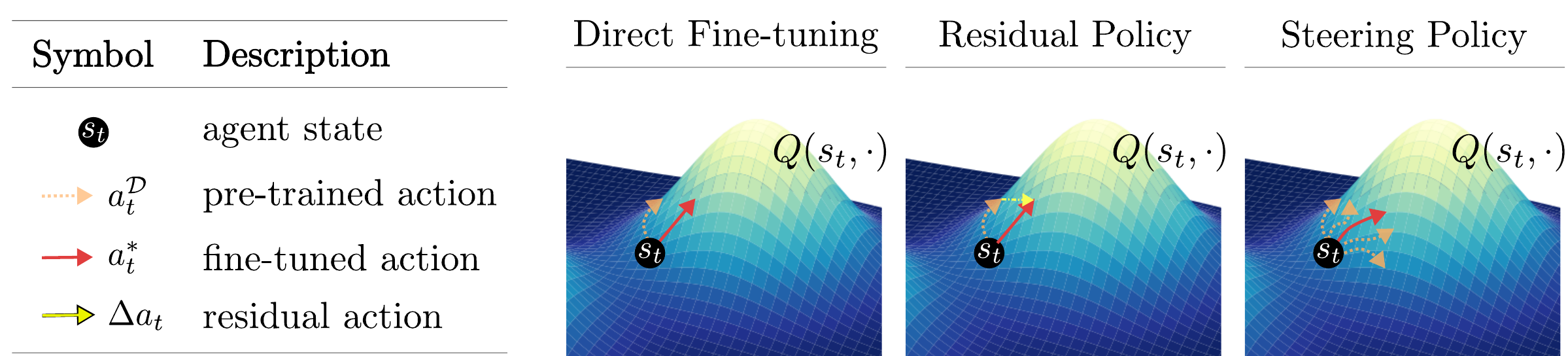}
  \caption{ \textbf{Taxonomy of \ac{rlft} Techniques Discussed in this Work.}  Each plot illustrates the learned action-value function $Q(s_t,\cdot)$ as the underlying reward landscape. Direct fine-tuning (left) adapts the pre-trained policy weights to optimize task performance, directly shifting the action distribution toward higher-value regions. Residual policies (center) learn an additive correction $\Delta a_t$ to the pre-trained action $a_t^{\mathcal{D}}$, combining them into a fine-tuned action $a_t^*$. Steering policies (right) learn a policy over the input latent noise of the generative model, biasing sampling toward regions of the noise space whose denoised actions have high-reward behaviors.  } 

  \label{fig:related_work_categories}
\end{figure}

\subsection{Fine-tuning of Pre-trained Generative Policies}
\label{sec:fine-tuning_tech}

Diffusion- and flow-based models provide expressive policy parameterizations for multimodal action distributions, but \ac{rlft} is challenging due to sequential sampling and the cost of backpropagating through the generative process. Recent work addresses these issues through three main strategies (illustrated in Figure~\ref{fig:related_work_categories}): direct fine-tuning, residual policies, and steering policies. 
\emph{Direct fine-tuning} approaches adapt the network weights either by distilling the model into a one-step sampler for easier backpropagation~\citep{park2025flow, chen2024diffusion}, by casting the denoising process as a sequential decision problem~\citep{ren2024diffusion}, or by using differentiable approximations that allow offline Q-learning without backpropagating through all denoising steps~\citep{kang2023efficient}. Despite their promise, such approaches often collapse to a single reward-maximizing mode. 
% adapt generative models to RL by directly biasing action samples toward high-value regions of the learned $Q$-function by either distilling the action sampling to enable one-step sampling which is more amenable for backpropagarion~\citep{park2025flow,chen2024diffusion}, reframing the diffusion as a decision-making problem~\cite{ren2024diffusion}, or thorugh differentiable approximations to the sampling process, enabling policy optimization via offline Q-learning while avoiding full gradient flow through the denoising trajectory~\cite{kang2023efficient}. These methods, however, often suffer from mode collapse.
\emph{Residual policy} learning methods instead freeze a pre-trained generative policy and learn a small corrective controller via \ac{rl} to address execution errors \citep{ankile2024imitation,yuan2024policy}. These techniques, along with careful regularization and architectural choices, can yield substantial performance gains over pure \ac{il}, with the potential to preserve the diversity learned from demonstrations.
\emph{Steering policy} methods instead bias the sampling process toward high-value actions without modifying the generative model itself. Some methods directly adjust training data or sampled actions using Q-values, either by nudging demonstration actions toward higher values~\citep{yang2023policy} or by combining diffusion with Q-learning to bias samples while staying close to the demonstration manifold~\citep{wang2022diffusion}. More recently, ~\cite{wagenmaker2025steering} proposed to learn to control the latent noise of generative models, guiding the sampling process toward regions of the noise space whose denoised actions yield higher reward.

% While all these approaches enable RL fine-tuning of pre-trained generative policies, they lack explicit mechanisms to preserve multimodality, and often converge to a single reward-maximizing solution. Our work extends the steering-policy framework of~\cite{wagenmaker2025steering}, by using it not only to bias behaviors toward reward, but also to discover and control latent modalities, and to quantify multimodality for use as an intrinsic reward.

Although all these approaches can successfully \ac{rlft} of pretrained policies, they lack an explicit mechanism to preserve multimodality, often collapsing to a single reward-maximizing behavior. Our approach extends the steering-policy framework~\cite{wagenmaker2025steering} by using it not only to bias behavior toward reward but also to uncover and control the latent multimodal structure of a pre-trained diffusion policy. Notably, this perspective positions the steering policy as a complementary module that can be combined with other fine-tuning methods to enforce the retention of diverse behaviors.

\subsection{Skill Discovery}

Multimodal behavior learning has also been explored through the lens of skill discovery methods.
The goal of skill discovery is to acquire a set of diverse and distinguishable behaviors without relying on external rewards. A common approach is to maximize mutual information between a latent skill variable and the states or trajectories visited by the policy, as in VIC~\citep{gregor2016variational}, DIAYN~\citep{eysenbach2018diversity}, VALOR~\citep{achiam2018variational}, VISR~\citep{hansen2019fast}, or DADS~\citep{sharma2019dynamics}. Other methods rely on successor features~\citep{machado2017eigenoption, hansen2019fast}, exploration bonuses~\citep{liu2021aps, liu2021behavior}, or hierarchical decompositions~\citep{kim2021unsupervised, zhang2021hierarchical} to induce skill diversity. 

Most of these works assume training policies from scratch in reward-free settings. However, purely diversity‑driven objectives often neglect reward alignment and directed exploration, yielding skills that may not transfer to specific manipulation goals. To mitigate this,  previous work has explored a range of approaches such as incorporating language guidance~\citep{LGSD}, combining discovery with generic extrinsic rewards~\citep{SLIM}, maximization of hard-to-achieve state transitions~\citep{CSD}, or mutual information maximization between agent and environment sections of state space~\citep{MUSIC, UnsupervisedRLTransferManip}. Our perspective is different: we leverage a pre-trained model to uncover diverse and useful behaviors already encoded in it. In particular, we are the first to study skill discovery in diffusion policies, where skills are represented as modes in the latent noise space of the generative model.

 %%%%%%%%%%%%%%%%%%%%%%%%%%%%%%%%%%%%%%%%%%%%%%%%%%%%%%%%%%%%%%%%%%%%%%%%%%%%%%%%%%%%%%%%%%
\section{Derivation of Mutual Information in Latent-Conditioned Policies}
\label{sec:app_MI_KL}

We begin by recalling the definition of conditional mutual information between a latent variable $w$ and actions $a$, given states $s$:
\begin{equation}
    I(W; A \mid S) := \E_{s \sim p(s)} \left[ 
        \E_{(a, w) \sim p(w, a \mid s)} 
        \left[ \log \frac{p(w, a \mid s)}{p(a \mid s) \, p(w \mid s)} \right] 
    \right].
\end{equation}

In the setting of latent-conditioned policies, we assume a generative process where the state $s \sim p(s)$ is sampled from a fixed distribution, the latent $w \sim p(w)$ is sampled independently of $s$, and actions are sampled from a conditional policy $\pi(a \mid s, w)$. This induces the joint distribution
\begin{equation}
    p(s, w, a) = p(s) \cdot p(w) \cdot \pi(a \mid s, w),
\end{equation}
and the conditional joint and marginals:
\begin{align}
    p(w, a \mid s) &= p(w) \cdot \pi(a \mid s, w), \\
    p(w \mid s) &= p(w).
\end{align}

Substituting these expressions into the definition of conditional mutual information, we obtain:
\begin{equation}
\begin{aligned}
    I(W; A \mid S)
    &= \E_{s \sim p(s)} \Big[ 
        \E_{w \sim p(w),\, a \sim \pi(a \mid s, w)} 
        \Big[ \log \frac{p(w)\, \pi(a \mid s, w)}{p(w)\, p(a \mid s)} \Big] 
    \Big] \\
    &= \E_{s \sim p(s)} \Big[ 
        \E_{w \sim p(w),\, a \sim \pi(a \mid s, w)} 
        \Big[ \log \frac{\pi(a \mid s, w)}{p(a \mid s)} \Big] 
    \Big].
\end{aligned}
\end{equation}

Recognizing this expression as the Kullback–Leibler (KL) divergence between the conditional distribution $\pi(a \mid s, w)$ and its marginal $p(a \mid s)$, we rewrite the mutual information as:
\begin{equation}
    I(W; A \mid S)
    = \E_{s \sim p(s)} \left[ 
        \E_{w \sim p(w)} \left[ 
            \KL \big( \pi(a \mid s, w) \,\Vert\, p(a \mid s) \big) 
        \right] 
    \right].
\end{equation}

In this formulation, $p(a \mid s)$ is interpreted as the marginal action distribution under latent sampling:
\begin{equation}
    p(a \mid s) = \E_{w \sim p(w)} \big[ \pi(a \mid s, w) \big].
\end{equation}

This derivation provides a formal and tractable characterization of the mutual information between latent variables and actions under a latent-conditioned policy. It also justifies the use of mutual information as a measure of multimodality: if $w$ has a significant influence on the action distribution $\pi(a \mid s, w)$, then the divergence between conditionals and the marginal $p(a \mid s)$ is large, leading to a high $I(W; A \mid S)$. Conversely, if the latent has little effect on the action distribution, the mutual information approaches zero.

\section{Multimodality Implies Positive Mutual Information}
\label{sec:proof_MI_positive}

\begin{proposition}
Let $W, A$ and $S$ be discrete random variables such that
\begin{align*}
P_{A|S,W}(\cdot | s_0, w_1) \ne P_{A|S,W}(\cdot | s_0, w_2)
\end{align*}
for some $s_0, w_1, w_2$ with $P_S(s_0) P_W(w_1) P_W(w_2) > 0$.
Then  $I(W;A|S) > 0$. 
\end{proposition}
\begin{proof}
First we show that for any discrete probability distributions $P$ and $Q$ on a common sample space $\mathcal{X}$, 
\begin{align}
P \ne Q \qquad\Rightarrow \qquad \DKL(P\parallel Q) > 0.
\label{KL-positive}
\end{align}
Indeed, as $\log$ is concave, 
\begin{align*}
-\DKL(P\parallel Q) = \sum_{x : P(x)>0} P(x) \log \frac{Q(x)}{P(x)} \overset{\text{(a)}}{\le} \log\left(\sum_{x : P(x)>0} P(x) \frac{Q(x)}{P(x)}\right) \overset{\text{(b)}}{\le} 0.
\end{align*}
Moreover, either:
\begin{itemize}
\item $Q(x_1)/P(x_1) \ne Q(x_2)/P(x_2)$ for some $x_1, x_2$ in the support of $P$, and as $\log$ is strictly concave, it follows that inequality (a) is strict;
\item or $Q(x)/P(x)$ is constant for $x$ in the support of $P$, and as $P \ne Q$ it follows that $\sum_{x : P(x) > 0} Q(x) < 1$ so inequality (b) is strict.
\end{itemize}
Therefore (\ref{KL-positive}) holds.

Let $i^* \in \{1,2\}$ be an index such that $P_{A|S}(\cdot | s_0) \ne P_{A|S,W}(\cdot|s_0,w_{i^*})$, noting that such an $i^*$ exists by the hypothesis that the distributions $P_{A|S,W}(\cdot | s_0, w_i)$ for $i=1,2$ are distinct.
Using the expression for $I(W;A|S)$ of Appendix~[YourRef], and the nonnegativity of KL-divergence,
\begin{align*}
I(W;A|S) &= \mathbb{E}_{S} \mathbb{E}_{W}[\DKL( P_{A|S,W}(\cdot | S, W) \parallel P_{A|S}(\cdot | S))] \\
&\ge P_S(s_0) P_W(w_{i^*}) \DKL( P_{A|s_0,w_{i^*}}(\cdot | s_0, w_{i^*}) \parallel P_{A|S}(\cdot | s_0)) \\
&> 0
\end{align*} 
where in the last line we used  the hypothesis that $P_S(s_0) P_W(w_1) P_W(w_2) > 0$ and equation (\ref{KL-positive}).
\end{proof}

\section{Method Details}

\subsection{Connection to Skill Discovery.}
\label{appendix:skill_discovery_discussion}
The variational lower bound in equation~\ref{eq:variational_MI} is formally analogous to those used in prior skill discovery methods, but its purpose in our setting is fundamentally different.
In mutual-information-based skill discovery, the bound is optimized jointly with the policy to encourage exploration and broaden coverage of the state space. By contrast, our diffusion policy is pre-trained and fixed, so the mutual information objective cannot alter the state distribution. Instead, maximizing $I(Z;  S)$ serves to uncover and control the intrinsic multimodality already embedded in the generative policy by promoting diversity in the action space $\mathcal{A}$. In addition, this bound provides a practical metric to quantify multimodality in pre-trained policies, as demonstrated in Section~\ref{sec:2D_Gaussians}.

%, but given the fixed pre-trained policy this also presents a lower bound for conditioned on states already visited by the pre-trained model. 
% While the mapping between states and actions might not be strictly bijective, so the lower bound does not fully capture the mutual information between $Z$ and $\mathcal{A}$, it nevertheless provides a principled mechanism to uncover latent modes encoded in the diffusion policy. Thus, rather than discovering novel skills through active interaction with the environment, our approach leverages the same mutual information lower bound to reveal and steer the intrinsic multimodality of a fixed generative policy.

\subsection{Curriculum Learning.} 

Unlike in standard skill discovery, we have access to full trajectory rollouts for each mode we want to discover. However, this makes the joint optimization of the steering policy and inference model challenging as the policy must maintain temporal consistency while producing behaviors that remain discriminable by the inference model $q_\phi$, which can lead to instability during training. To mitigate this, we introduce a curriculum strategy that gradually increases the trajectory horizon. Concretely, instead of unrolling episodes for the full environment length $T$ from the outset, we begin training with shorter horizons $H < T$ and progressively extend them until reaching the maximum length. This staged schedule eases the optimization by allowing the policy to first acquire locally consistent behaviors before being required to sustain them over longer time horizons, thereby improving the stability and quality of the learned latent modes. The proposed curriculum is visualized in Figure~\ref{fig:curriculum}.

\begin{figure}[h]
    \centering
    \includegraphics[width=0.98\linewidth]{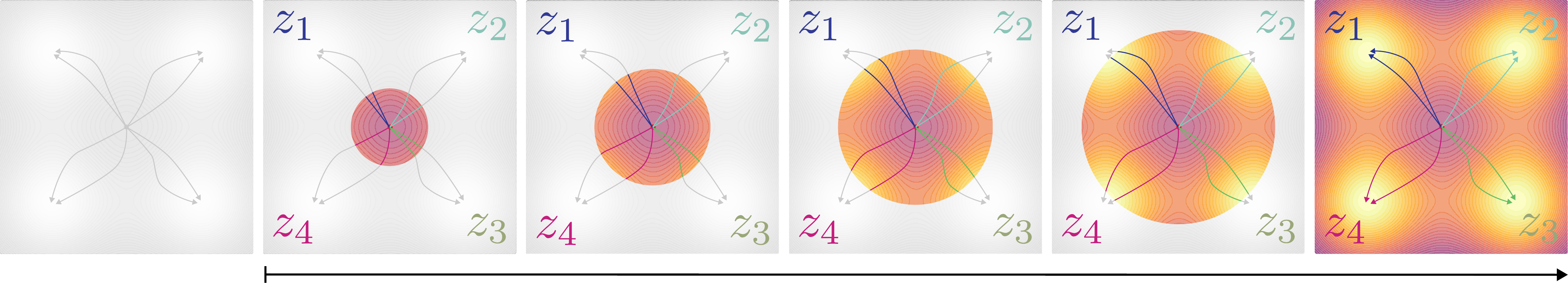}
    \caption{\textbf{Curriculum Learning.} Illustration of the curriculum strategy in a toy environment with four discrete modes. The environment is defined by a mixture of four Gaussian modes (details in Section~\ref{sec:2D_Gaussians}), each corresponding to a distinct cluster of trajectories. Starting from short horizons, the inference model $q_\phi$ only needs to discriminate local trajectory prefixes, which simplifies learning. As the horizon gradually increases, the trajectory distributions expand, and the modes become more separable across the state-action space. The curriculum thus enables the steering policy to develop temporally consistent and discriminable behaviors, progressively uncovering the underlying latent structure of the pre-trained model.}

    \label{fig:curriculum}
\end{figure}

\subsection{Algorithm}
\label{appendix:algorithm}

We outline here Algorithm~\ref{alg:mode_finetuning}. We begin from the pre-trained diffusion policy $\pi_{\theta}(a\mid s,w)$ and initialize the steering policy $\pi^{\mathcal{W}}_{\psi}(w\mid s,z)$, inference model $q_{\phi}(z\mid s,a)$, and critic $V_{\omega}(s,z)$, with intrinsic scale $\lambda\!\ge\!0$, uniform prior $p(z)$, epochs $E$, episodes per epoch $N$, warm-start $E_{\text{wp}}$, initial horizon $H_0$, max horizon $T$, and a scheduler $H(e)\!\in\![H_0,T]$ that increases the rollout horizon by a fixed step every $20$ epochs after a first warm-up of $100$ epochs. For each epoch $e$ and episode $n$, we sample a latent $z\!\sim\!p(z)$ once and keep it fixed over the rollout of length $H(e)$; at each step we draw $w_t\!\sim\!\pi^{\mathcal{W}}_{\psi}(w\mid s_t,z)$, then $a_t\!\sim\!\pi_{\theta}(a\mid s_t,w_t)$, and transition $s_{t+1}\!\sim\!p(\cdot\mid s_t,a_t)$. The intrinsic reward is ${r^{\text{int}}_t=\lambda\big(\log q_{\phi}(z\mid s_{t+1},a_t)-\log p(z)\big)}$. During the \emph{mode-discovery} stage ($e<E_{\text{wp}}$) we optimize using intrinsic-only returns $r^{\text{tot}}_t=r^{\text{int}}_t$, allowing the curriculum $H(e)$ to grow from $H_0$ toward $T$ so the policy first attains locally consistent behaviors before sustaining them over longer horizons. After the warm-start ($e\ge E_{\text{wp}}$), we introduce the task reward and train with ${r^{\text{tot}}_t = r_{\text{env}}(s_t,a_t) + r^{\text{int}}_t}$ to steer toward high-return regions without collapsing diversity. At the end of each epoch, we update the actor and critic with PPO, minimizing $L_{\pi}^{\text{PPO}}(\psi)+c_V L_V(\omega)+c_{\mathcal{H}} L_{\mathcal{H}}(\psi)$, and train the inference model by NLL, $\min_{\phi} L_q(\phi) = -\mathbb{E}[\log q_{\phi}(z\mid s,a)]$; this repeats for $e=1,\dots,E$ with horizon scheduling and the stage switch as specified.

\section{Implementation Details} 
\label{appendix:implementation_details}
We now detail the implementation and training of the pre-trained policy, all the baseline policies, and the discriminator. We also describe how our method integrates with these general fine-tuning strategies. All approaches employ PPO as the \ac{rl} algorithm for \ac{rlft} with clipping parameter $\epsilon=0.2$, GAE $\lambda=0.95$, discount $\gamma=0.99$, and Adam with learning rate $3\times 10^{-4}$. To facilitate reproducibility, we will release the full codebase together with all hyperparameters required to reproduce the results reported in this paper.

\subsection{Pre-trained policy and DPPO fine-tuning} 
The diffusion policy is trained with the standard behavioral cloning objective for diffusion models, where the network predicts the injected noise conditioned on the noisy actions. We follow the implementation and hyperparameter setup of DPPO~\cite{ren2024diffusion}, using a cosine noise schedule during training. The action horizon coincides with the execution horizon and consists of $4$ action steps per chunk. Pre-training is performed with $20$ denoising steps, while inference uses DDIM~\citep{song2020denoising} sampling with $2$ steps. For frozen policies, we set $\eta=0$, whereas for fine-tuning, we set $\eta=1$, which is equivalent to applying DDPM~\citep{ho2020denoising}. This choice ensures steerability of the policy and avoids memoryless noise schedules. The policy head is implemented as a multi-layer perceptron (MLP) with hidden dimensions $\{512, 512, 512\}$, and a time-embedding dimension of $16$, which we found to improve training stability compared to UNet backbones, similar to~\cite{ren2024diffusion}. For fine-tuning, we follow the implementation and hyperparameters introduced in~\cite{ren2024diffusion}, with the only addition of decreasing the number of fine-tuning steps of the denoising process from $10$ to $2$ to ensure a non-memoryless noise schedule.

\subsection{Residual Policy} 
The residual policy learns an additive correction to the action chunk $a_{t:t+H}$ of length $H$ proposed by the pre-trained diffusion policy, such that $a^*_{t:t+H} = a_{t:t+H} + \lambda \Delta a_{t:t+H}$. 
Concretely, the residual network receives as input the state and the pre-trained action chunk, and outputs a correction term that is passed through a $\tanh$ activation to ensure bounded updates, $\pi^{\mathrm{RES}}(\Delta a_{t:t+H}\mid s_t, a_{t:t+H})$. 
To prevent the residual from completely overriding the original action, its contribution is scaled by a tunable factor $\lambda$, which balances task success with fidelity to the pre-trained behavior. 
This scaling parameter is selected following prior work and tuned empirically to trade off between preserving the original action distribution and improving task success rates. 
The residual policy is implemented as a Gaussian policy parameterized by a multilayer perceptron with hidden layers of dimension $\{256, 256, 256\}$ and Mish activations.

\subsection{Steering policy} 
The steering policy $\pi_\psi^{\mathcal{W}}(w \mid s, z)$ is implemented as a Gaussian policy parameterized by an MLP with hidden layers of size $\{256, 256, 256\}$. %To constrain its support within that of the original diffusion prior, we apply a KL regularization during training of the form
% \[
% \mathcal{L}_{\mathrm{KL}} = \mathbb{E}_{s,z}\Big[ 
% D_{\mathrm{KL}}\!\left(\pi_\psi^{\mathcal{W}}(w \mid s,z)\,\big\|\,\mathcal{N}(0,I)\right)
% \Big],
% \]
% where $\mathcal{N}(0,I)$ denotes the isotropic Gaussian prior used in the diffusion model. 
The latent variable $z \in {0,1,\dots,K{-}1}$ is sampled from a uniform categorical prior $p(z)$, as we empirically found discrete latents easier to learn and more stable than continuous ones. The dimensionality of the latent space is a hyperparameter; in the experiments, we consider $K=\{ 4, 8, 16\}$. Training proceeds in two stages: for the first 200 epochs, the steering policy is optimized only with the intrinsic reward $\log q_\phi(z \mid s, a) - \log p(z)$, serving as a mode-discovery phase; in the remaining epochs, the environment reward is added to steer behaviors toward high-return regions while retaining multimodality. %\al{Missing the critic?}

\subsection{Inference model}
The inference model $q_\phi(z \mid s)$ is implemented as a categorical classifier over the latent codes ${z \in \{0,\dots,K-1\}}$. It consists of a multilayer perceptron with hidden layers of dimension $\{256,256,256\}$, Mish activations~\citep{misra2019mish}, and a final softmax output producing the class probabilities $q_\phi(z \mid s)$. To prevent overfitting to small variations in continuous states, Gaussian noise with standard deviation $\{1.0, 0.01, 0.001\}$ (depending on the task) is injected into the inputs during training only. 
The model is trained by minimizing the negative log-likelihood $
\mathcal{L}_{\mathrm{NLL}}(\phi) = - \mathbb{E}_{(s,a,z)}\big[\log q_\phi(z \mid s)\big],
$
where the expectation is taken over state-action pairs generated by the steering policy and latent codes sampled from the prior $p(z)$. During training of the steering policy, the log-posterior $\log q_\phi(z \mid s)$ serves as an intrinsic reward, combined with the prior correction term $-\log p(z)$, thereby providing the intrinsic objective for mode discovery and diversity-preserving fine-tuning.

\subsection{Integrating with other fine-tuning techniques.} 
The steering policy with mode discovery uncovers and controls the behavioral modes of the pre-trained diffusion mode, steering them toward regions of high reward. However, because this mechanism does not update the diffusion weights directly, its performance remains bounded by the expressiveness of the pre-trained policy. From this perspective, the steering policy can be viewed as an \emph{exploration agent} that guides state visitation in a structured way, and can therefore be seamlessly combined with existing fine-tuning methods discussed in Section~\ref{sec:rw}. A key distinction is that our framework provides access to a discriminator that evaluates whether the fine-tuned behaviors remain consistent with the discovered modes, supplying an intrinsic reward that discourages collapse into a single strategy. While the steering policy itself can continue to adapt jointly with the diffusion model, we found it beneficial to update the discriminator with a very low learning rate: this allows it to accommodate novel states encountered during fine-tuning while preserving the previously identified mode structure, thereby stabilizing multimodality retention.

% \

\section{Baseline Methods and Evaluation Metrics Discussion}

Following the characterization introduced in Section~\ref{sec:fine-tuning_tech}, we benchmark our approach against representative strategies for on-policy fine-tuning of generative policies, focusing on diffusion models but noting that analogous evaluations apply to flow-matching policies. Specifically, we consider methods that do (i) direct fine-tuning, (ii) residual corrections, and (iii) steering, noting that none of these explicitly seek to preserve multimodality. As a direct fine-tuning approach, we include \texttt{DPPO}~\citep{ren2024diffusion}, which optimizes the diffusion policy weights with PPO. We consider the DDIM parameterization of the generative process to ensure non-memoryless noise schedules, while maintaining a balance between $\eta > 0$ and the number of reverse diffusion steps to facilitate weight fine-tuning. To examine the effect of decreasing the number of reverse diffusion steps, we also consider the original hyperparameters of the \texttt{DPPO} baseline that uses the full denoising chain for action sampling with DDPM parameterization, and fine-tunes the last $10$ steps, denoted \texttt{DPPO[10]}, which makes the generation process non-memoryless.

As a residual fine-tuning approach (\texttt{RES}), we evaluate Policy Decorator~\citep{yuan2024policy}, where a lightweight residual network is trained on top of the frozen pre-trained diffusion model. This allows task adaptation while limiting catastrophic interference with the base model. 
Finally, we consider~\cite{wagenmaker2025steering} as a steering-based policy \texttt{DSRL}, which adapts the latent noise distribution $w$ to bias the pre-trained policy toward high-reward behaviors. This category operates entirely in the latent space and, like the others, does not include any explicit mechanism for mode discovery or diversity preservation. 

Importantly, our approach is orthogonal to these categories: the proposed multimodality-preserving regularizer can be combined with either residual or steering-based fine-tuning under non-memoryless noise schedules. Accordingly, we report results both for the standalone baselines and for their variants augmented with our multimodality regularizer, denoted as \texttt{X[\methodname{}]}, where \texttt{X} indicates the corresponding baseline. Full implementation details for all baselines and their regularized variants are provided in Appendix~\ref{appendix:implementation_details}.

\paragraph{Evaluation Metrics.}
We assume access to the ground truth modes of the trajectories executed by the policy in simulation. and we evaluate fine-tuned policies along two axes: \emph{task success} and \emph{behavioral diversity}. For task success, we report the overall success rate $\mathrm{SR}$, and two mode-aggregated success measures: the success rate weighted for each mode
$
{\mathrm{SR}_{\text{M}}=\tfrac{1}{K}\sum_{i=1}^K \mathrm{SR}_i,}
$
which guards against degenerate solutions (e.g.\, \(100\%\) success on a single mode but failure on others), and
mode coverage ${\mathrm{mc}@\tau=\tfrac{1}{K}\sum_{i=1}^K \mathbf{1}\{\mathrm{SR}_i \ge \tau\}}$,
the fraction of modes solved above threshold \(\tau\).

To further measure multimodality, we follow the D3IL benchmark~\citep{jia2024towards} and compute the entropy of the empirical distribution over modes among all rollouts: $H(\pi)=-\sum_{i=1}^K p_i\log p_i$, where $p_i$ is the fraction of episodes in mode $i$. A higher entropy reflects more balanced usage of the available modes, whereas a reduction after fine-tuning is indicative of mode collapse. All metrics are computed from $N=1024$ evaluation episodes with fixed seeds for fair comparison, and we report both the mean and standard deviation over three independent runs with different random seeds.

\begin{figure*}[t!]
  \centering

  % ---------- LEFT: Original ----------
  \begin{minipage}[t]{0.32\textwidth}\vspace{40pt}
    \centering
    \includegraphics[width=0.9\linewidth]{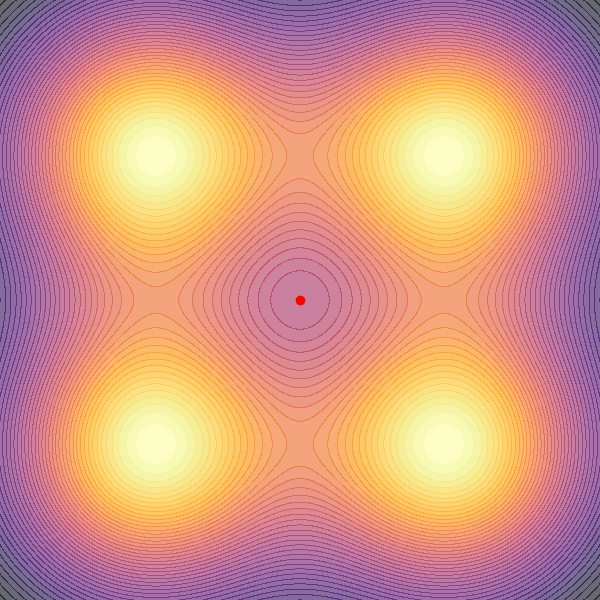}
    \subcaption{Original.}
    % \label{fig:original}
  \end{minipage}\hfill%
  \begin{minipage}[t]{0.32\textwidth}\vspace{40pt}
    \centering
    \includegraphics[width=0.9\linewidth]{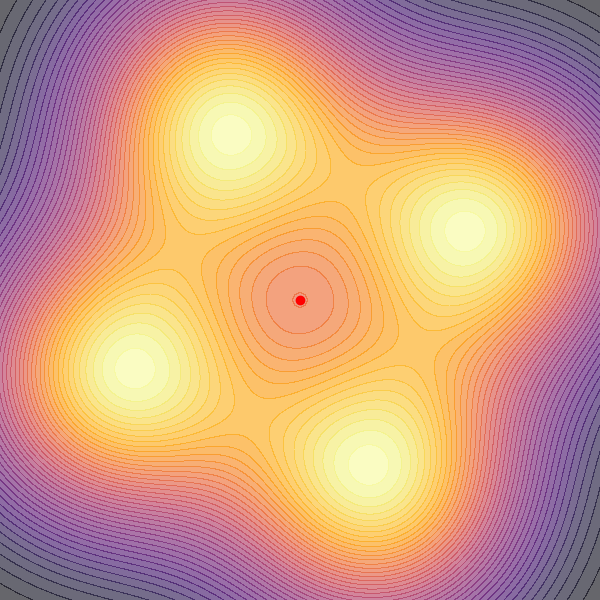}
    \subcaption{\textbf{G1}: $\pi/8$}
    % \label{fig:original}
  \end{minipage}\hfill%
  \begin{minipage}[t]{0.32\textwidth}\vspace{40pt}
    \centering
    \includegraphics[width=0.9\linewidth]{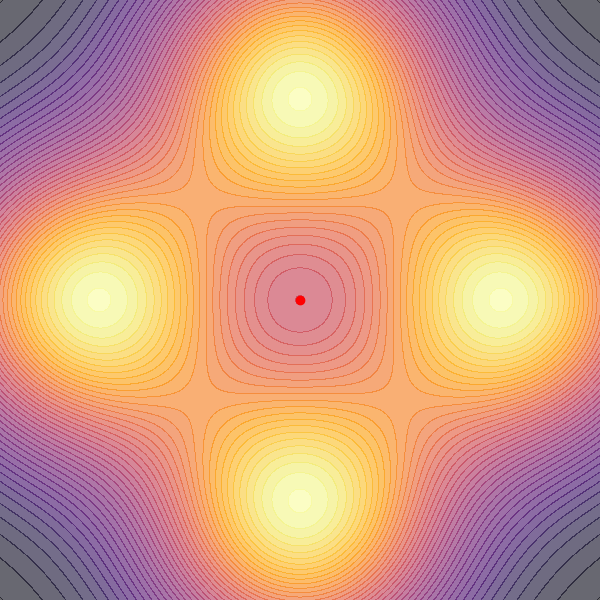}
    \subcaption{\textbf{G2}:  $\pi/4$}
    % \label{fig:original}
  \end{minipage}

  \caption{Reward landscapes: (a) Original environment and the rotated goal variants.}
  \label{fig:reward_landscapes}
\end{figure*}

\section{2D Gaussian Mixture Environment}

We provide in this section detailed information regarding the implementation of the 2D Gaussian mixture environment, as well as ablation evaluation on the dimensionality of the latent space, the structure learned by the steering policy, and the effect of removing the steering policy after fine-tuning.

\subsection{Implementation Details}
\label{appendix:gaussian_env}
We designed a two-dimensional navigation task where the reward landscape is given by a mixture of $4$ Gaussians.  
The agent’s state is its position $(x,y) \in \mathbb{R}^2$, initialized at the origin $(0,0)$.  
Actions are modeled as displacements $(\Delta x, \Delta y)$ applied at each step.  
The instantaneous reward at position $\mathrm{pos}=(x,y)$ is defined as
\begin{equation}
    r(x,y) = \sum_{(c_x,c_y) \in \mathcal{C}} \exp\!\left(-\tfrac{(x-c_x)^2 + (y-c_y)^2}{2\sigma^2}\right),
\end{equation}
where $\mathcal{C}$ is the set of goal centers and $\sigma$ controls the spread of each Gaussian mode.  Each Gaussian mode contributes equally to the reward. This creates a symmetric multimodal environment where all goal regions are equally attractive.  
An episode is successful if the agent reaches within a fixed distance of any goal center. Figure~\ref{fig:reward_landscapes} provides visualizations of the different reward landscapes used in the experiments.

% We consider two variants of this reward landscape:  

% \begin{itemize}
% \item \textbf{Balanced landscape.} Each Gaussian mode contributes equally to the reward. This creates a symmetric multimodal environment where all goal regions are equally attractive.  

% \item \textbf{Unbalanced landscape.} To introduce variability in mode prominence, we assign each Gaussian a random weight $w_i \sim \mathcal{U}(0,1)$.  
% To avoid degenerate scaling while preserving relative preferences, the weights are normalized via a softmax transformation, i.e.
% \[
% \tilde{w}_i = \frac{\exp(w_i)}{\sum_j \exp(w_j)},
% \]
% and the reward is defined as $r(x,y) = \sum_i \tilde{w}_i \, \exp\!\left(-\tfrac{(x-c_x^{(i)})^2 + (y-c_y^{(i)})^2}{2\sigma^2}\right)$.  
% This ensures that all modes remain present but with uneven reward magnitudes, yielding a more challenging and realistic multimodal landscape.
% \end{itemize}

% We refer to these as the \textbf{unbalanced G1} and \textbf{unbalanced G2} environments.  

\subsection{Expert Demonstrations and Policies Rollouts}
\label{appendix:gaussian_env_demos}

We collect demonstrations using a simple heuristic that samples actions to move the agent towards a randomly selected goal center, adding a small random noise to the action. Figure~\ref{fig:q1-env-trajectories} (top) shows the expert demonstration dataset used for the experiments in section~\ref{sec:2D_Gaussians}, along with the corresponding rollouts of the pre-trained policies (middle) as well as a Monte Carlo estimate of the action distribution at $t=0$ (bottom).

The modes learned by the steering policy are illustrated in Figure~\ref{fig:q1-modes-by-z}, where the policy is conditioned on different latent codes $z\in\{0,1,2,3\}$. The consistent trajectories across different $z$ values suggest that the steering policy is able to learn a latent space that organizes noise into meaningful and controllable behavioral modes.

% \begin{figure*}[t!]
%   \centering
%     \centering
%     % Row 1
%     \begin{subfigure}[t]{0.32\textwidth}
%       \centering
%       \includegraphics[width=\linewidth]{figures/2D_Mixture/expert_trajectories_1.png}
%       \caption{Dataset (1 mode)}
%     \end{subfigure}\hfill
%     \begin{subfigure}[t]{0.32\textwidth}
%       \centering
%       \includegraphics[width=\linewidth]{figures/2D_Mixture/expert_trajectories_2.png}
%       \caption{Dataset (2 modes)}
%     \end{subfigure}\hfill
%     \begin{subfigure}[t]{0.32\textwidth}
%       \centering
%       \includegraphics[width=\linewidth]{figures/2D_Mixture/expert_trajectories_4.png}
%       \caption{Dataset (4 modes)}
%     \end{subfigure}

%     \captionof{figure}{Expert datasets with different multimodal behaviors used to pre-train diffusion models to investigate mutual information as a proxy of multimodality. }
%     \label{fig:appendix_gaussian_demo}
% \end{figure*}

\begin{figure*}[t]
  \centering

    % ---------- LEFT: 2 rows × 3 cols ----------
    \begin{minipage}[t]{0.60\textwidth}
      \centering
      % Row 1
      \begin{subfigure}[t]{0.32\textwidth}
        \centering
        \includegraphics[width=0.9\linewidth]{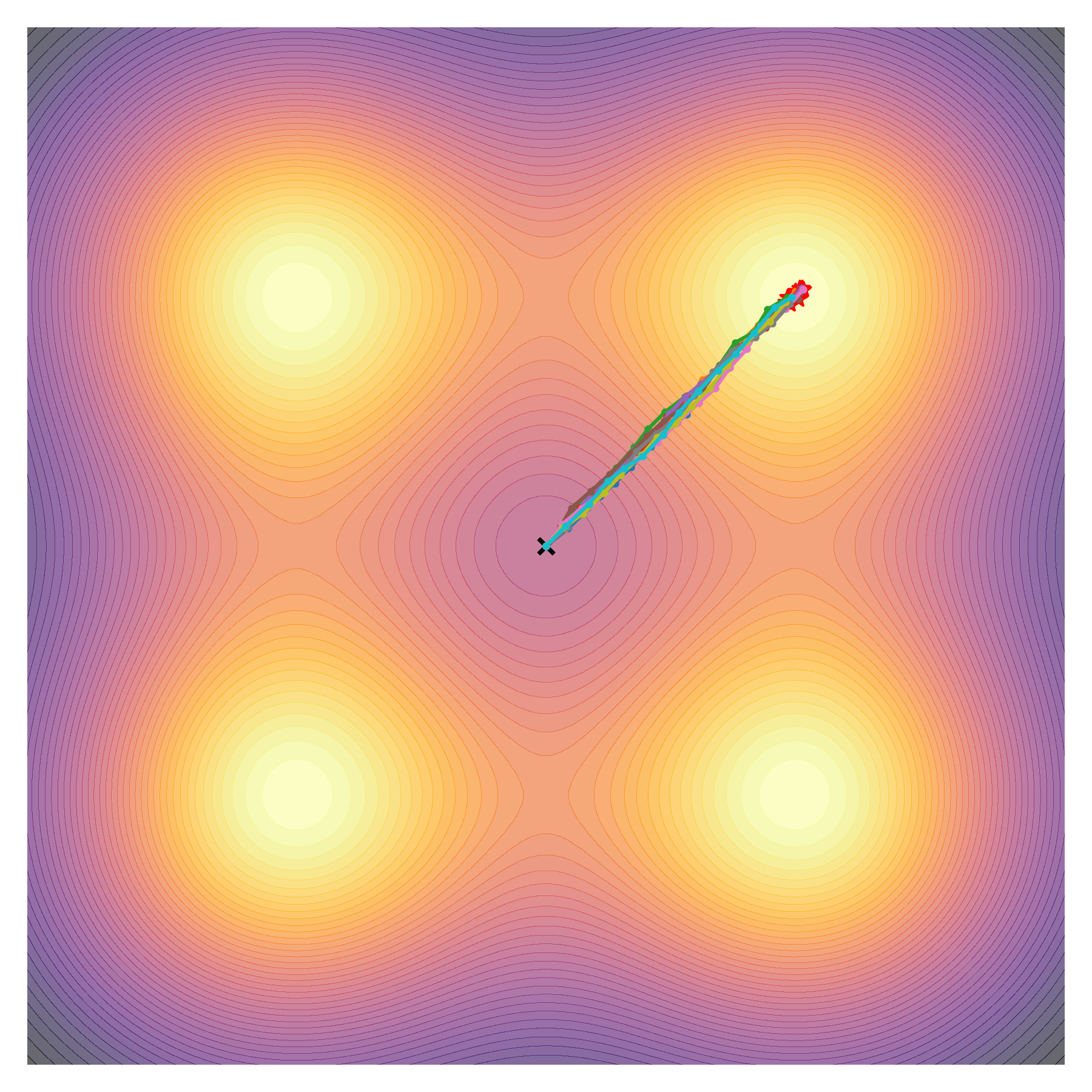}
        \caption{Dataset (1 mode)}
      \end{subfigure}\hfill
      \begin{subfigure}[t]{0.32\textwidth}
        \centering
        \includegraphics[width=0.9\linewidth]{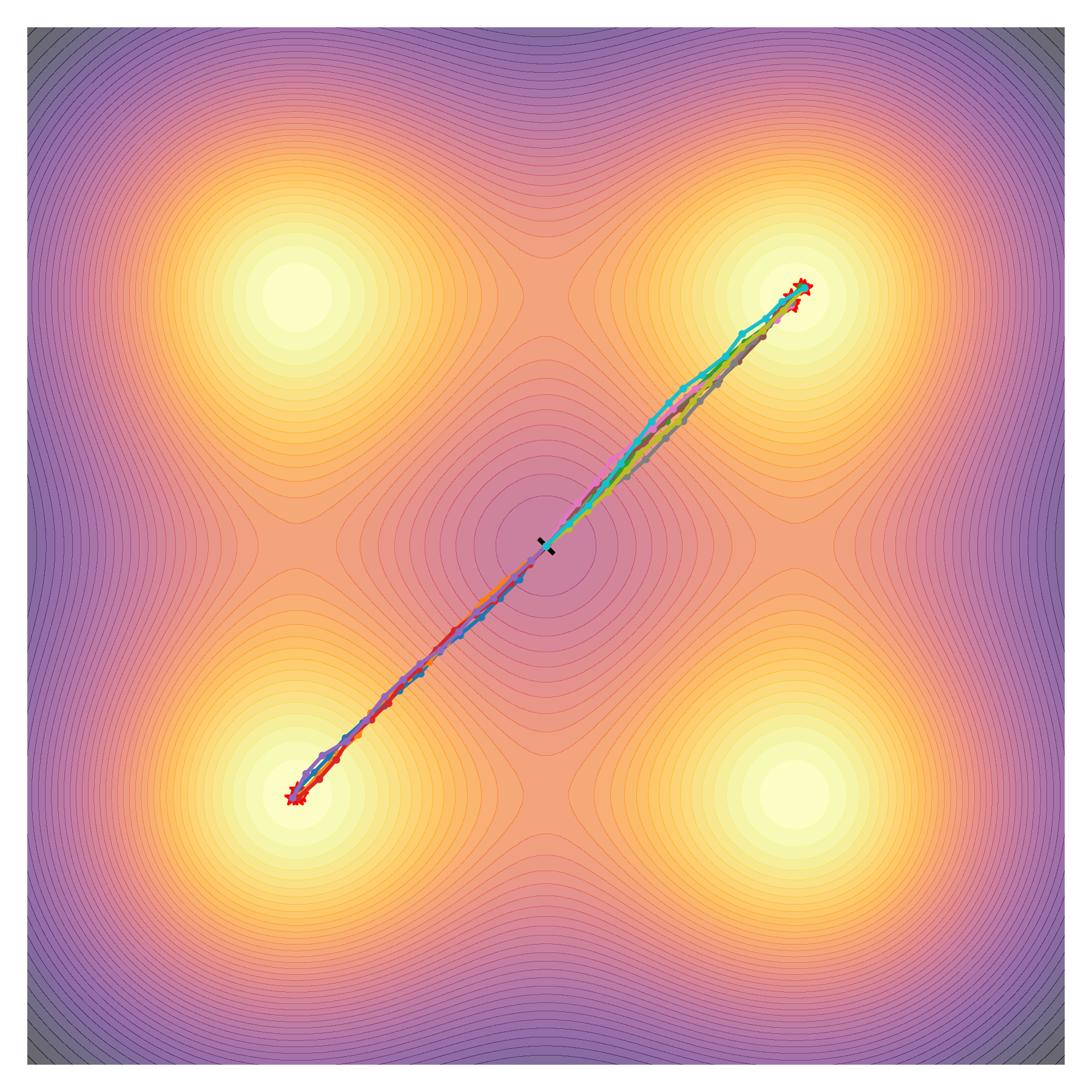}
        \caption{Dataset (2 modes)}
      \end{subfigure}\hfill
      \begin{subfigure}[t]{0.32\textwidth}
        \centering
      \includegraphics[width=0.9\linewidth]{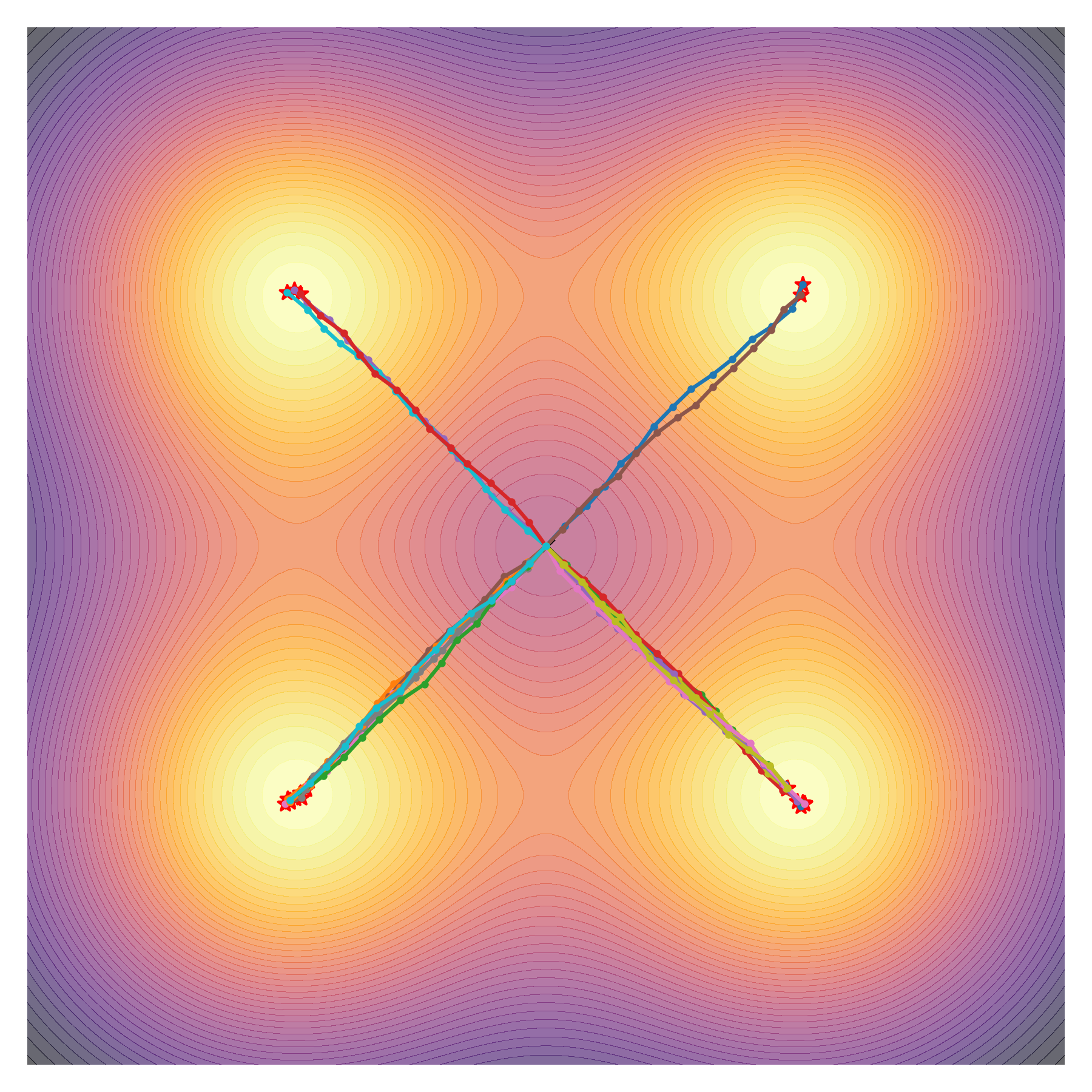}
        \caption{Dataset (4 modes)}
      \end{subfigure}

      \vspace{0.6em}
  % ---------- LEFT: 3 rows × 3 cols ----------
    % Row 1
    \begin{subfigure}[t]{0.32\textwidth}
      \centering
      \includegraphics[width=0.9\linewidth]{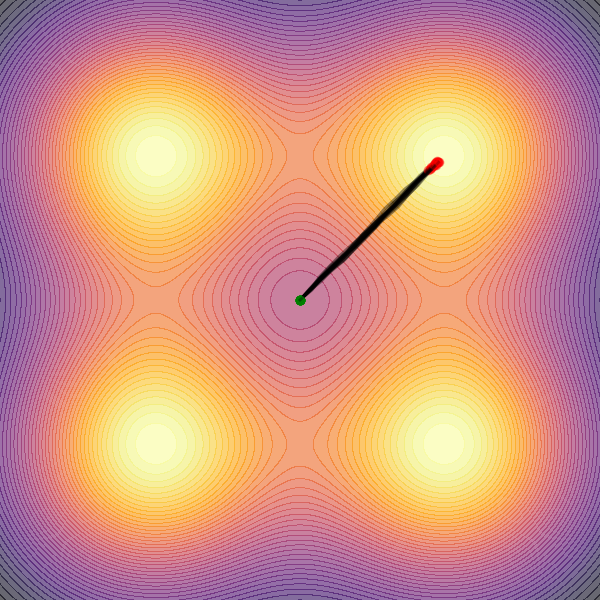}
      \caption{Policy (1 mode)}
    \end{subfigure}\hfill
    \begin{subfigure}[t]{0.32\textwidth}
      \centering
      \includegraphics[width=0.9\linewidth]{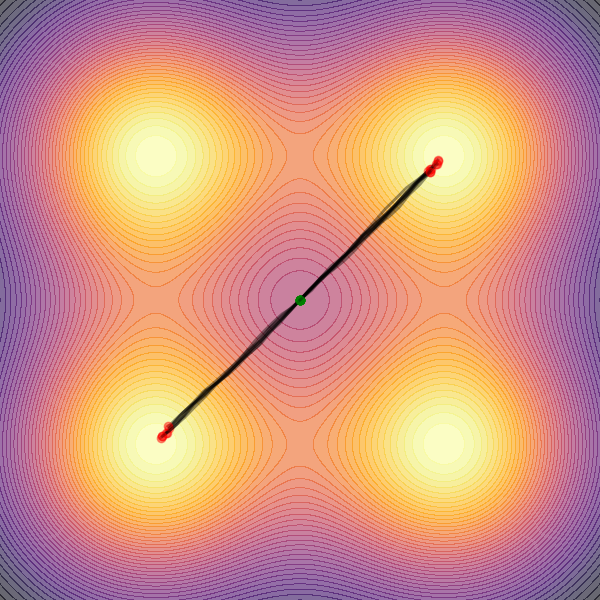}
      \caption{Policy (2 modes)}
    \end{subfigure}\hfill
    \begin{subfigure}[t]{0.32\textwidth}
      \centering
    \includegraphics[width=0.9\linewidth]{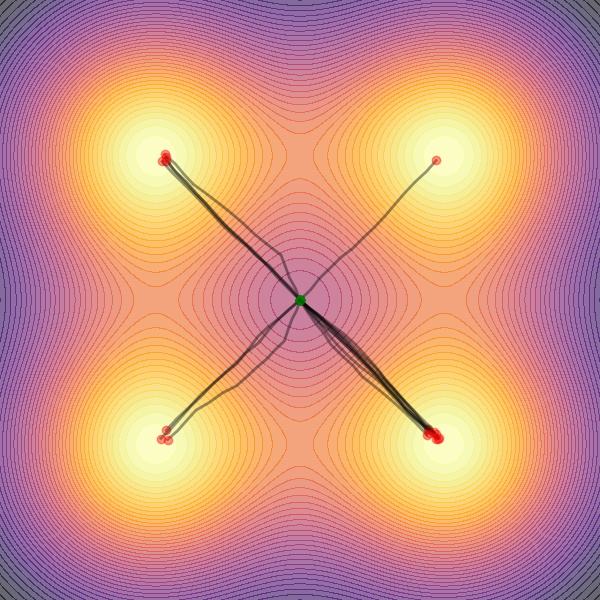}
      \caption{Policy (4 modes)}
    \end{subfigure}

    \vspace{0.6em}

    % Row 2
    \begin{subfigure}[t]{0.32\textwidth}
      \centering
      \includegraphics[width=\linewidth]{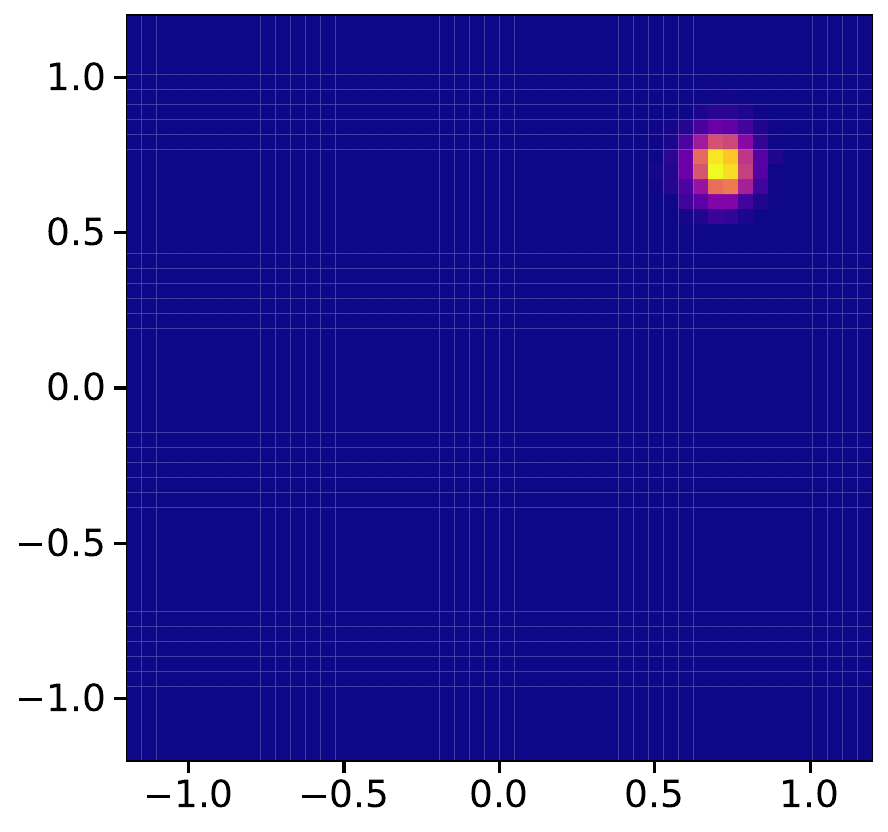}
      \caption{1 mode}
    \end{subfigure}\hfill
    \begin{subfigure}[t]{0.32\textwidth}
      \centering
      \includegraphics[width=\linewidth]{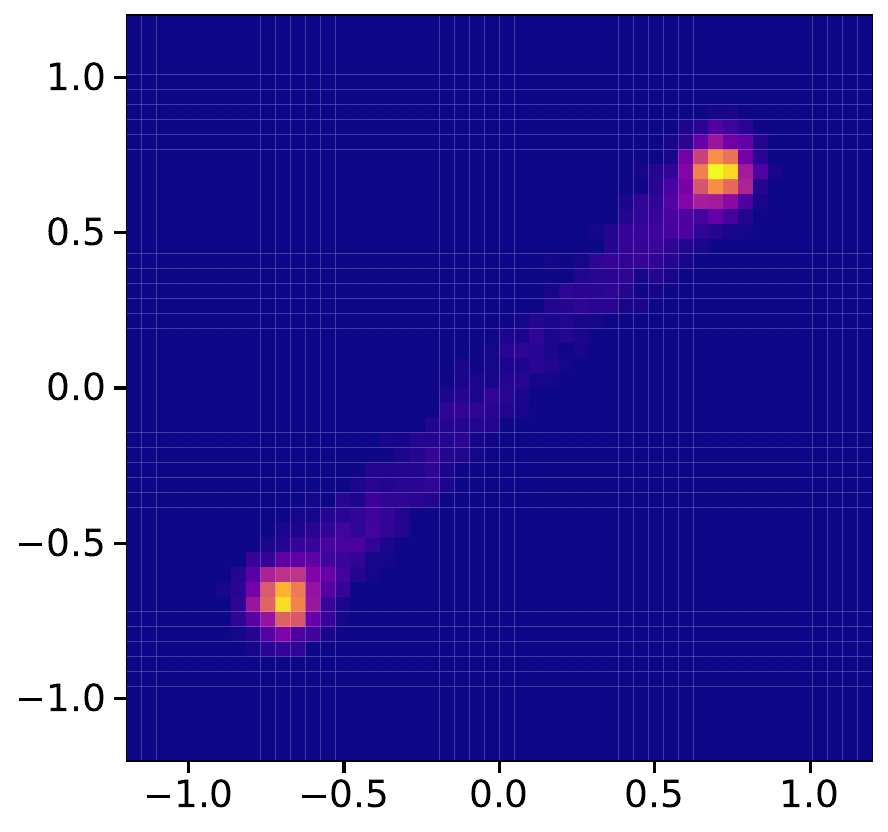}
      \caption{2 modes.}
    \end{subfigure}\hfill
    \begin{subfigure}[t]{0.32\textwidth}
      \centering
      \includegraphics[width=\linewidth]{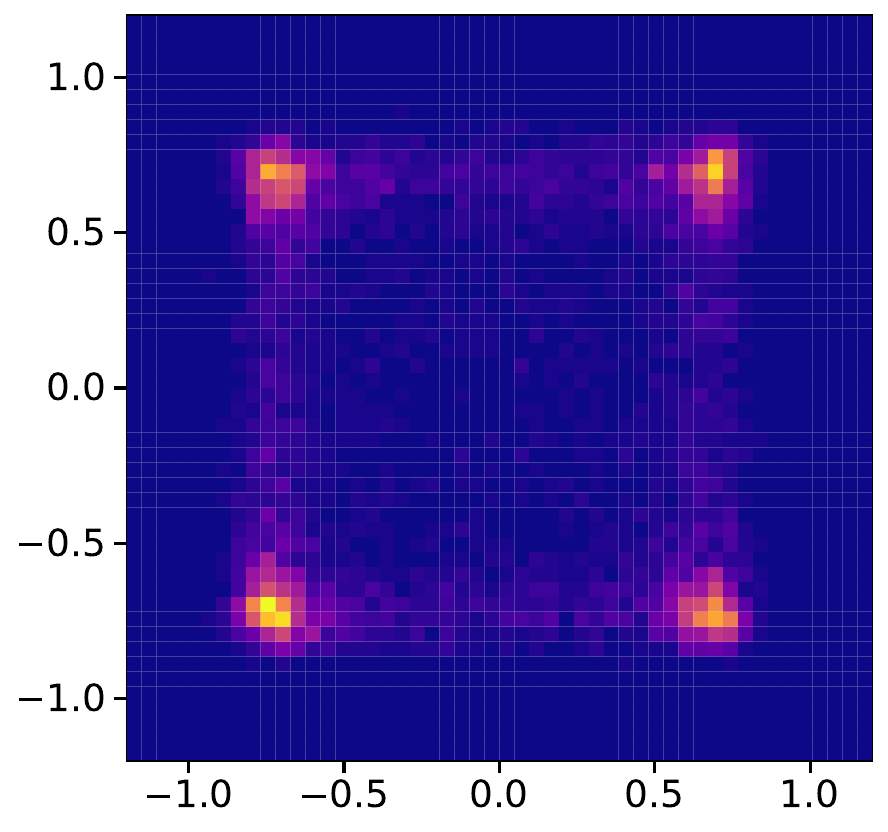}
      \caption{4 modes.}
    \end{subfigure}
    \captionof{figure}{(Top) Trajectories generated from policies pre-trained. % on demonstration datasets containing different degrees of multimodality. 
    (Bottom) Monte Carlo estimate of the action distribution ($\Delta x, \Delta y$) at $t=0$.}
    \label{fig:q1-env-trajectories}
  \end{minipage}
  \hfill
  % ---------- RIGHT: 2 rows × 2 cols (modes by z) ----------
  \begin{minipage}[t]{0.36\textwidth}
    \centering

    \vspace{0.001em}
    \vspace{-3em}
    \begin{subfigure}[t]{0.48\textwidth}
      \centering
      \includegraphics[width=\linewidth]{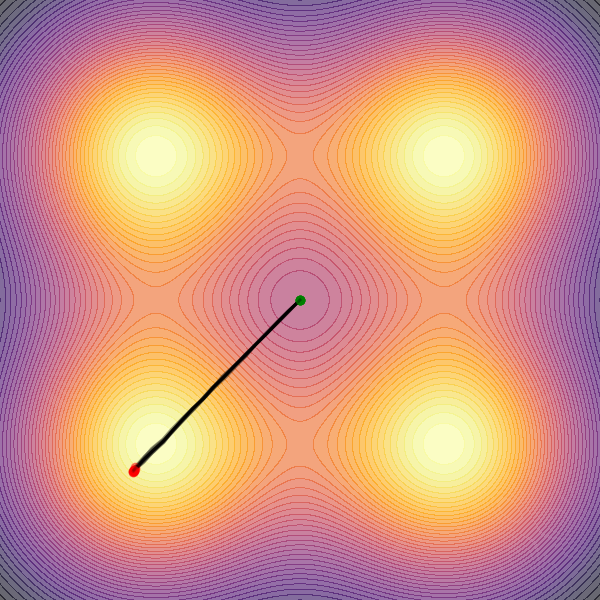}
      \caption{$z=0$}
    \end{subfigure}\hfill
    \begin{subfigure}[t]{0.48\textwidth}
      \centering
      \includegraphics[width=\linewidth]{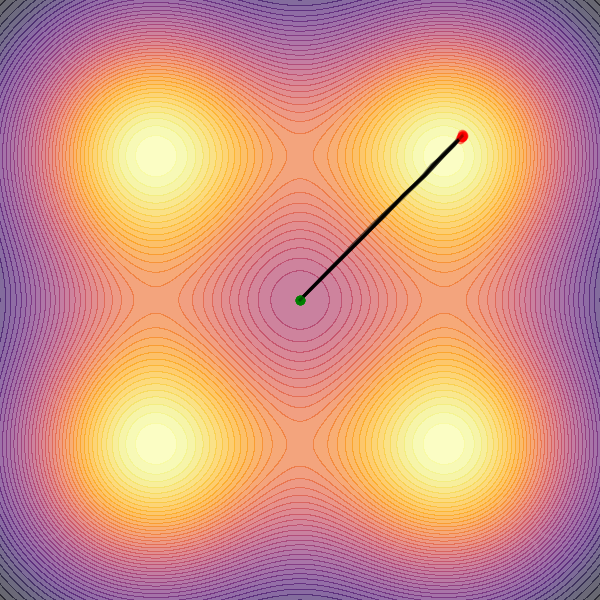}
      \caption{$z=1$}
    \end{subfigure}

    \vspace{0.6em}

    \begin{subfigure}[t]{0.48\textwidth}
      \centering
      \includegraphics[width=\linewidth]{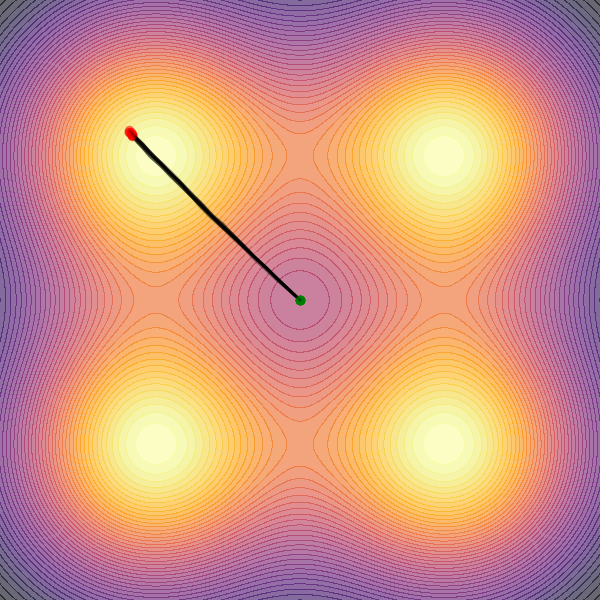}
      \caption{$z=2$}
    \end{subfigure}\hfill
    \begin{subfigure}[t]{0.48\textwidth}
      \centering
      \includegraphics[width=\linewidth]{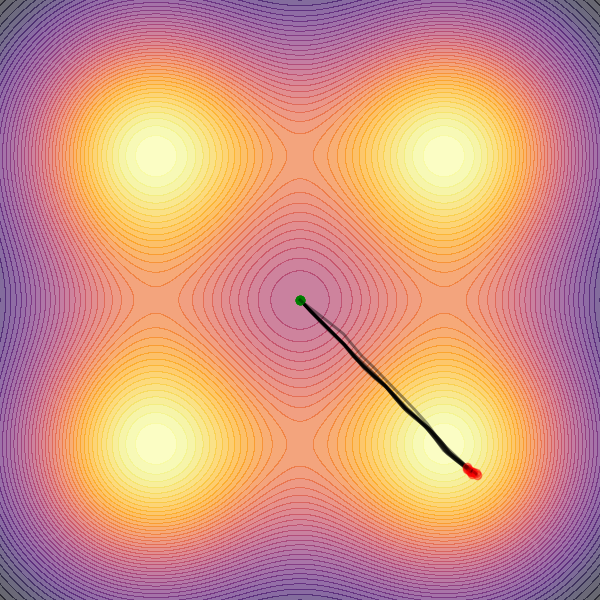}
      \caption{$z=3$}
    \end{subfigure}
    \captionof{figure}{Rollouts generated by steering the policy with latent codes \(z\in\{0,1,2,3\}\).}
    \label{fig:q1-modes-by-z}

  \end{minipage}
\end{figure*}

% \begin{wraptable}[14]{r}{0.58\textwidth} % [9] = reserve ~9 text lines
% \vspace{-0.6\baselineskip}              % tweak if needed
% \centering
% \caption{\textbf{Ablation on the dimensionality of $\mathcal{Z}$.} }
% \resizebox{0.58\textwidth}{!}{%
% \begin{tabular}{lcccc}
% \toprule
% \rowcolor{lightgray}   & \multicolumn{4}{c}{\textbf{G2}} \\
% \rowcolor{lightgray} \textbf{Method} &  $\mathrm{SR}$ &  $\mathrm{SR}_{\mathrm{M}}$ & $\mathrm{mc}@0.80$  & $\mathcal{H}$  \\
% \midrule
% \rowcolor{lightgray} \multicolumn{5}{c}{\textbf{ $|\mathcal{Z}|=4$}} \\
% \texttt{RES[\methodname{}]} & $1.00 \pm 0.00$ & $0.75 \pm 0.00$ & $3.00 / 4$ & $0.74 \pm 0.00$ \\
% \texttt{DPPO[\methodname{}]}  & $1.00 \pm 0.00$ & $0.75 \pm 0.00$ & $3.00 / 4$ & $0.74 \pm 0.00$ \\
% \midrule
% \rowcolor{lightgray} \multicolumn{5}{c}{\textbf{ $|\mathcal{Z}|=8$}} \\
% \texttt{RES[\methodname{}]}  & $1.00 \pm 0.00$ & $1.00 \pm 0.00$ & $4.00 / 4$ & $0.92 \pm 0.00$ \\
% \texttt{DPPO[\methodname{}]} & $0.64 \pm 0.45$ & $0.63 \pm 0.45$ & $2.33 / 4$ & $0.99 \pm 0.00$ \\
% \midrule
% \rowcolor{lightgray} \multicolumn{5}{c}{\textbf{$|\mathcal{Z}|=16$}} \\
% \texttt{RES[\methodname{}]} & $1.00 \pm 0.00$ & $1.00 \pm 0.00$ & $4.00 / 4$ & $0.94 \pm 0.00$ \\
% \texttt{DPPO[\methodname{}]}  & $0.79 \pm 0.00$ & $0.82 \pm 0.00$ & $2.00 / 4$ & $0.94 \pm 0.00$ \\
% \bottomrule
% \end{tabular}%
% }

% \label{tab:z_values}
% \end{wraptable}

\subsection{Dimensionality of $\mathcal{Z}$} 
\label{appendix:dim_z}
We next examine the effect of the latent dimensionality $|\mathcal{Z}|$ on multimodality preservation. We repeat the \textbf{G2} evaluation using the \texttt{RES} and \texttt{DPPO} baselines with mode discovery, varying the number of latent codes. 

\begin{wraptable}[18]{r}{0.58\textwidth} % [9] = reserve ~9 text lines
\vspace{-0.6\baselineskip}              % tweak if needed
\centering
\caption{\textbf{Ablation on the dimensionality of $\mathcal{Z}$.} }
\resizebox{0.58\textwidth}{!}{%
\begin{tabular}{lcc>{\columncolor{lightblue!20}}c c}
\toprule
  & \multicolumn{4}{c}{\textbf{G2}} \\
 \textbf{Method} &  $\mathrm{SR}$ &  $\mathrm{SR}_{\mathrm{M}}$ & $\mathrm{mc}@0.80$  & $\mathcal{H}$  \\
\midrule
\rowcolor{lightgray!50} \multicolumn{5}{c}{\textbf{ $|\mathcal{Z}|=4$}} \\
\texttt{RES[\methodname{}]} & $1.00 \pm 0.00$ & $0.75 \pm 0.00$ & $3.00 / 4$ & $0.74 \pm 0.00$ \\
\texttt{DPPO[\methodname{}]}  & $1.00 \pm 0.00$ & $0.75 \pm 0.00$ & $3.00 / 4$ & $0.74 \pm 0.00$ \\
\midrule
\rowcolor{lightgray!50} \multicolumn{5}{c}{\textbf{ $|\mathcal{Z}|=8$}} \\
\texttt{RES[\methodname{}]}  & $1.00 \pm 0.00$ & $1.00 \pm 0.00$ & $4.00 / 4$ & $0.92 \pm 0.00$ \\
\texttt{DPPO[\methodname{}]} & $0.64 \pm 0.45$ & $0.63 \pm 0.45$ & $2.33 / 4$ & $0.99 \pm 0.00$ \\
\midrule
\rowcolor{lightgray!50} \multicolumn{5}{c}{\textbf{$|\mathcal{Z}|=16$}} \\
\texttt{RES[\methodname{}]} & $1.00 \pm 0.00$ & $1.00 \pm 0.00$ & $4.00 / 4$ & $0.94 \pm 0.00$ \\
\texttt{DPPO[\methodname{}]}  & $0.79 \pm 0.00$ & $0.82 \pm 0.00$ & $2.00 / 4$ & $0.94 \pm 0.00$ \\
\bottomrule
\end{tabular}%
}

\label{tab:z_values}
\end{wraptable}
Results are reported in Table~\ref{tab:z_values}. A dimension of $|\mathcal{Z}|=4$, which matches the ground-truth number of modes, fails to fully capture all task modalities. This limitation stems from our inference model, which distinguishes modes through state coverage and can become sensitive to minor state variations, occasionally treating nearby but distinct states as different modes. 
Increasing dimensionality ($|\mathcal{Z}|=8,16$) improves coverage by promoting exploration of diverse trajectories. However, excessively large latent spaces introduce inefficiencies: for instance, \texttt{DPPO[\methodname{}]} deteriorates at $|\mathcal{Z}|=16$, likely due to a trade-off between task optimization and diversity. These results suggest that latent dimensionality should be tuned to the complexity of the multimodal structure, and that more robust inference models beyond simple state coverage may further improve mode discovery, representing an interesting direction for future work. %Given our design of the discriminator, we can observe how for example a dimensionality of $4$ leads to the loss of some modalities as the modes discovered by the model are diverse enough to be recognised by the discriminator but to not cover well the full state space. Increasing the dimensionality helps exploration and thus the retention of all modalities, but the higher the number of modes, the harder the problem becomes, making more samples inefficient for the training of DDIM, for example. 

\section{Tasks Description}
\label{appendix:manip_tasks}

We evaluate our approach on five robotics tasks. These include three robotic manipulation tasks and one locomotion task implemented within the ManiSkill~\citep{tao2024maniskill3} framework: \emph{Reach}, \emph{Lift}, \emph{Avoid} (re-implemented from D3IL~\citep{jia2024towards}), as manipulation tasks and \emph{ANYmal} as a locomotion task. We further integrate the \emph{Franka Kitchen} environment from D4RL~\citep{fu2020d4rl} as a sequential multi-task setting. Each task (shown in Figure~\ref{fig:tasks}) exhibits distinct forms and degrees of multimodality. Multimodality arises either from goal diversity or, for a fixed goal, from multiple feasible trajectories that lead to successful completion. All manipulation tasks are performed with a Franka Emika Panda robot, where agent actions are parameterized as 6-DoF end-effector delta poses $(\Delta x, \Delta y, \Delta z, \Delta \text{roll}, \Delta \text{pitch}, \Delta \text{yaw})$. The action space for the locomotion task consists of the 12 delta joint position commands controlling the four legs of the ANYmal.

% \begin{figure}[t!]  
%   \centering
%     \centering
%     % Row 1
%     \begin{subfigure}[t]{0.32\textwidth}
%       \centering
%       \includegraphics[width=\linewidth]{figures/tasks/reach.png}
%       \caption{\emph{Reach}}
%     \end{subfigure}\hfill
%     \begin{subfigure}[t]{0.32\textwidth}
%       \centering
%       \includegraphics[width=\linewidth]{figures/tasks/lift.png}
%       \caption{\emph{Lift}}
%     \end{subfigure}\hfill
%     \begin{subfigure}[t]{0.32\textwidth}
%       \centering
%       \includegraphics[width=\linewidth]{figures/tasks/avoid.png}
%       \caption{\emph{Avoid}}
%     \end{subfigure}

%     \captionof{figure}{Visualization of the three ManiSkill tasks used in our evaluation: \emph{Reach}, \emph{Lift}, and \emph{Avoid}.  
% For each task we display four random environment initializations and highlight representative modes for solving the task.}
%     \label{fig:tasks}
% \end{figure}

\paragraph{Reach (2 modes).} In \emph{Reach}, the agent must contact a green sphere while avoiding a gray obstacle; success can be achieved by approaching from either side. These left–right approaches constitute two disjoint solution classes, creating a simple bimodal structure. This task is comparatively simple, as multimodality appears only at the beginning of the trajectory, after which the policy is effectively committed to a single mode. The state space comprises the robot joint positions and velocities, the end-effector pose, as well as goal and bar poses. The maximum episode length is $100$ steps. The task is considered to be successful if the agent reaches the goal within a pre-defined threshold

\paragraph{Lift (2 modes).} In \emph{Lift}, the agent must lift a peg into vertical position. The peg can be grasped and lifted upright from either the red or blue side, yielding multiple valid grasping strategies. Here, multimodality reflects the existence of several grasp affordances around the object, which lead to distinct grasp–lift trajectories even when the final goal is identical. The initial randomization of object configurations increases the ambiguity and difficulty of separating modes.  The state space comprises the robot joint positions and velocities, the end-effector pose, as well as the peg pose. The maximum episode length is $200$ steps. The task is considered to be successful if the peg is successfully lifted (assessed through the pose of the object) and stable.

\paragraph{Avoid (24 modes).} In the \emph{Avoid} task, the agent must cross the green line by avoiding the obstacles in the table. This is the most challenging as numerous modalities emerge later in the trajectory, each corresponding to a distinct avoidance strategy with different path lengths. In this case, only the initial end-effector position is randomized at reset, while the obstacle remains fixed, emphasizing the diversity of possible avoidance strategies. The state representation encompasses the end-effector’s desired position and actual position in Cartesian space, with the caveat that the robot’s height (z position) remains fixed. The actions are represented by the desired velocity of the robot along the x and y axes. The maximum episode length is $300$ steps. The task is considered to be successful if the robot-end-effector reaches the green finish line.

\paragraph{ANYmal Locomotion (4 modes).}
In this locomotion task, a quadrupedal ANYmal robot must navigate to one of four goal locations placed at fixed positions in the environment. Multimodality arises from goal diversity: each goal corresponds to a distinct target direction and therefore induces different optimal locomotion strategies and turning behaviors. Demonstrations are generated by training four separate RL agents, each optimized to reach a single goal, producing unimodal expert trajectories for each target. When combined, these demonstrations form a four-modal behavior distribution that the policy must preserve. The state space includes proprioceptive robot observations (joint states, base pose, and velocities) and the relative goal position with respect to the agent position. The maximum episode length is $200$ steps, and an episode is successful if the robot reaches any goal within a predefined tolerance.

\paragraph{Franka Kitchen (24 modes).}The Franka Kitchen environment from D4RL~\citep{fu2020d4rl} contains demonstrations of a robot manipulating several articulated objects (microwave, kettle, burner, light switch). We train on the mixed demonstration dataset, which contains trajectories performing different task combinations in varying orders, but never completing all four evaluation subtasks sequentially. As a result, multimodality emerges both from the diversity of partial task orders and from multiple valid ways to interact with each object. For evaluation, we follow the common benchmark and consider four subtasks—\texttt{microwave}, \texttt{kettle}, \texttt{bottom burner}, \texttt{light switch}. Success is achieved when the policy completes three of the four subtasks within the episode, possibly in any order. The state space consists of robot joint states, end-effector pose, and object poses; the maximum horizon is $280$ steps.

All environments provide dense or intermediate reward functions to support fine-tuning, and we employ a heuristic to identify the mode associated with each trajectory, enabling consistent evaluation of multimodality. Additional implementation details will be available upon the release of the codebase.

\begin{figure}[t!]
  \centering
  % Five tasks in one row
  \begin{subfigure}[t]{0.18\textwidth}
    \centering
    \includegraphics[width=\linewidth]{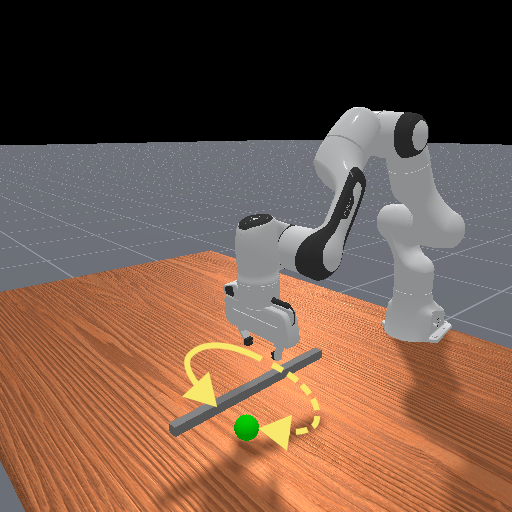}
    \caption{\emph{Reach}}
  \end{subfigure}\hfill
  \begin{subfigure}[t]{0.18\textwidth}
    \centering
    \includegraphics[width=\linewidth]{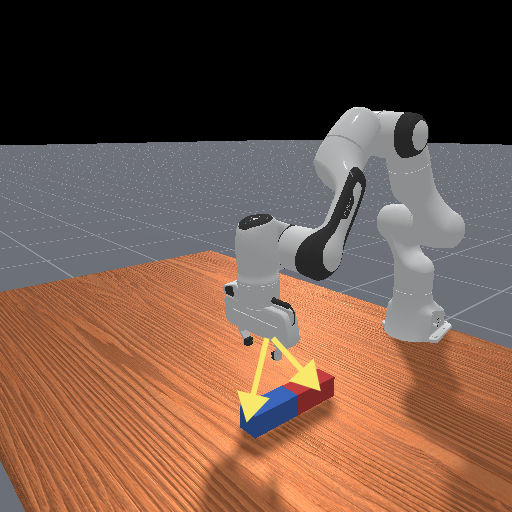}
    \caption{\emph{Lift}}
  \end{subfigure}\hfill
  \begin{subfigure}[t]{0.18\textwidth}
    \centering
    \includegraphics[width=\linewidth]{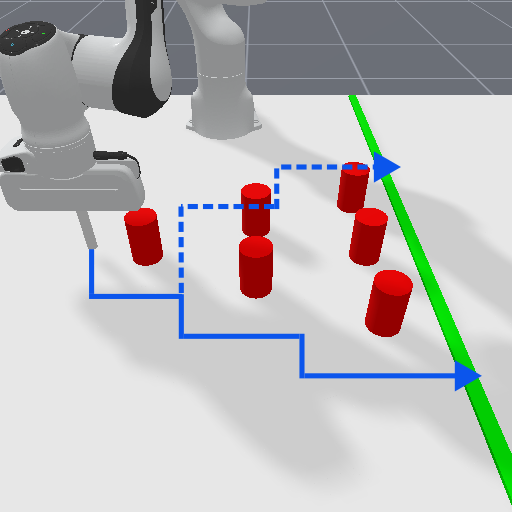}
    \caption{\emph{Avoid}}
  \end{subfigure}\hfill
  \begin{subfigure}[t]{0.18\textwidth}
    \centering
    \includegraphics[width=\linewidth]{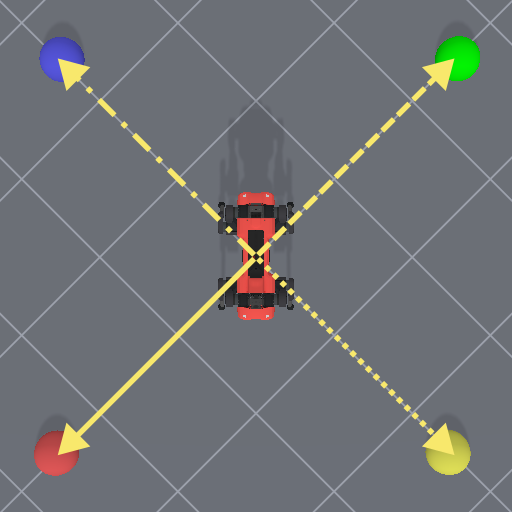}
    \caption{\emph{ANYmal}}
  \end{subfigure}\hfill
  \begin{subfigure}[t]{0.23\textwidth}
    \centering
    \includegraphics[width=\linewidth]{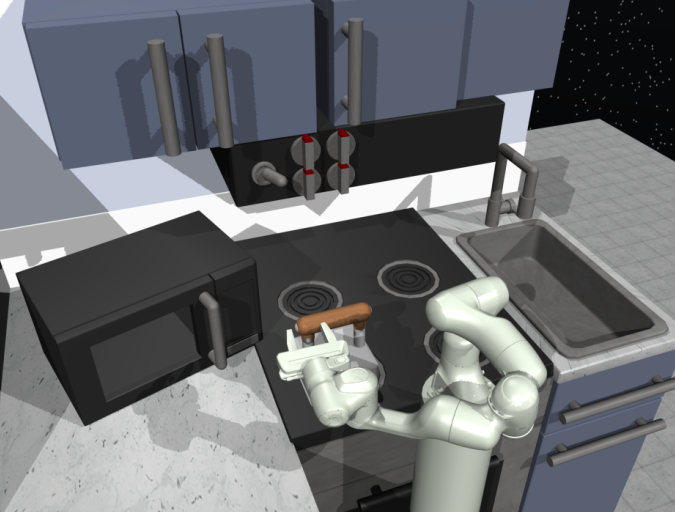}
    \caption{\emph{Franka Kitchen}}
  \end{subfigure}

  \caption{Visualization of the five ManiSkill tasks used in our evaluation. 
  For each task, except the \emph{Franka Kitchen}, we highlight representative modes for solving the task.}
  \label{fig:tasks}
\end{figure}

% \paragraph{Implementation Details}
% For each task, we collect $1000$ demos with a motion planner for each task and pre-train the diffusion model for $1000$ epochs. \al{Can we evaluate the difference between a trajectory trained with RL and the motion planning demos to quantify the reward landscape shift?} The baseline models are fine-tuned using only the task rewards, while the regularized version first pre-trains the steering policy with only the mode discovery objective.\al{Some other details}.

\section{Ablation Experiments}

\subsection{Method Ablations}
\label{appendix:manipulation_ablation}

\begin{wrapfigure}{r}{0.5\textwidth}  % r = right, l = left
\vspace{-3.9\baselineskip}             % adjust vertical spacing if needed
\centering
\includegraphics[width=0.9\linewidth]{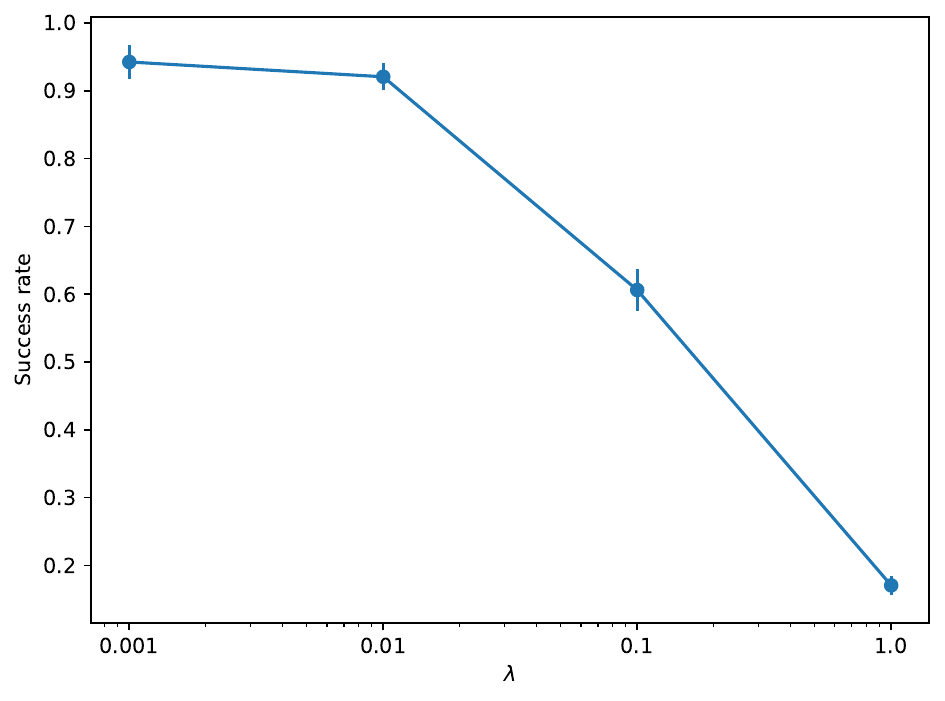}

\caption{Impact of the regularization coefficient $\lambda$ on the task success rate. }
\label{fig:lambda}
\vspace{-0.6cm} 
\end{wrapfigure}

We first study the effect of the regularization weight $\lambda$ on task performance, focusing on the \emph{Lift} task with the \texttt{RES[\methodname{}]} baseline. Figure~\ref{fig:lambda} shows that as $\lambda$ increases, the intrinsic reward increasingly dominates over the task reward, leading to a drop in success rate. This illustrates the trade-off: stronger regularization favors diversity at the expense of task performance.

Next, we analyze the impact of (i) pre-training with only the mode-discovery reward (\texttt{[NO-FT \methodname{}]}) and (ii) omitting fine-tuning of the inference model and steering policy when adapting the main policy with another fine-tuning technique (\texttt{[NO-PRE \methodname{}]}), (iii) removing the curriculum stage during the mode-discovery phase (\texttt{[NO-CURR \methodname{}]}). These ablations, reported in Table~\ref{tab:ablation}) for the \emph{Lift} task with \texttt{RES[\methodname{}]}, reveal that all factors negatively affect performance. In particular, disabling fine-tuning of the inference model and steering policy is catastrophic: the mutual-information signal becomes uninformative as the policy is driven toward out-of-distribution states relative to pre-training.

\begin{table*}[t!]
\centering
\begin{minipage}[t]{0.48\textwidth}
  \centering
  \caption{Ablation experiments on key design choices: (i) pre-training with only the mode-discovery reward, (ii) omitting fine-tuning of the inference model and steering policy when adapting the main policy with another fine-tuning technique, and (iii) removing the curriculum stage during the mode-discovery phase.}
  \label{tab:ablation}
  \resizebox{\linewidth}{!}{%
  \begin{tabular}{lcc>{\columncolor{lightblue!20}}c c}
    \toprule
    \textbf{Method} &  $\mathrm{SR}$ &  $\mathrm{SR}_{\mathrm{M}}$ & $\mathrm{mc}@0.80$  & $\mathcal{H}$  \\
    \midrule
    \texttt{PRE} & $0.14 \scriptscriptstyle\pm 0.01$ & $0.15 \scriptscriptstyle\pm 0.01$ & $0.00 / 2$& $0.97 \scriptscriptstyle\pm 0.01$ \\
    \midrule
    \texttt{RES[\methodname{}]} & $0.99 \scriptscriptstyle\pm 0.00$ & $0.99 \scriptscriptstyle\pm 0.00$ & $2.00 / 2$& $1.00 \scriptscriptstyle\pm 0.00$ \\
    \midrule
    \texttt{RES[NO-PRE \methodname{}]} & $0.91 \scriptscriptstyle\pm 0.04$ & $0.79 \scriptscriptstyle\pm 0.11$ & $1.33 / 2$& $0.74 \scriptscriptstyle\pm 0.08$  \\
    \texttt{RES[NO-FT \methodname{}]} & $0.00 \scriptscriptstyle\pm 0.00$ & $0.00 \scriptscriptstyle\pm 0.00$ & $0.00 / 2$& $0.00 \scriptscriptstyle\pm 0.00$  \\
    \texttt{RES[NO-CURR \methodname{}]}& $0.85 \scriptscriptstyle\pm 0.08$ & $0.83 \scriptscriptstyle\pm 0.08$ & $1.33 / 2$& $0.95 \scriptscriptstyle\pm 0.05$   \\
    \bottomrule
  \end{tabular}
  }
\end{minipage}\hfill
\begin{minipage}[t]{0.48\textwidth}
  \centering
  \caption{Ablation experiment on removing the steering policy after fine-tuning with \methodname{}}
  \label{tab:remove_steering}
  \resizebox{\linewidth}{!}{%
  \begin{tabular}{lcc>{\columncolor{lightblue!20}}c c}
    \toprule
    \textbf{Method} &  $\mathrm{SR}$ &  $\mathrm{SR}_{\mathrm{M}}$ & $\mathrm{mc}@0.80$  & $\mathcal{H}$  \\
    \midrule
    \texttt{PRE} & $0.14 \scriptscriptstyle\pm 0.01$ & $0.15 \scriptscriptstyle\pm 0.01$ & $0.00 / 2$& $0.97 \scriptscriptstyle\pm 0.01$ \\
    \midrule
    \rowcolor{lightgray!50} \multicolumn{5}{c}{\emph{With Steering Policy}} \\
    \midrule
    \texttt{RES[\methodname{}]} & $0.99 \scriptscriptstyle\pm 0.00$ & $0.99 \scriptscriptstyle\pm 0.00$ & $2.00 / 2$& $1.00 \scriptscriptstyle\pm 0.00$ \\
    \texttt{DPPO[\methodname{}]} & $0.99 \scriptscriptstyle\pm 0.00$ & $0.55 \scriptscriptstyle\pm 0.07$ & $1.00 / 2$& $0.06 \scriptscriptstyle\pm 0.04$ \\
    \midrule
    \rowcolor{lightgray!50} \multicolumn{5}{c}{\emph{Without Steering Policy (Random Sampling)}} \\
    \midrule
    \texttt{RES[\methodname{}]} & $0.95 \scriptscriptstyle\pm 0.02$ & $0.94 \scriptscriptstyle\pm 0.02$ & $2.00 / 2$& $0.93 \scriptscriptstyle\pm 0.03$ \\
    \texttt{DPPO[\methodname{}]} & $0.99 \scriptscriptstyle\pm 0.00$ & $0.58 \scriptscriptstyle\pm 0.06$ & $1.00 / 2$& $0.08 \scriptscriptstyle\pm 0.03$\\
    \bottomrule
  \end{tabular}
  }
\end{minipage}
\end{table*}

Finally, we evaluate whether policies fine-tuned with \methodname{} retain multimodality and performance once the steering head is removed, i.e., actions are again driven by the original latent noise prior. Table~\ref{tab:remove_steering} reports success and multimodality metrics for only the \texttt{DPPO[\methodname{}]} and \texttt{RES[\methodname{}]} on \emph{Lift}, as removing the steering on the \texttt{DSRL} baseline would regress the performance back to the original pre-trained policy. The residual baseline shows minimal degradation after removing the steering head, indicating that residual updates internalize the discovered modes into the policy. Similarly, \texttt{DPPO[\methodname{}]} exhibits similar performance with respect to the version including the steering head.

We hypothesize that \methodname{}’s regularization on the steering output, penalizing deviations from the original normal noise, encourages compatibility between the learned behaviors and the base diffusion noise. During fine-tuning, steering guides exploration over $z$ to expose distinct modes, while the regularizer keeps the induced noise close to the prior, allowing the policy to absorb mode structure without depending on explicit steering at inference. Consequently, \texttt{RES[\methodname{}]} can execute diverse behaviors when sampling from the unmodified prior, preserving multimodality with limited impact on task success and making it a strong candidate for fine-tuning generative policies.

\subsection{Mode Stability Analysis}
\label{appendix:mode stability}

To assess the stability of the latent steering variable discovered by \methodname{}, we evaluate whether different training seeds induce consistent mappings between latent codes and environment modes in the \emph{ANYmal} task. We consider $3$ different seeds, all steering policies were trained under identical conditions and number of fine-tuning epochs, and for each checkpoint, we rolled out $1024$ episodes from randomized initial states. We predefined the sampled latent code $z$ for each episode (shared across the evaluations of the three seeds) and recorded the resulting environment-defined mode. 

The confusion matrices in Figure~\ref{fig:confusion_matrices} show that all checkpoints exhibit highly consistent mode assignments across initial-state perturbations; however, the third checkpoint collapses two latent codes into the same mode, mirroring the behavior observed in our 2D Gaussian Mixture analysis, where small state perturbations cause our diversity objective to treat trajectories ending in the same mode as distinct, thereby compressing the latent space.   A qualitative visualization of the modes learned by the three checkpoints is shown in Figure~\ref{fig:modes_anymal}. We further report the success rate per mode in Table~\ref{tab:sr_modes}.

To quantify seed-level agreement, we compare the checkpoints using three metrics: (i) the Normalized Mutual Information (NMI), which measures the similarity of the mode distributions produced across seeds; (ii) the Adjusted Rand Index (ARI), which evaluates the alignment of the underlying mode-cluster structure; and (iii) a Z-Consistency score, defined as the fraction of latent codes whose dominant mode matches across seeds. As shown in Table~\ref{tab:z_stability_pairwise}, NMI and ARI remain high, indicating that all seeds recover consistent sets of behavioral modes, while the collapsed latent in the third checkpoint naturally leads to non-perfect scores. The Z-Consistency score is zero for the first checkpoint comparisons, confirming that although the learned modes are stable, the specific latent-code assignments are not necessarily preserved across seeds, reflecting the inherent permutation symmetry of unsupervised latent discovery.

\begin{figure}[t!]
  \centering
\centering
\includegraphics[width=\linewidth]{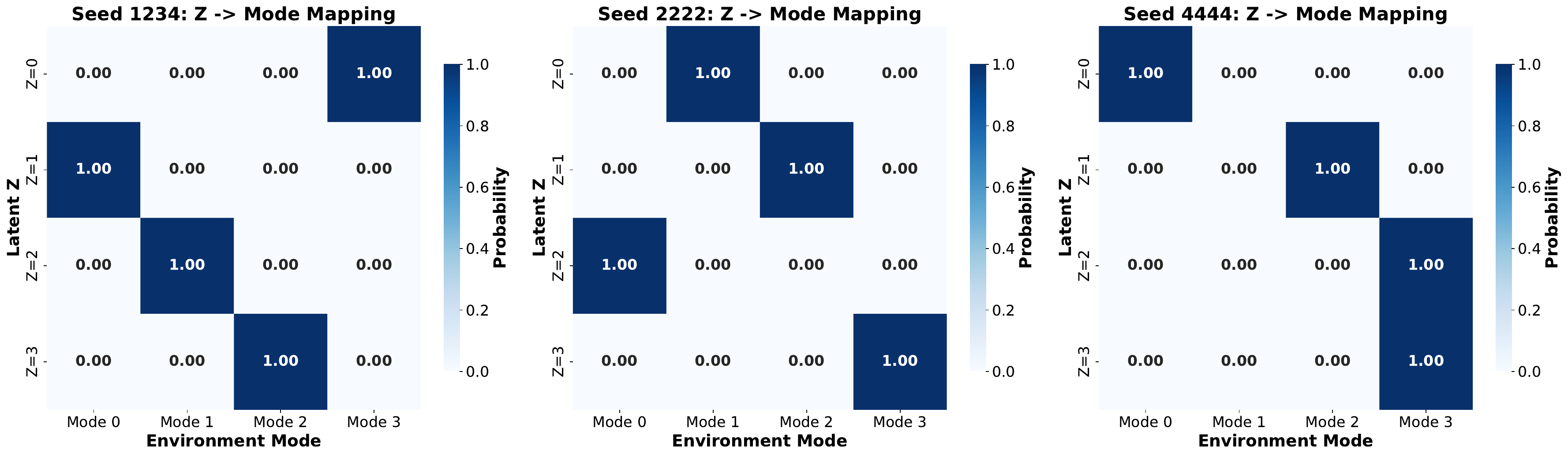}
\caption{Confusion matrices of the mappings from the latent $z\in\mathcal{Z}$ to the ground truth environment's modes.}
\label{fig:confusion_matrices}

\end{figure}

\begin{figure}[t]
\centering
\label{fig:locomotion_skills_viz}

\begin{subfigure}[t]{0.3\linewidth}
    \centering
    \includegraphics[width=\linewidth]{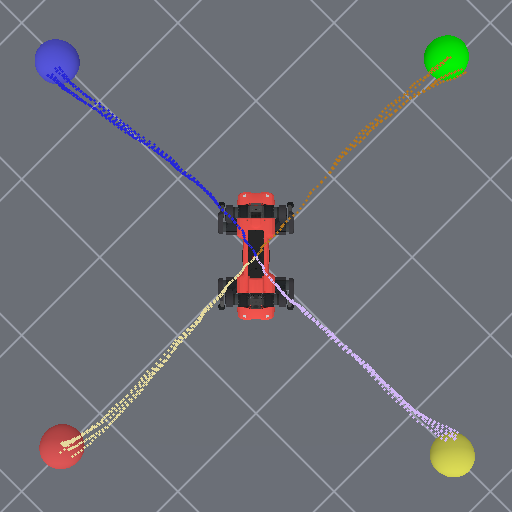}
    \caption{Checkpoint $1$ (seed 1234).}
    \label{fig:cp1-seed1234}
\end{subfigure}
\hfill
\begin{subfigure}[t]{0.3\linewidth}
    \centering
    \includegraphics[width=\linewidth]{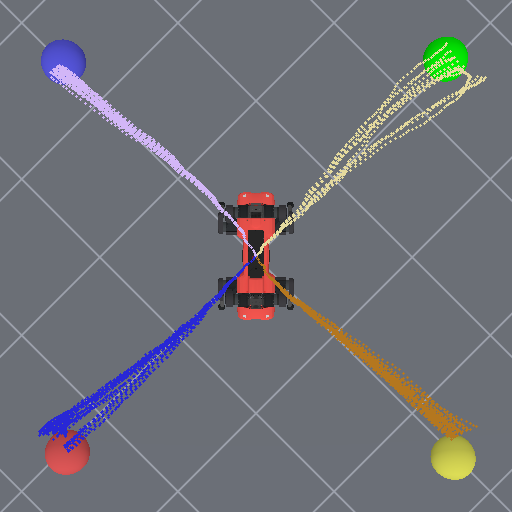}
    \caption{Checkpoint $2$ (seed 2222).}
    \label{fig:cp1-seed1234-b}
\end{subfigure}
\hfill
\begin{subfigure}[t]{0.3\linewidth}
    \centering
    \includegraphics[width=\linewidth]{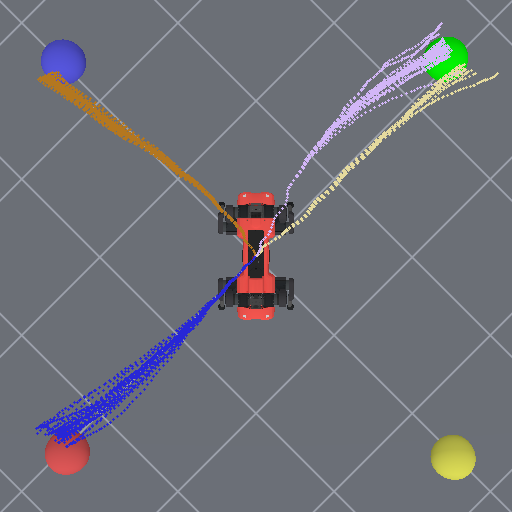}
    \caption{Checkpoint $3$ (seed 4444).}
    \label{fig:cp1-seed1234-c}
\end{subfigure}

\caption{Qualitative visualization of the trajectories distributions for different checkpoints, where different colors correspond to different $z\in\mathcal{Z}$.}
\label{fig:modes_anymal}
\end{figure}

\begin{table*}[t]

\centering
\begingroup
% \scriptsize
% \setlength{\tabcolsep}{2pt} % tighter columns (default ~6pt)
\begin{minipage}{0.465\linewidth}
\centering
\caption{Success rate per mode.}
\label{tab:sr_modes}
\begin{tabular}{@{}cccc@{}}
\toprule
 \textbf{Mode} & \textbf{C1} $(\uparrow)$ & \textbf{C2} $(\uparrow)$ & \textbf{C3} $(\uparrow)$ \\
\midrule
\rowcolor{lightgray!50}0 & 1.00 & 1.00 & 1.00 \\
1 & 0.97 & 0.98 & -- \\
\rowcolor{lightgray!50}2 & 1.00 & 0.91 & 0.95 \\
3 & 0.97 & 0.89 & 0.92 \\
\bottomrule
\end{tabular}
\end{minipage}
\hfill
\begin{minipage}{0.465\linewidth}
\centering
\caption{Pairwise comparison metrics for latent Z stability across checkpoints. }
\label{tab:z_stability_pairwise}
\begin{tabular}{@{}lccc@{}}
\toprule
 \textbf{Pair} & \textbf{NMI} $(\uparrow)$ & \textbf{ARI} $(\uparrow)$ & \textbf{Z Consistency} $(\uparrow)$ \\
\midrule
\rowcolor{lightgray!50}\texttt{C1 vs C2} & 1.00 & 1.00 & 0.00 \\
\texttt{C1 vs C3} & 0.87 & 0.74 & 0.00 \\
\rowcolor{lightgray!50}\texttt{C2 vs C3} & 0.87 & 0.74 & 0.50 \\
\bottomrule
\end{tabular}
\end{minipage}
\endgroup

\end{table*}

\subsection{Noise and Dynamics Perturbations}
\label{sec:noise_dyn_perturbations}

This section evaluates the robustness of the proposed method to environmental perturbations beyond reward shifts, specifically focusing on observation noise and dynamics alterations. We conduct these experiments in the \emph{ANYmal} environment. For the observation-noise ablation, we inject Gaussian noise with standard deviation $0.01$ into the observations prior to feeding them into the policy. For the dynamics-shift ablation, we immobilize the first joint of one leg by forcing the corresponding action to zero before each simulator step. These perturbation magnitudes were chosen to avoid completely destabilizing the pre-trained policy: the method assumes a non-trivial degree of multimodality in the underlying model, and more severe interventions (e.g.\, blocking the second joint) caused immediate falls, thereby collapsing behavioral diversity and falling outside the scope of this study.

We evaluated the steering policy and compared its performance against the same policy trained under the same shifts without the \methodname{} regularization, using the same metrics introduced in the previous section. We further included the performance of the pre-trained model evaluated with the suggested shifts. Taken together, the results in Tables~\ref{tab:obs_noise} and~\ref{tab:dynamics_perturbation} show that \methodname{} consistently preserves multimodal behavior under observation noise and maintains distinct latent-behavioral modes even when the underlying system dynamics are perturbed. Although dynamic shifts reduce overall task performance, the fine-tuned steering policy still succeeds on two of the four environment modes (per-mode success: (0) 0.02, (1) 0.92, (2) 0.79, (3) 0.98), corresponding to $\mathrm{mc}@80=2/4$. This demonstrates the robustness of the discovered behavioral modes, while also revealing room for improvement in handling stronger dynamic variations.

\begin{table*}[t]
\centering
\begin{minipage}{0.465\linewidth}
\centering
\caption{Observation noise perturbation. }
\label{tab:obs_noise}
\resizebox{\linewidth}{!}{%
\centering
\begin{tabular}{lcc>{\columncolor{lightblue!20}}c c}
\toprule
\textbf{Method} & $\mathrm{SR}$$(\uparrow)$ & $\mathrm{SR}_{\mathrm{M}}$ $(\uparrow)$& $\mathrm{mc}@0.80$ $(\uparrow)$& $\mathcal{H}$ $(\uparrow)$\\
\midrule
\rowcolor{lightgray!50}\texttt{PRE} & $0.32 $ & $0.31  $ & \cellcolor{lightblue!20}$0.00 / 4$& $0.99 $  \\
\midrule
\texttt{DSRL} & $1.00 $ & $0.25  $ & $1.00 / 4$& $0.00  $ \\
\rowcolor{lightgray!50}\texttt{DSRL[\methodname{}]} & $0.96 $ & $0.96  $ & \cellcolor{lightblue!20}$4.00 / 4$& $0.99 $  \\
\bottomrule
\end{tabular}%
}
\end{minipage}\hfill
\begin{minipage}{0.465\linewidth}
\centering
\caption{Dynamics shift perturbation.}
\label{tab:dynamics_perturbation}
\centering
\resizebox{\linewidth}{!}{%
\begin{tabular}{lcc>{\columncolor{lightblue!20}}c c}
\toprule
\textbf{Method} & $\mathrm{SR}$ $(\uparrow)$& $\mathrm{SR}_{\mathrm{M}}$ $(\uparrow)$ & $\mathrm{mc}@0.80$ $(\uparrow)$& $\mathcal{H}$ $(\uparrow)$ \\
\midrule
\rowcolor{lightgray!50}\texttt{PRE} & $0.05 $ & $0.06  $ & \cellcolor{lightblue!20}$0.00 / 4$& $0.90 $  \\
\midrule
\texttt{DSRL} & $0.98 $ & $0.24  $ & $1.00 / 4$& $0.00  $ \\
\rowcolor{lightgray!50}\texttt{DSRL[\methodname{}]} & $0.69 $ & $0.69  $ & \cellcolor{lightblue!20}$2.00 / 4$& $0.99 $  \\
\bottomrule
\end{tabular}%
}
\end{minipage}
\end{table*}

\subsection{Successful Kitchen Tasks Sequences}

Table~\ref{tab:kitchen_sequences} reports the distinct successful task sequences executed by each method in the Franka Kitchen environment. The pre-trained policy exhibits multiple valid one- and two-task sequences, while all baselines collapse to a single sequence per task count. In contrast, \methodname{} recovers multiple successful sequences across all levels, often preserving the original multimodality but losing for some seeds the sequence ``[microwave, bottom burner]''.

\begin{table}[h!]
\caption{Unique successful action sequences discovered by each method, grouped by number of completed tasks. Underlined sequences are lost for some seeds.}
\label{tab:kitchen_sequences}
\resizebox{\linewidth}{!}{%
\centering
\small
\begin{tabular}{l c l}
\toprule
\textbf{Method} & \textbf{\# Tasks} & \textbf{Unique Sequences} \\
\midrule
\tablerowcolors
\multirow{2}{*}{\texttt{PRE} }
  & 1 & [microwave], [kettle] \\
  & 2 & [kettle, bottom burner], [microwave, bottom burner], [microwave, kettle] \\
    & 3 & - \\
\midrule
\multirow{3}{*}{\texttt{RES} / \texttt{DSRL} / \texttt{DPPO} / \texttt{DPPO[10]}}
  & 1 & [kettle] \\
  & 2 & [kettle, bottom burner] \\
  & 3 & [kettle, bottom burner, light switch] \\
\midrule
\multirow{3}{*}{\texttt{DSRL[\methodname{}]}}
  & 1 & [microwave], [kettle] \\
  & 2 & [kettle, bottom burner], \underline{[microwave, bottom burner]}, [microwave, kettle] \\
  & 3 & [kettle, bottom burner, light switch], \underline{['microwave', 'bottom burner', 'light switch']}, [microwave, kettle, bottom burner], \\
\bottomrule
\end{tabular}
}
\end{table}

\end{document}